\let\mypdfximage\pdfximage
\def\pdfximage{\immediate\mypdfximage}

\documentclass[10pt]{article} %
\usepackage[accepted]{arxiv}
\def\month{11}  %
\def\year{2023} %
\def\openreview{\url{https://openreview.net/forum?id=1309}} %

\usepackage[utf8]{inputenc} %
\usepackage[T1]{fontenc}    %
\usepackage{microtype}
\usepackage{graphicx}
\let\classAND\AND
\let\AND\relax
\usepackage{algorithmic}

\let\AND\classAND
\AtBeginEnvironment{algorithmic}{\let\AND\algoAND}

\usepackage{algorithm}
\usepackage{booktabs} %
\usepackage[aboveskip=2pt]{subcaption} %

\usepackage[colorlinks=true,linkcolor=black,urlcolor=black,citecolor=black]{hyperref}       %

\usepackage{comment}

\usepackage{scalerel}

\usepackage[utf8]{inputenc} %
\usepackage[T1]{fontenc}    %
\usepackage{url}            %
\usepackage{booktabs}       %
\usepackage{amsfonts}       %
\usepackage{nicefrac}       %
\usepackage{microtype}      %
\usepackage{natbib}
\usepackage{amsthm}
\newtheorem{prop}{Proposition}

\renewenvironment{proof}{{\bfseries Proof sketch.}}{\qedsymbol}
\usepackage{amsmath}
\usepackage{mathrsfs}
\usepackage{wrapfig}
\usepackage{amssymb}
\usepackage{tabularx}

\usepackage{placeins}
\usepackage{subcaption}
\usepackage[capitalize,nameinlink]{cleveref}

\crefname{section}{Sec.}{Secs.}
\crefname{proposition}{Prop.}{Props.}
\crefname{lemma}{Lem.}{Lems.}
\crefname{model}{Mod.}{Mods.}
\crefname{appendix}{App.}{Apps.}
\crefname{algorithm}{Alg.}{Algs.}
\crefname{prop}{Prop.}{Props.}

\usepackage{xspace}
\newcommand{\eg}{\textit{e.g.}\xspace}
\newcommand{\ie}{\textit{i.e.}\xspace}
\newcommand{\cf}{\textit{cf.}\xspace}

\usepackage{tikz,pgfplots}
\usetikzlibrary{calc}
\usetikzlibrary{plotmarks,shapes,arrows,patterns,fadings}
\usetikzlibrary{shapes.misc}
\pgfplotsset{compat=newest} 

\newlength{\figurewidth}
\newlength{\figureheight}

\definecolor{primarycolor}{HTML}{2C6EBA}
\definecolor{secondarycolor}{HTML}{7CAC56}

\renewcommand{\mid}[0]{\,|\,}

\renewcommand{\paragraph}[1]{\textbf{#1}~~}

\usepackage{xspace}
\def\eg{\textit{e.g.}\xspace}
\def\ie{\textit{i.e.}\xspace}
\def\cf{\textit{cf.}\xspace}

\usepackage{bm}
\newcommand{\mathbold}[1]{\bm{#1}}
\newcommand{\mbf}[1]{\mathbf{#1}}

\newcommand{\T}{\top}    %
\newcommand{\dd}{\,\mathrm{d}} %
\newcommand{\R}{\mathbb{R}}    %
\newcommand{\N}{\mathrm{N}}   %

\DeclareMathOperator{\E}{\mathbb{E}}

\DeclareMathOperator{\Var}{Var}

\DeclareMathOperator*{\argmin}{arg\,min}

\newcommand{\KL}[2]{\mathrm{D}_\text{KL}\left[#1\,\|\,#2\right]}
\newcommand{\vzero}{\bm{0}}

\DeclareMathOperator{\diam}{diam}

\newcommand{\vbeta}[0]{\mathbold{\beta}}

\newcommand{\vxi}[0]{\mathbold{\xi}}

\renewcommand{\mid}[0]{\,|\,}

\newcommand{\vp}{\mbf{p}}

\newcommand{\vx}{\mbf{x}}
\newcommand{\vy}{\mbf{y}}
\newcommand{\vz}{\mbf{z}}

\newcommand{\MC}{\mbf{C}}

\newcommand{\MI}{\mbf{I}}

\newcommand{\MP}{\mbf{P}}

\newcommand{\MT}{\mbf{T}}

\newcommand{\MX}{\mbf{X}}

\def\pmeasure{\mathbb{P}}
\def\qmeasure{\mathbb{Q}}
\def\probmeasures{\mathscr{P}}
\def\data{\mathcal{D}}
\def\ths{\textsuperscript{th}\xspace}

\tikzset{cross/.style={cross out, draw=black, minimum size=2*(#1-\pgflinewidth), inner sep=0pt, outer sep=0pt, line width=1pt}, cross/.default={3pt}}

\def\Benes{Bene\v{s}\xspace}

\newcommand\uglyepsilon\epsilon
\renewcommand\epsilon\varepsilon

\newcommand*\circled[1]{\tikz[baseline=(char.base)]{\node[shape=circle,fill=black!10,inner sep=1.2pt] (char) {\textcolor{black} #1};}}

\renewcommand{\algorithmiccomment}[1]{$\textcolor{gray}{/\ast}$ \hfill \textcolor{gray}{#1} \hfill $\textcolor{gray}{\ast/}$}

\title{Transport with Support: Data-Conditional Diffusion Bridges}

\author{\name Ella Tamir \email ella.tamir@aalto.fi \\
      \addr Department of Computer Science\\
      Aalto University
      \AND
      \name Martin Trapp \email martin.trapp@aalto.fi \\
      \addr Department of Computer Science\\
      Aalto University
      \AND
      \name Arno Solin \email arno.solin@aalto.fi\\
      \addr Department of Computer Science\\
      Aalto University}

\begin{document}

\maketitle

\begin{abstract}
The dynamic Schrödinger bridge problem provides an appealing setting for solving constrained time-series data generation tasks posed as optimal transport problems. It consists of learning non-linear diffusion processes using efficient iterative solvers.
Recent works have demonstrated state-of-the-art results (\eg, in modelling single-cell embryo RNA sequences or sampling from complex posteriors) but are limited to learning bridges with only initial and terminal constraints.
Our work extends this paradigm by proposing the Iterative Smoothing Bridge (ISB). 
We integrate Bayesian filtering and optimal control into learning the diffusion process, enabling the generation of constrained stochastic processes governed by sparse observations at intermediate stages and terminal constraints. 
We assess the effectiveness of our method on synthetic and real-world data generation tasks and we show that the ISB generalises well to high-dimensional data, is computationally efficient, and provides accurate estimates of the marginals at intermediate and terminal times. 
\end{abstract}

\section{Introduction}

Generative diffusion models have gained increasing popularity and achieved impressive results in a variety of challenging application domains, such as computer vision \citep[\eg, ][]{ho2020denoise,song2021maximum,dhariwal2021diffusion}, reinforcement learning \citep[\eg, ][]{Janner2022planning}, and time series modelling \citep[\eg, ][]{Rasul2021timeseries,vargas2021solving,tashiro2021csdi,Park2022sde}.
Recent works have explored connections between denoising diffusion models and the dynamic Schrödinger bridge problem \citep[SBP,~\eg,][]{vargas2021solving, bortoli2021diffusion, Shi2022condSchroedinger} to adopt iterative schemes for solving the dynamic optimal transport problem more efficiently.
The solution of the SBP then acts as a denoising diffusion model in finite time and is the closest in Kullback--Leibler (KL) divergence to the forward noising process of the noising model under marginal constraints.
Data may then be generated by time reversal of the process, \ie, through the denoising process.

In many applications, the interest is not purely in modelling transport between an initial and terminal state distribution.
For example, in naturally occurring generative processes, we typically observe snapshots of realizations {\em along intermediate stages} of individual sample trajectories (see \cref{fig:teaser}). Such problems arise in medical diagnosis (\eg, tissue changes and cell growth), demographic modelling, environmental dynamics, and animal movement modelling---see \cref{fig:birds} for modelling bird migration and wintering patterns. %
Recently, constrained optimal control problems have been explored by adding additional fixed path constraints~\citep{maoutsa2020interacting, maoutsa2021deterministic} or modifying the prior processes~\citep{fernandes2021shooting}. However, defining meaningful fixed path constraints or prior processes for the optimal control problems can be challenging, while sparse observational data are accessible in many real-world applications.

\begin{figure*}[t]
  \centering\scriptsize
  \begin{subfigure}[b]{.48\linewidth}
  \centering
  \begin{tikzpicture}
    \node[anchor=south west,inner sep=0] (image) at (0,0) {\includegraphics[width=\linewidth,trim=0 150 0 100,clip]{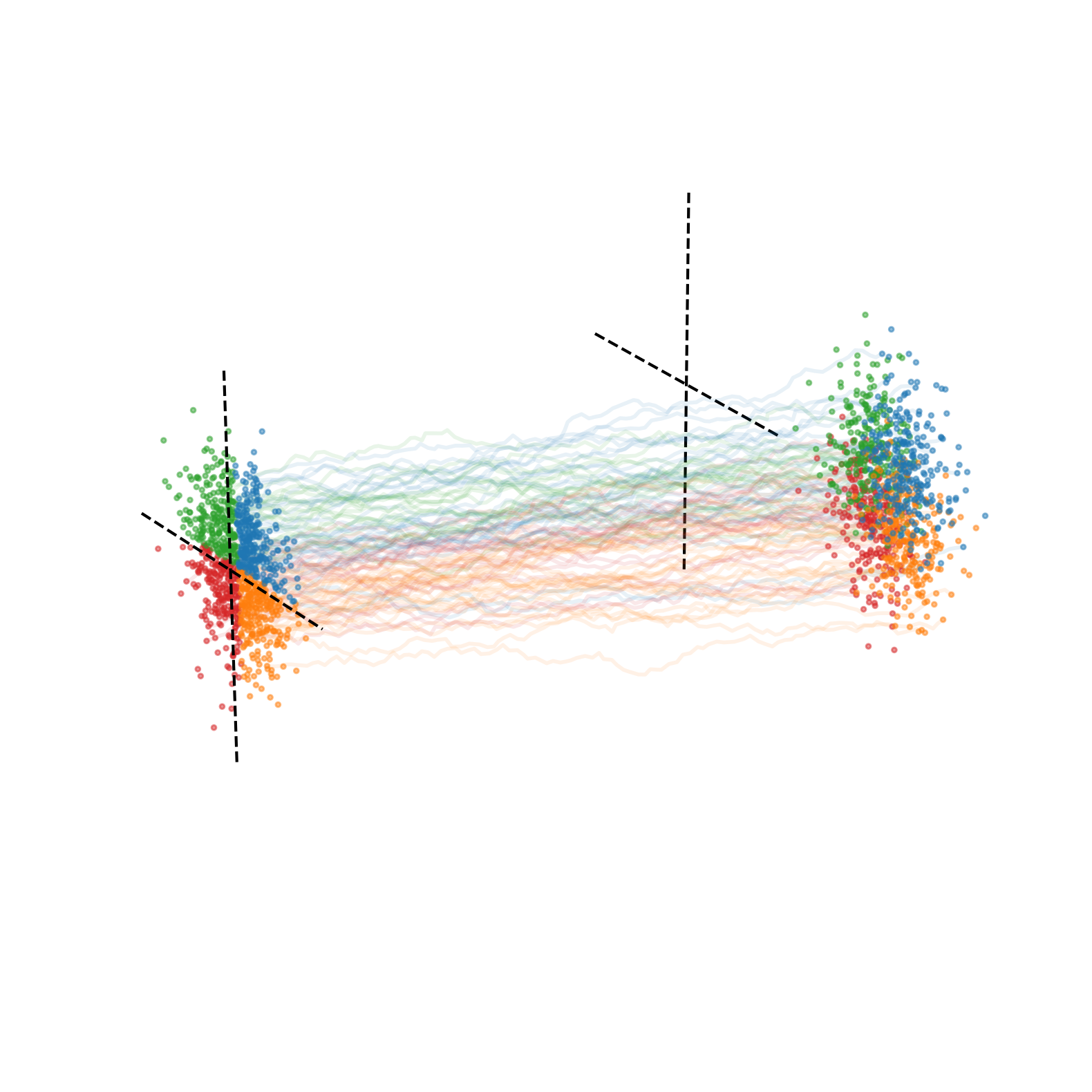}};
    \begin{scope}[x={(image.south east)},y={(image.north west)}]
        \draw[black,-latex'] (.22,.7) to[bend left=10] node[rotate=25,yshift=1em] {Time, $t$} (.62,1);      
        \node at (.22,0) {$\pi_0 \sim \mathrm{N}(\vzero,\MI)$};
        \node at (.77,.1) {$\pi_T \sim \mathrm{N}((10,0)^\T,\MI)$};
    \end{scope}
  \end{tikzpicture}
  \caption{Unconstrained transport \\ (Schrödinger bridge)}
  \label{fig:gaussplain}
  \end{subfigure}
  \hfill
  \begin{subfigure}[b]{.48\linewidth}
  \centering
  \captionsetup{justification=centering}
  \begin{tikzpicture}
    \node[anchor=south west,inner sep=0] (image) at (0,0) {\includegraphics[width=\linewidth,trim=0 150 0 100,clip]{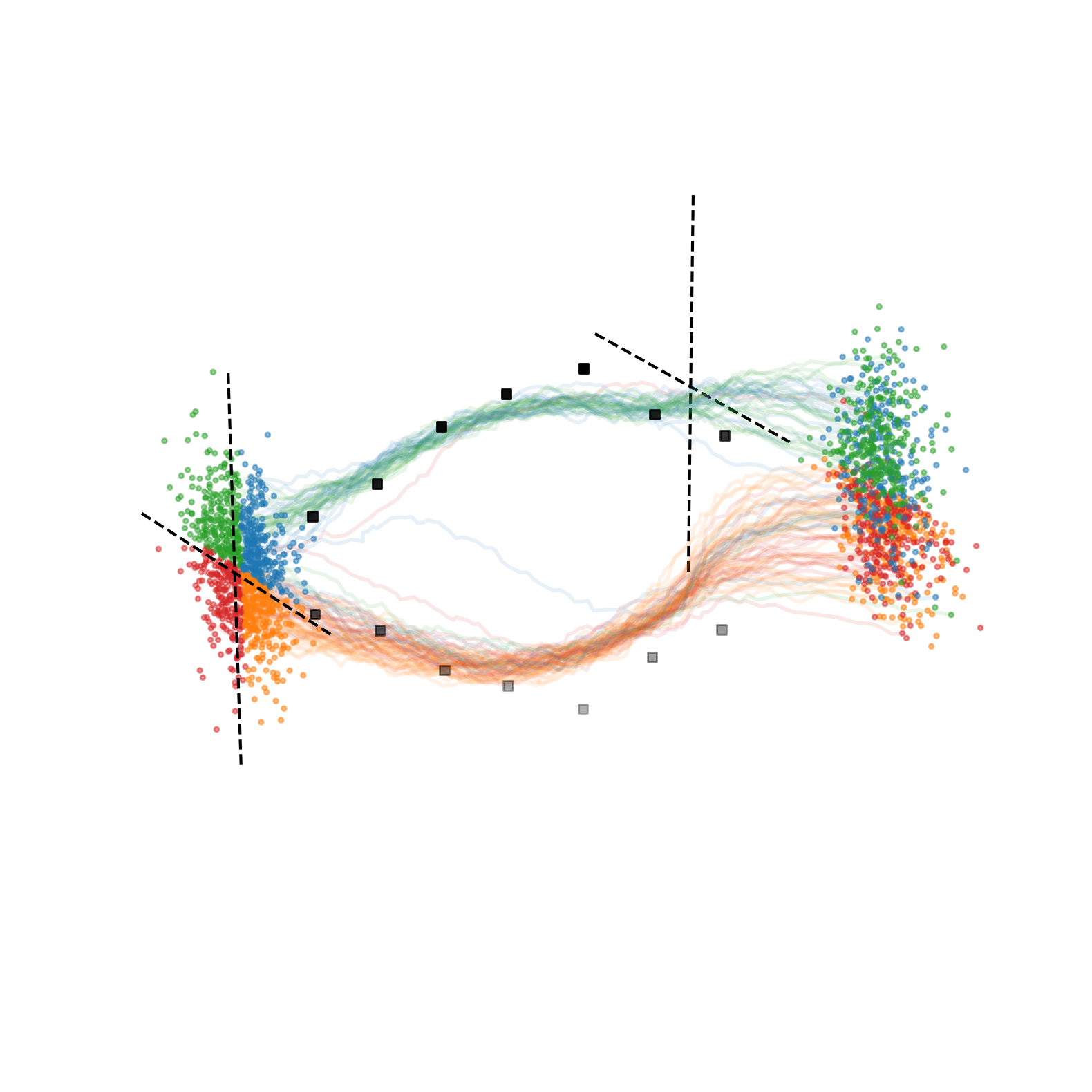}};
    \begin{scope}[x={(image.south east)},y={(image.north west)}]
        \node (a) at (.2,.9) {Sparse observations};
        \draw (a) -| (.405,.66);
        \node at (.22,0) {$\pi_0 \sim \mathrm{N}(\vzero,\MI)$};
        \node at (.77,.1) {$\pi_T \sim \mathrm{N}((10,0)^\T,\MI)$};
    \end{scope}
  \end{tikzpicture}
  \caption{Constrained transport \\ 
  (Iterative smoothing bridge)}
  \label{fig:gaussobs}
  \end{subfigure}
  \caption{Illustrative example transport between an initial unit Gaussian and a shifted unit Gaussian at the terminal time $T$. Unconstrained transport on the left and the solution constrained by sparse observations (\protect\tikz\protect\node[inner sep=2pt,fill=black!80]{};) on the right. Colour coding of the initial points is only for distinguishing the paths.}
  \label{fig:teaser}
\end{figure*}
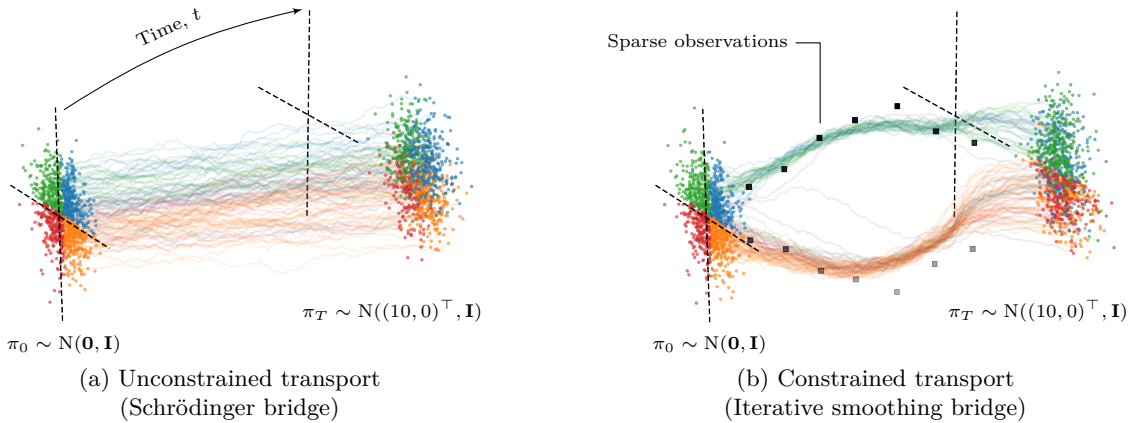

In this work, we propose the {\em Iterative Smoothing Bridge} (ISB), an iterative method for learning a dynamical system for time-series data generation under constraints on both the initial and terminal distribution and sparse observational data constraints. The sparse observational constraints act as a way to encourage the paths sampled from the transport process to lie close to the observed data points. We perform the conditioning by leveraging the iterative pass idea from the Iterative Proportional Fitting procedure  \citep[IPFP, see][]{kullback1968probability,bortoli2021diffusion} and applying differentiable particle filtering \citep{reich2013nonparametric, corenflos2021diffpf} within the outer loop. 
Integrating sequential Monte Carlo methods \citep[\eg, ][]{doucet2001sequential,chopin2020introduction} into the IPFP framework in such a way is non-trivial and can be understood as a novel iterative version of the algorithm by \citet{maoutsa2021deterministic} but with more general marginal constraints and additional path constraints defined by data.\looseness-1

We summarize the contributions as follows.
{\em (i)}~We propose a novel method for learning a dynamical model where the terminal constraints match a bridge problem and additional constraints are placed in form of sparse observations, inspired by optimal transport approaches.
{\em (ii)}~Thereof, we utilize the strong connections between the constrained bridging problem and particle filtering in sequential Monte Carlo, extending those links from pure inference to learning.
Additionally, {\em (iii)}~we demonstrate practical efficiency and show that the iterative smoothing bridge approach scales to high-dimensional data. \looseness-1

\subsection{Related Work}
\paragraph{Schrödinger bridges}
The problem of learning a stochastic process moving samples from one distribution to another can be posed as a type of transport problem known as a dynamic Schrödinger bridge problem \citep[SBP, \eg,][]{schrod1932surla, leonard2014schrodi}, where the marginal densities of the stochastic process are desired to resemble a given reference measure. In machine learning literature, the problem has been studied through learning the drift function of the dynamical system \citep{bortoli2021diffusion,wang2021DeepGL,vargas2021solving, bunne2022recovering, shi2023diffusion}. When an SDE system also defines the reference measure, the bridge problem becomes a constrained optimal control problem \citep[\eg,][]{caluya2021wasserstein, caluya2021constraint, chen2021stochastic, liu2022deep}, which has been leveraged in learning Schrödinger bridges by \citet{chen2022likelihood} through forward--backward SDEs. Moreover, neural stochastic control has been studied in \citet{zhang2022neural}. An optimal control problem with both initial and terminal distribution constraints and a fixed path constraint has been studied in \citet{maoutsa2020interacting} and \citet{maoutsa2021deterministic}, where particle filtering is applied to continuous path constraints but the boundary constraints are defined by a single point. \citet{maoutsa2023geometric} studies a problem setting where stochastic dynamics are inferred based on sparse observation, but with geometric constraints rather than in the form of a Schrödinger bridge problem. 

\paragraph{Diffusion models in machine learning}
The recent advances in diffusion models in machine learning literature have been focused on generating samples from complex distributions defined by data through transforming samples from an easy-to-sample distribution by a dynamical system \citep[\eg,][]{ho2020denoise, song2020sdegen, song2021maximum, nichol2021ImprovedDD}.
The concept of reversing SDE trajectories via score-based learning \citep{hyvarinen2005estimation,vincent2011connection} has allowed for models scalable enough to be applied to high-dimensional data sets directly in the data space.
In earlier work, score-based diffusion models have been applied to problems where the dynamical system itself is of interest, for example, for the problem of time series amputation in \citet{tashiro2021csdi}, inverse problems in imaging in \citet{song2022solving} and for importance sampling \citet{doucet2022scorebased}. Interpreting the diffusion modelling problem as optimal control has recently been studied in \citet{berner2022an}. Other dynamical models parametrized by neural networks have been applied to modelling latent time-series based on observed snapshots of dynamics \citep{rubanova2019latentode, li2020sdeadjoint}, but without further constraints on the initial or terminal distributions. \looseness-1

\paragraph{State-space models}
In their general form, state-space models combine a latent space dynamical system with an observation (likelihood) model. Evaluating the latent state distribution based on observational data can be performed by applying particle filtering and smoothing \citep{doucet2000smc} or by approximations of the underlying state distribution of a non-linear state-space model by a specific model family, for instance, a Gaussian \citep[see][for an overview]{sarkka2013bayesian}. Speeding up parameter inference and learning in state-space models has been widely studied \citep[\eg,][]{schon2011system,svensson2017flexible,kokkala2014em}. 
Particle smoothing can be connected to Schr\"odinger bridges via the two-filter smoother \citep[\eg,][]{bresler1986twofilter,  briers2009smoothing, hostettler2015twofilter}, where the smoothing distribution is estimated by performing filtering both forward from the initial constraint and backwards from the terminal constraint. We refer to \citet{mitter1996filtering} and \citet{todorov2008general} for a more detailed discussion on the connection of stochastic control and filtering and to \citet{chopin2020introduction} for an introduction to particle filters. 

\paragraph{Data Assimilation methods}
Data assimilation (DA) methods leverage techniques from state-space literature to `assimilate' observations into a mechanistic model in order to inform the model dynamics based on measurements \citep[\eg,][]{asch2016data,wang2000data}.
Approaches based on DA have found wide-spread use in scientific applications and have been extended to incorporate sparse observational data, for example, in numerical weather predictions \citep{whitaker2009comparison}, modelling cell state evolution in epithelial-mesenchymal transitions \citep{mendez2020cell}, or in oceanographic scenarios \citep{beiser2023comparison}.
A crucial difference to our work is that DA relies on a precise mechanistic model while our approach is data-driven, providing additional flexibility in modelling scenarios where formulating precise model dynamics is not possible.

\section{Background}
\label{sec:background}
Let $\mathcal{C} = C([0,T], \R^d)$
denote the space of continuous functions from $[0,T]$ to $\R^d$ and let $\mathcal{B}(\mathcal{C})$ denote the Borel $\sigma$-algebra on $\mathcal{C}$.
Let $\probmeasures(\pi_0, \pi_T)$ denote the space of probability measures on $(\mathcal{C}, \mathcal{B}(\mathcal{C}))$ such that the marginals at $0,T$ coincide with probability densities $\pi_0$ and $\pi_T$, respectively. The KL divergence from measure $\qmeasure$ to measure $\pmeasure$ is written as $\KL{\qmeasure}{\pmeasure}$, where we assume that $\qmeasure \ll \pmeasure$.
For modelling the time dynamics, we assume a (continuous-time) state-space model consisting of a non-linear latent It\^o SDE \citep[see, \eg,][]{Oksendal:2003,Sarkka+Solin:2019} in $[0,T] \times \R^d$ with drift function $f_\theta(\cdot)$ and diffusion function $g(\cdot)$, and a Gaussian observation model, \ie,
\begin{equation}\label{eq:solvefamily}
\vx_0 \sim \pi_0, \quad 
\dd\vx_t = f_{\theta}(\vx_t, t)\dd t + g(t) \dd \vbeta_t, %
\end{equation}
and $\vy_{k} \sim \N(\vy_{k}  \mid \vx_{t}, \sigma^2\, \MI_d) \, \big|_{t=t_k}$
where the drift function $f_{\theta}: \mathbb{R}^d \times [0,T] \to \mathbb{R}^d$ is a mapping modelled by a neural network (NN) parameterized by $\theta \in \Theta$, diffusion $g: [0, T] \to \mathbb{R}$ and $\vbeta_t$ denotes standard $d$-dimensional Brownian motion.  
$\vx_t$ denotes the latent stochastic process and $\vy_t$ denotes the observation-space process. In practice, we consider the continuous-discrete time setting, where the process is observed at discrete time instances $t_k$ such that observational data can be given in terms of a collection of input--output pairs $\{(t_j,\vy_{j})\}_{j=1}^M$.\looseness-1

\subsection{Schrödinger Bridges and Optimal Control}
\label{sec:schrodi}
The Schrödinger bridge problem \citep[SBP,][]{schrod1932surla, leonard2014schrodi} is 
an entropy-regularized optimal transport problem where the optimality is measured through the KL divergence from a reference measure $\pmeasure$ to the posterior $\qmeasure$, with fixed initial and final densities $\pi_0$ and $\pi_T$, \ie, 
\begin{equation}\label{eq:schrod}
    \min_{\qmeasure \in \probmeasures(\pi_0, \pi_T)} \KL{\qmeasure}{\pmeasure} .
\end{equation}

In this work, we consider only the case where the measures $\pmeasure$ and $\qmeasure$ are constructed as the marginals of an SDE, \ie, 
$\mathbb{Q}_t$ is the probability measure of the marginal of the SDE in \cref{eq:solvefamily}
at time $t$, whereas $\mathbb{P}_t$ corresponds to the probability measure of the marginal of a reference SDE
$\dd \vx_t = f(\vx_t, t)\dd t + g(t)\dd\vbeta_t$,
at time $t$, where we call $f$ the reference drift.
Under the optimal control formulation of the SBP \citep{caluya2021constraint} the KL divergence in \cref{eq:schrod} reduces to
\begin{equation}
\E \bigg [ \int_{0}^{T}\frac{1}{2g(t)^2} \|f_{\theta}(\vx_t, t) - f(\vx_t, t) \|^2 \dd t \bigg ],
\end{equation}
where the expectation is over paths from \cref{eq:solvefamily}.
\citet{ruschendorf1993note} and \citet{ruschendorf1995convergence} showed that a solution to the SBP can be obtained by iteratively solving two half-bridge problems using the Iterative Proportional Fitting procedure (IPFP) for $l=0, 1, \ldots, L$ steps,
\begin{align}
    \mathbb{Q}_{2l+1} &= \argmin_{\qmeasure \in \probmeasures(\cdot, \pi_T)} \KL{\qmeasure}{\mathbb{Q}_{2l}} \quad \text{and} & \quad 
    \mathbb{Q}_{2l+2} &= \argmin_{\qmeasure \in \probmeasures(\pi_0, \cdot)} \KL{\qmeasure}{\mathbb{Q}_{2l+1}},
\end{align}
where $\mathbb{Q}_0$ is set as the reference measure, and $\probmeasures(\pi_0, \cdot)$ and $\probmeasures(\cdot, \pi_T)$ denote the sets of probability measures with only either the marginal at time $0$ or time $T$ coinciding with $\pi_0$ or $\pi_T$, respectively. {Recently, the IPFP to solving Schrödinger bridges has been adapted as a machine learning problem \citep{bernton2019schrodinger, vargas2021solving, bortoli2021diffusion}.}
In practice, the interval $[0, T]$ is discretized and the forward drift $f_{\theta}$ and the backward drift $b_{\phi}$ of the corresponding reverse-time process \citep{haussmann1986time, follmer1988random} are modelled by NNs. Under the Gaussian transition approximations, each step in the discrete-time diffusion model can be reversed by applying an objective based on mean-matching.

\section{Methods}
\label{sec:methods}
Given an initial and terminal distribution $\pi_0$ and $\pi_T$, we are interested in learning a data-conditional bridge between $\pi_0$ and $\pi_T$. 
Let $\data = \{(t_j, \vy_j) \}_{j=1}^{M}$ be a set of $M$ sparsely observed values, \ie, only a few or no observations are made at each point in time and let the state-space model of interest be given by \cref{eq:solvefamily}.
Note that we deliberately use $(t_j, \vy_j)$ (instead of $(t_k, \vy_k)$) to highlight that we allow for multiple observations at the same time point $t_k$.
Our aim is to find a parameterization of the drift function $f_\theta$ such that evolving $N$ particles $\vx^i_t$, with $\vx^i_0 \sim \pi_0$ (with $i = 1,2, \ldots, N$),  according to \cref{eq:solvefamily} will result in samples $\vx_T^i$ from the terminal distribution $\pi_T$.
Inspired by the IPFP by \citet{bortoli2021diffusion}, which decomposes the SBP into finding two half-bridges, we propose to iteratively solve two modified half-bridge problems where the additional sparse observations are accounted for simultaneously. For this, let \looseness-1
\begin{align}
  \dd \vx_t &= f_{l, \theta}(\vx_{t}, t)\dd t + g(t) \dd \vbeta_t, \quad &&\vx_0 \sim \pi_0, \label{eq:forwardSDE} \\
  \dd \vz_t &= b_{l, \phi}(\vz_{t}, t)\dd t + g(t) \dd \hat{\vbeta}_t, \quad &&\vz_0 \sim \pi_T, \label{eq:reverseSDE}
\end{align}
denote the forward and backward SDE at iteration $l = 1,2,\dots,L$, where $\hat{\vbeta}_t$ is the reverse-time Brownian motion. For simplicity, we denote $\vbeta_t = \hat{\vbeta}_t$ when the direction of the SDE is clear.

To learn the Iterative Smoothing Bridge dynamics, we iteratively employ the following steps:
\circled{1} evolve \emph{forward} particle trajectories according to \cref{eq:forwardSDE} with drift $f_{l-1,\theta}$ and filter w.r.t.\ the observations $\{(t_j,\vy_{j})\}_{j=1}^M$, \circled{2} learn the drift function $b_{l, \phi}$ for the reverse-time SDE, \circled{3} evolve   \emph{backward} particle trajectories according to \cref{eq:reverseSDE} with the drift $b_{l,\phi}$ learned in step \circled{2} and filter w.r.t.\ the observations $\{(t_j,\vy_{j})\}_{j=1}^M$, and \circled{4} learn the drift function $f_{l,\theta}$ for the forward SDE based on the backward particles. \cref{fig:fwdbwd} illustrates the forward and backward process of our iterative scheme for a data-conditioned denoising diffusion bridge. 
Next, we will go through steps \circled{1}--\circled{4} in detail and introduce the Iterative Smoothing Bridge method for data-conditional diffusion bridges.

\begin{figure*}[t!]
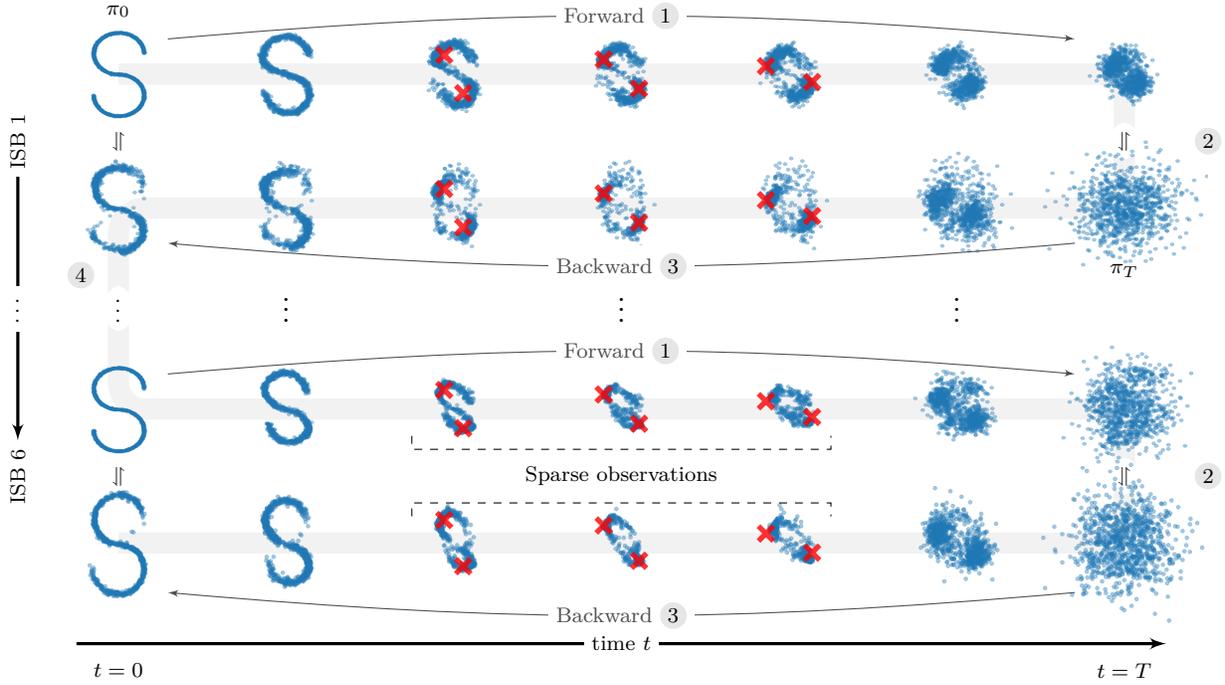

  \centering\footnotesize
  \begin{tikzpicture}
    \tikzstyle{block} = [draw=none]
    \tikzstyle{myarrow} = [black!70,-latex']

    \newlength{\blockwidth}
    \newlength{\blockheight}    
    \newlength{\arrownudge}    
    \setlength{\figurewidth}{0.23\textwidth}
    \setlength{\figureheight}{.9\figurewidth}
    \setlength{\blockwidth}{0.135\textwidth}
    \setlength{\blockheight}{.8\blockwidth}
    \setlength{\arrownudge}{0.25\blockwidth}    

    \draw[line width=8pt,draw=black!05, rounded corners=10pt] 
      (0,0) -- (6\blockwidth,0) --
        node[fill=white,rounded corners=6pt,inner sep=3,rotate=90] {$\large\leftrightharpoons$}
      (6\blockwidth,-\blockheight) -- (0,-\blockheight) -- 
        node[fill=white,rounded corners=6pt,inner sep=3,rotate=90] {$\large\cdots$}
      (0,-2.5\blockheight) -- (6\blockwidth,-2.5\blockheight) --
        node[fill=white,rounded corners=6pt,inner sep=3,rotate=90] {$\large\leftrightharpoons$}
      (6\blockwidth,-3.5\blockheight) -- (0,-3.5\blockheight);

    \node[] at (6.5\blockwidth,-.5\blockheight) {\circled{2}};
    \node[] at (6.5\blockwidth,-3\blockheight) {\circled{2}};

    \node[inner sep=3,rotate=90] at (0,-.5\blockheight) {$\large\leftrightharpoons$};
    \node[inner sep=3,rotate=90] at (0,-3\blockheight) {$\large\leftrightharpoons$};
    
    \node[] at (-.5,-1.5\blockheight) {\circled{4}};

    \foreach \img in {0,1,...,6}
      \node[block] (\img_forward_1) at (\img*\blockwidth,0) {\input{fig/s_shape/particles_fwd_0_\img.tex}};

    \draw[myarrow] ($(0_forward_1.north east) + (-\arrownudge,-\arrownudge)$) to[bend left=5] node[fill=white] {Forward \circled{1} } ($(6_forward_1.north west) + (\arrownudge,-\arrownudge)$);

    \foreach \img in {0,1,...,6}
      \node[block] (\img_backward_1) at (6*\blockwidth - \img*\blockwidth,-\blockheight) {\input{fig/s_shape/particles_bwd_0_\img.tex}};

    \draw[myarrow] ($(0_backward_1.south west) + (\arrownudge,\arrownudge)$) to[bend left=5] node[fill=white] {Backward \circled{3}} ($(6_backward_1.south east) + (-\arrownudge,\arrownudge)$);

    \node at (\blockwidth,-1.7\blockheight) {\large\vdots};
    \node at (3\blockwidth,-1.7\blockheight) {\large\vdots};
    \node at (5\blockwidth,-1.7\blockheight) {\large\vdots};

    \foreach \img in {0,1,...,6}
      \node[block] (\img_forward_5) at (\img*\blockwidth,-2.5\blockheight) {\input{fig/s_shape/particles_fwd_5_\img.tex}};

    \draw[myarrow] ($(0_forward_5.north east) + (-\arrownudge,-\arrownudge)$) to[bend left=5] node[fill=white] {Forward \circled{1} } ($(6_forward_5.north west) + (\arrownudge,-\arrownudge)$);

    \setlength{\figurewidth}{0.23\textwidth}
    \setlength{\figureheight}{\figurewidth}

    \foreach \img in {0,1,...,6}
      \node[block] (\img_backward_5) at (6*\blockwidth - \img*\blockwidth,-3.5\blockheight) {\input{fig/s_shape/particles_bwd_5_\img.tex}};

    \draw[myarrow] ($(0_backward_5.south west) + (\arrownudge,\arrownudge)$) to[bend left=5] node[fill=white] {Backward \circled{3}} ($(6_backward_5.south east) + (-\arrownudge,\arrownudge)$);    

    \node[rotate=90] (ISB1) at (-.6\blockwidth,-.5\blockheight) {\footnotesize ISB 1};
    \node[rotate=90] (ISB6) at (-.6\blockwidth,-3\blockheight) {\footnotesize ISB 6};
    \draw[black,-latex',very thick] (ISB1) to node[fill=white,rotate=90] {$\cdots$} (ISB6);

    \draw[black,-latex',very thick] (-0.25\blockwidth,-4.25\blockheight) to node[fill=white] {time $t$} (6.25\blockwidth,-4.25\blockheight);
    \node at (0,-4.45\blockheight) {$t=0$};
    \node at (6\blockwidth,-4.45\blockheight) {$t=T$};

    \draw[dashed] (1.75\blockwidth,-3.3\blockheight) -- (1.75\blockwidth,-3.2\blockheight) -- (4.25\blockwidth,-3.2\blockheight) -- (4.25\blockwidth,-3.3\blockheight);
    \node at (3\blockwidth,-3\blockheight) {Sparse observations};
    \draw[dashed] (1.75\blockwidth,-2.7\blockheight) -- (1.75\blockwidth,-2.8\blockheight) -- (4.25\blockwidth,-2.8\blockheight) -- (4.25\blockwidth,-2.7\blockheight);

    \node[anchor=north] at (0_forward_1.north) {$\pi_0$};
    \node[anchor=south] at (0_backward_1.south) {$\pi_T$};

  \end{tikzpicture}%
  \vspace*{-6pt}
  \caption{Sketch of a diffusion bridge between a 2D data distribution ($\pi_0$) and an isotropic Gaussian ($\pi_T$) constrained by sparse observations (\protect\tikz[baseline=-.5ex]\protect\node[cross=5.5pt,line width=2.5pt,red]{};). The forward diffusion at the first iteration (ISB~1) learns to account for the sparse observations but does not converge to the correct terminal distribution ($t=T$), and the backward diffusion {\it vice versa}. After iterating (ISB~6), the forward and backward diffusions converge to the correct targets and are able to account for the sparse observational data.
  }
  \label{fig:fwdbwd}
\end{figure*}

\subsection{The Iterative Smoothing Bridge}\label{sec:steps}
The Iterative Smoothing Bridge (ISB) method iteratively generates particle filtering trajectories (steps \circled{1} and \circled{3} in \cref{fig:fwdbwd}) and learns the parameterizations of the forward and backward drift functions $f_{l, \theta}$ and $b_{l, \phi}$ (steps \circled{2} and \circled{4}) by minimizing a modified version of the mean-matching objective presented by \citet{bortoli2021diffusion}. {Note that steps \circled{2} and \circled{4} are dependent on applying differential resampling in the particle filtering steps \circled{1} and \circled{3} for reversing the generated trajectories.} 
We will now describe the forward trajectory generating step \circled{1} and the backward drift learning step \circled{2} in detail. Steps \circled{3} and \circled{4} are given by application of \circled{1} and \circled{2} on their reverse-time counterparts.

\paragraph{Step \circled{1} (and \circled{3}):}
Given a fixed discretization of the time interval $[0, T]$ denoted as $\{ t_k\}_{k=1}^K$ with $t_1 = 0$ and $t_K = T$, denote the time step lengths as $\Delta_k = t_{k+1} - t_k$. By truncating the It\^o--Taylor series of the SDE, we can consider an Euler--Maruyama \citep[\eg, Ch.~8 in][]{Sarkka+Solin:2019} type of discretization for the continuous-time problem. 
We give the time-update of the $i$\ths particle at time $t_k$ evolved according to \cref{eq:forwardSDE}, \ie, 
\begin{equation}
  \tilde{\vx}_{t_k}^i = \vx_{t_{k-1}} + f_{l-1, \theta}(\vx_{t_{k-1}, t_{k-1}}) \Delta_k  + g(t_{k-1})\sqrt{\Delta_k}\,\vxi_k^i,
\end{equation}
where $\vxi_k^i \sim \N(\vzero, \MI)$. Notice that we have not yet conditioned on the observational data. In step \circled{3}, the particles $\tilde{\vz}_{t_k}^i$ of the backward SDE \cref{eq:reverseSDE} are similarly obtained. {The SDE dynamics sampled in steps \circled{1} and \circled{3} apply the learned drift functions $f_{l-1, \theta}$ and $b_{l, \phi}$ from the previous step and do not require sampling from the underlying SDE model.}
For times $t_k$ at which no observations are available, we set $\vx_t^i = \tilde{\vx}_t^i$ (and $\vz_{t_k}^i = \tilde{\vz}_{t_k}^i$ respectively) and otherwise compute the particle filtering weights $w_{t_k}^i$ based on the observations $\{(t_j, \vy_j) \in \data \mid t_j = t_k\}$ for resampling. See \cref{sec:compmethods} for details on the particle filtering. %

For resampling, we employ a {\em differentiable resampling} procedure, where the particles and weights $(\tilde{\vx}_{t_k}^i, w_{t_k}^i)$ are transported to uniformly weighted particles $(\vx_{t_k}^i, \frac{1}{N})$ by solving an entropy-regularized optimal transport problem \citep{cuturi2013sinkhorn, peyre2017computational, corenflos2021diffpf}  (see \cref{app:diff}).
Through application of the $\epsilon$-regularized optimal transport map $\MT_{(\epsilon)} \in \mathbb{R}^{N \times N}$ \citep[see][]{corenflos2021diffpf} the particles are resampled via the map to 
$\vx_{t_k}^i = \tilde{\MX}^{\T}_{t_k} \, \MT_{(\epsilon),i}$,
where $\tilde{\MX}_{t_k} \in \R^{N \times d}$ denotes the stacked particles $\{\tilde{\vx}_{t_k}^i \}_{i=1}^N$ at time $t_k$ before resampling. 

\paragraph{Step \circled{2} (and \circled{4}): }
Given the particles $\{\vx_{t_{k}}^i \}_{k=1, i=1}^{K, N}$, we now aim to learn the drift function for the respective reverse-time process. The purpose of this step is to find a mean-matching reversal of the trajectories, in other words we aim to find $f_{l, \theta}$ such that it best explains the change we observe from $\{\vx_{t_{k}}^i \}$ to $\{\vx_{t_{k+1}}^i \}$ for each trajectory $i=1, 2, \ldots, N$ and particle $k=1,2,\ldots, N$. We will review the loss functions used for this optimization step, where the loss outside observation times will match the approach in \citet{bortoli2021diffusion} and the loss at observation times is motivated by a smoothing of trajectories (see \cref{app:hamiltonjacobi} for a discussion).
 
In case no observation is available at time $t_k$, we apply the mean-matching loss based on a Gaussian transition approximation proposed in \citet{bortoli2021diffusion}:\looseness-1
\begin{equation}\label{eq:basicloss}
    \ell_{k+1,\text{nobs}}^i = \| b_{l,  \phi}(\vx_{t_{k+1}}^i, t_{k+1})\Delta_k - \vx_{t_{k+1}}^i 
 - f_{l-1, \theta}(\vx_{t_{k+1}}^i, t_k)\Delta_k + \vx_{t_k}^i +  f_{l-1, \theta}(\vx_{t_k}^i, t_k)\Delta_k \|^2. 
\end{equation}
In case an observation is available at time $t_k$ the particle values $\tilde{\MX}_{t_k}$ will be coupled through the optimal transport map. Therefore, the transition density is a sum of Gaussian variables (see \cref{app:maths} for details and a derivation), and the mean-matching loss is therefore given by:
\begin{equation}\label{eq:obsloss}
\ell_{k+1,\text{{obs}}}^i = \|  b_{l,  \phi}(\vx_{t_{k+1}}^i, t_{k+1})\Delta_k - \vx_{t_{k+1}}^i - f_{l-1, \theta}(\vx_{t_{k+1}}^i, t_k)\Delta_k 
  + \textstyle\sum_{n=1}^N T_{(\epsilon),i, n} \left( \vx_{t_k}^n + f_{l-1, \theta}(\vx_{t_k}^n, t_k)\Delta_k \right)\|^2 .
\end{equation}
The derivation of \cref{eq:obsloss} relies on mean-matching the trajectories as in \citet{bortoli2021diffusion} combined with applying the differentiable resampling optimal transport map $\MT_{(\epsilon)}\tilde{\vx}_{t_k}^i=\vx_{t_k}^i$ on all the particles to obtain the transition density $p_{\vx_{t_{k}} \mid \vx_{t_{k-1}}^i}(\vx_{t_k})$ at observation times, resulting in a Gaussian distribution dependent on all the particles. In addition, we apply the property that the reverse drift should satisfy
\begin{equation}\label{eq:smoothreverse1}
b_{l, \phi} (\vx_{t_{k+1}}, t_{k+1}) = f_{l-1, \theta}(\vx_{t_{k+1}}, t_k )- g(t_{k+1})^2 \nabla \ln p_{t_{k+1}},
\end{equation}
where $p_{t_{k+1}}$ is the particle filtering density after differential resampling at time $t_{k+1}$. Thus the impact of observations to the loss function is two-fold, the observations define the value of the transport matrix $\MT_{(\epsilon)}$ and the marginal score $\nabla \ln p_{t_{k+1}}$. The use of the reverse drift \cref{eq:smoothreverse1} is further motivated by the smoothing reverse presented in \cref{app:hamiltonjacobi}, where we discuss how \cref{eq:smoothreverse1} matches a backwards controlled drift.

The overall objective function is a combination of both loss functions, with the respective mean-matching loss depending on whether $t_k$ is an observation time. The final loss function is written as:
\begin{equation}\label{eq:fullloss}
\ell(\phi) = \sum_{i=1}^N \left[ \sum_{k=1}^{K }\ell_{k, \text{obs}}^i (\phi) \mathbb{I}_{y_{t_k} \not = \emptyset} +\ell_{k, \text{nobs}}^i (\phi) \mathbb{I}_{y_{t_k} = \emptyset} \right],
\end{equation}
where $\mathbb{I}_{\text{cond.}}$ denotes an indicator function that returns `1' iff the condition is true, and `0' otherwise. Consequently, the parameters $\phi$ of $b_{l,\phi}$ are learned by minimizing \cref{eq:fullloss} through gradient descent. In practice, a cache of trajectories $\{\vx_{t_{k}}^i \}_{k=1, i=1}^{K, N}$ is maintained through training of the drift functions, and refreshed at a fixed number of inner loop iterations, as in \citet{bortoli2021diffusion}, avoiding differentiation over the SDE generation computational graph. The calculations for step \circled{4} follow similarly. We present a high-level description of the ISB steps in \cref{alg:isb}. \looseness-1 

The learned backward drift $b_{l, \phi}$ can be interpreted as an analogy of the backward drift in \citet{maoutsa2021deterministic}, connecting our approach to solving optimal control problems through Hamilton--Jacobi equations, see \cref{app:hamiltonjacobi} for an analysis of the backwards SDE and the control objective.
While we are generally considering problem settings where the number of observations is low, we propose that letting $M \to \infty$ yields the underlying marginal distribution, see \cref{prop:density} in \cref{app:infinity}.\looseness-1

\begin{algorithm}[tb]
   \caption{The Iterative Smoothing Bridge}
   \label{alg:isb}
   \renewcommand{\algorithmiccomment}[1]{\hfill\textcolor{gray}{\(\triangleright\) #1}}
\begin{algorithmic}
   \REQUIRE Marginal constraints $(\pi_0, \pi_T)$, observations $\data = \{(t_j,\vy_{j})\}_{j=1}^M$, initial drift function $f_{0,\theta}$, iterations $L$, discretization steps $K$, number of particles $N$, observation noise schedule $\kappa (l)$
   \ENSURE Learned forward and backward drift $(f_{\theta}, b_{\phi})$
  \FOR{$l = 1$ {\bfseries to} $L$}

   \STATE %
   \tikz[overlay]{\draw[gray](-1em,0.5em)--(-1em,-8.5em);\node[rotate=90,color=darkgray,fill=white] at (-1em,-4em) {\scriptsize Forward process};}%
   Initialize forward particles $\{\vx_{0}^i\}_{i=1}^N \sim \pi_0$
   \FOR{$k = 1$ {\bfseries to} $K$} 
	\STATE Generate $\{\vx_{k}^i\}_{i=1}^N$ using $\{\vx_{k-1}^i\}_{i=1}^N$ \COMMENT{\cref{eq:forwardSDE}} 
   	\IF{Observations at $t_k$}
   	\STATE $\{\vx_{k}^i\}_{i=1}^N \gets \textbf{DiffResample}(\{\vx_{k}^i\}_{i=1}^N, \kappa(l))$ 
   	\ENDIF
  \ENDFOR
  \STATE Optimize the forward loss function w.r.t.\ $\phi$ \COMMENT{\cref{eq:fullloss}}
  \STATE %
  \tikz[overlay]{\draw[gray](-1em,0.5em)--(-1em,-8.5em);\node[rotate=90,color=darkgray,fill=white] at (-1em,-4em) {\scriptsize Backward process};}%
  Initialize backward particles $\{\vz_{K}^i\}_{i=1}^N \sim \pi_T$
    \FOR{$k=K$ {\bfseries to} $1$} 
        \STATE Generate $\{\vz_{k-1}^i\}_{i=1}^N$ using $\{\vz_{k}^i\}_{i=1}^N$  \COMMENT{\cref{eq:reverseSDE}}
   \IF{Observations at $t_k$}
   	\STATE $\{\vz_{k-1}^i\}^N_{i=1} \gets \textbf{DiffResample}(\{\vz_{k-1}^i\}^N_{i=1}, \kappa (l))$ 
   	\ENDIF
  \ENDFOR
  \STATE Optimize the backwards loss function w.r.t.\ $\theta$ \COMMENT{\cref{eq:appfulllossbwd}}
   \ENDFOR
\end{algorithmic}
\end{algorithm}

\subsection{Computational Considerations}\label{sec:compmethods}
The ISB algorithm is a generic approach to learn data-conditional diffusion bridges under various choices of, \eg, the particle filter proposal density or the reference drift. Next, we cover practical considerations for the implementation of the method and highlight the model choices in the experiments.\looseness-1

\paragraph{Multiple observations per time step}
Naturally, we can make more than one observation at a single point in time $t_k$, denoted as $\data_{t_k} = \{(t_j, \vy_j) \in \data \mid t_j = t_k\}$. To compute particle weights $w^i_{t_k}$ for the $i$\ths particle we consider only the $H$-nearest neighbours of $\vx_{t_k}^i$ in $\data_{t_k}$ instead of all observations in $\data_{t_k}$. By restricting to the $H$-nearest neighbours, denoted as $\data^H_{t_k}$, we introduce an additional locality to the proposal density computation, which can be helpful in the case of multimodality. On the other hand, letting $H > 1$ results in weights which take into account the local density of the observations, not only the distance to the nearest neighbour. In experiments with few observations, we set $H=1$, the choice of $H$ is discussed when we have set the value higher.\looseness-1

\paragraph{Particle filtering proposal}
The proposal density chosen for the ISB is the bootstrap filter, where the proposal matches the Gaussian transition density $p(\vx_{t_k} \mid \vx_{t_{k-1}})$. Assuming a Gaussian noise model $\N(\vzero, \sigma^2 \MI)$, the unnormalized log-weights for the $i$\ths particle at time $t_k$ are given by $\log w_{t_k}^i = -\nicefrac{1}{2\sigma^2} \sum_{\vy_j \in \data^H_{t_k}} \| \vx_{t_k}^i - \vy_j\|^2$. While we restrict our approach in practice to the bootstrap filter, applying more sophisticated proposals such as in the auxiliary particle filter \citep{pitt1999filtering} could improve the results in some cases, although restricting the problem to a linear observation model.

\paragraph{Observational noise schedule}
In practice, using a constant observation noise variance $\sigma^2$ can result in an iterative scheme which does not have a stationary point as $L \to \infty$. Even if the learned drift function $f_{l, \theta}$ was optimal, the filtering steps \circled{1} and \circled{3} would alter the trajectories unless all particles would have uniform weights. Thus, we introduce a noise schedule $\kappa(l)$ which ensures that the observation noise increases in the number of ISB iterations, causing ISB to converge to the IPFP \citep{bortoli2021diffusion} as $L \to \infty$. We found that letting the observation noise first decrease and then increase (in the spirit of simulated annealing) often outperformed a strictly increasing observation noise schedule.
The noise schedule is studied in \cref{app:ablation}, where we derive the property that letting $L \to \infty$ yields IPFP.\looseness-1

\paragraph{Drift initialization}
Depending on the application, one may choose to incorporate additional information by selecting an appropriate initial drift.
A possible choice includes a pre-trained neural network drift learned to transport $\pi_0$ to $\pi_T$ without accounting for observations. However, starting from a drift for the unconstrained SBP can be problematic in cases where the observations are far away from the unconstrained bridge. To encourage exploration, one may choose $f_0=0$ for the initial drift. In various problem settings, we found a zero drift and starting from the SBP to be successful in the experiments. See \cref{app:ablation} for discussion.

\section{Experiments}
\label{sec:experiments}
To assess the properties and performance of the ISB, we present a range of experiments that demonstrate how the iterative learning procedure can incorporate both observational data and terminal constraints. We start with simple examples that build intuition (\cf \cref{fig:teaser} and \cref{fig:fwdbwd}) and show standard ML benchmark tasks. For quantitative assessment, we design an experiment with a non-linear SDE for which the marginal distributions are available in closed-form. Finally, we demonstrate our model both in a highly multimodal bird migration task, conditioned image generation, and in a single-cell embryo RNA modelling problem. Ablation studies are found in \cref{app:ablation}.

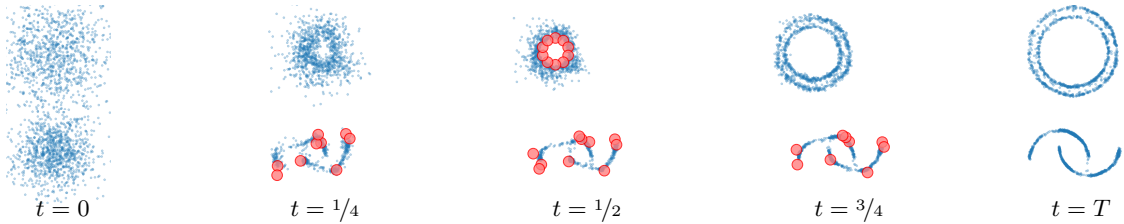
\begin{figure}[t!]
  \centering
  \setlength{\figurewidth}{0.18\linewidth}
  \setlength{\figureheight}{\figurewidth}
  \setlength{\blockwidth}{.18\linewidth}
  \begin{subfigure}[b]{\blockwidth}
    \centering
\begin{tikzpicture}

\begin{axis}[
height=\figureheight,
hide x axis,
hide y axis,
tick pos=left,
width=\figurewidth,
xmin=-10.8168037133601, xmax=12.6231202032664,
ymin=-13.0521161056362, ymax=10.3878078109904
]
\addplot graphics [includegraphics cmd=\pgfimage,xmin=-22.9636517032675, xmax=23.6057335615534, ymin=-16.9328982110379, ymax=14.113358632176] {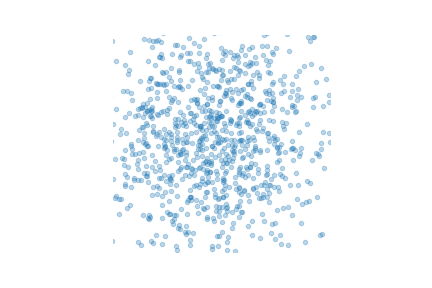};
\end{axis}

\end{tikzpicture}
  \end{subfigure}
  \hfill
  \begin{subfigure}[b]{\blockwidth}
    \centering
\begin{tikzpicture}

\begin{axis}[
height=\figureheight,
hide x axis,
hide y axis,
tick pos=left,
width=\figurewidth,
xmin=-10.8168037133601, xmax=12.6231202032664,
ymin=-13.0521161056362, ymax=10.3878078109904
]
\addplot graphics [includegraphics cmd=\pgfimage,xmin=-22.9636517032675, xmax=23.6057335615534, ymin=-16.9328982110379, ymax=14.113358632176] {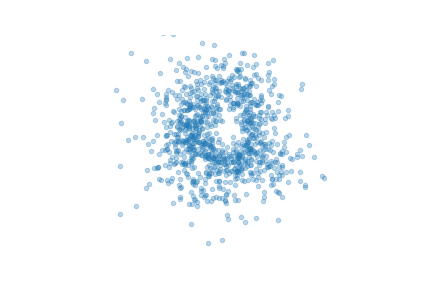};
\end{axis}

\end{tikzpicture}
  \end{subfigure}
  \begin{subfigure}[b]{\blockwidth}
    \centering
\begin{tikzpicture}

\begin{axis}[
height=\figureheight,
hide x axis,
hide y axis,
tick pos=left,
width=\figurewidth,
xmin=-10.8168037133601, xmax=12.6231202032664,
ymin=-13.0521161056362, ymax=10.3878078109904
]
\addplot graphics [includegraphics cmd=\pgfimage,xmin=-22.9636517032675, xmax=23.6057335615534, ymin=-16.9328982110379, ymax=14.113358632176] {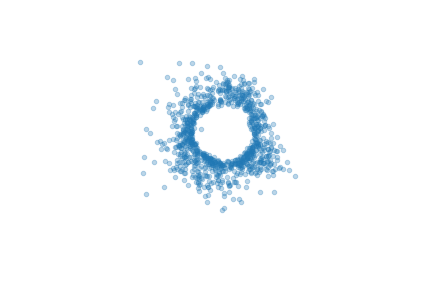};
\addplot [draw=red, fill=red!50, mark=*, only marks, mark options={solid}, mark size=2pt, opacity=0.8]
table{%
x  y
1.5 3
3.26335573196411 2.42705106735229
4.35316944122314 0.92705100774765
4.35316944122314 -0.92705100774765
3.26335573196411 -2.42705106735229
1.5 -3
-0.263355761766434 -2.42705106735229
-1.35316956043243 -0.92705100774765
-1.35316956043243 0.92705100774765
-0.263355761766434 2.42705106735229
};
\end{axis}

\end{tikzpicture}
  \end{subfigure}  
  \hfill
  \begin{subfigure}[b]{\blockwidth}
    \centering
\begin{tikzpicture}

\begin{axis}[
height=\figureheight,
hide x axis,
hide y axis,
tick pos=left,
width=\figurewidth,
xmin=-10.8168037133601, xmax=12.6231202032664,
ymin=-13.0521161056362, ymax=10.3878078109904
]
\addplot graphics [includegraphics cmd=\pgfimage,xmin=-22.9636517032675, xmax=23.6057335615534, ymin=-16.9328982110379, ymax=14.113358632176] {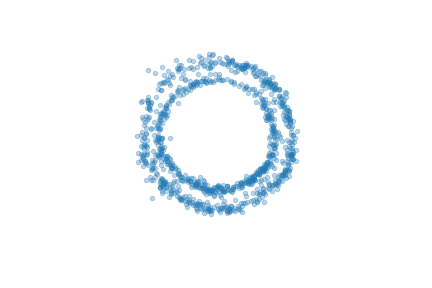};
\end{axis}

\end{tikzpicture}
  \end{subfigure}
  \hfill  
  \begin{subfigure}[b]{\blockwidth}
    \centering
\begin{tikzpicture}

\begin{axis}[
height=\figureheight,
hide x axis,
hide y axis,
tick pos=left,
width=\figurewidth,
xmin=-10.8168037133601, xmax=12.6231202032664,
ymin=-13.0521161056362, ymax=10.3878078109904
]
\addplot graphics [includegraphics cmd=\pgfimage,xmin=-22.9636517032675, xmax=23.6057335615534, ymin=-16.9328982110379, ymax=14.113358632176] {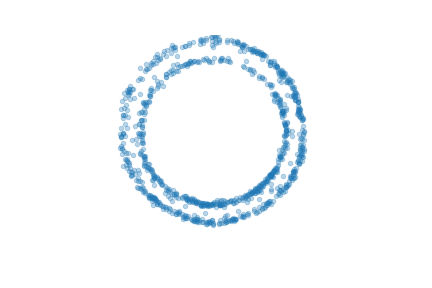};
\end{axis}

\end{tikzpicture}
  \end{subfigure}\\[-.5em]
  \begin{subfigure}[b]{\blockwidth}
    \centering
\begin{tikzpicture}

\begin{axis}[
height=\figureheight,
hide x axis,
hide y axis,
tick pos=left,
width=\figurewidth,
xmin=-7.96473726945852, xmax=16.5689135381996,
ymin=-10.1051134125469, ymax=14.4285373951112
]
\addplot graphics [includegraphics cmd=\pgfimage,xmin=-20.6783675721025, xmax=28.0639850258937, ymin=-14.1669761290465, ymax=18.3279256029509] {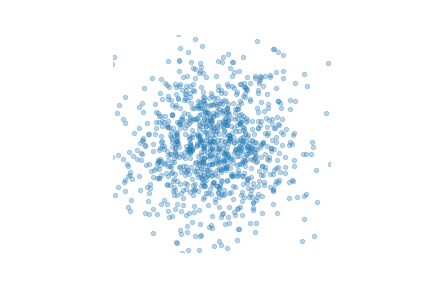};
\end{axis}

\end{tikzpicture}\vspace{-.5em}
    \caption*{\small $t=0$}
  \end{subfigure}
  \hfill
  \begin{subfigure}[b]{\blockwidth}
    \centering
\begin{tikzpicture}

\begin{axis}[
height=\figureheight,
hide x axis,
hide y axis,
tick pos=left,
width=\figurewidth,
xmin=-7.96473726945852, xmax=16.5689135381996,
ymin=-10.1051134125469, ymax=14.4285373951112
]
\addplot graphics [includegraphics cmd=\pgfimage,xmin=-20.6783675721025, xmax=28.0639850258937, ymin=-14.1669761290465, ymax=18.3279256029509] {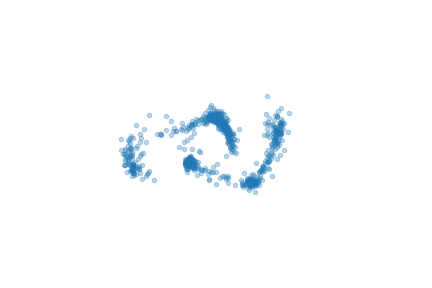};
\addplot [draw=red, fill=red!50, mark=*, only marks, mark options={solid}, mark size=2pt, opacity=0.8]
table{%
x  y
-5.74272012710571 -1.2900493144989
10.282054901123 6.23231601715088
11.1886224746704 5.12994289398193
-7.97766780853271 1.20648860931396
8.09059429168701 -2.23533320426941
4.45591497421265 4.15110778808594
3.3600537776947 3.95255351066589
3.7745156288147 6.05395746231079
-0.357202261686325 -0.101248815655708
-5.85373735427856 -3.47627544403076
};
\end{axis}

\end{tikzpicture}\vspace{-.5em}
    \caption*{\small $t=\nicefrac{1}{4}$}
  \end{subfigure}
  \hfill  
  \begin{subfigure}[b]{\blockwidth}
    \centering
\begin{tikzpicture}

\begin{axis}[
height=\figureheight,
hide x axis,
hide y axis,
tick pos=left,
width=\figurewidth,
xmin=-7.96473726945852, xmax=16.5689135381996,
ymin=-10.1051134125469, ymax=14.4285373951112
]
\addplot graphics [includegraphics cmd=\pgfimage,xmin=-20.6783675721025, xmax=28.0639850258937, ymin=-14.1669761290465, ymax=18.3279256029509] {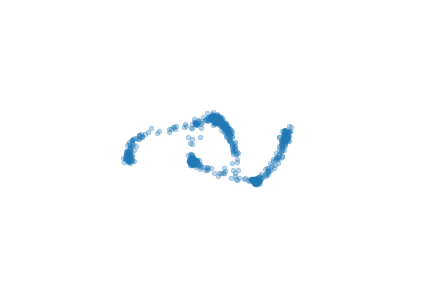};
\addplot [draw=red, fill=red!50, mark=*, only marks, mark options={solid}, mark size=2pt, opacity=0.8]
table{%
x  y
-5.85671424865723 -0.975227236747742
11.3966102600098 4.97208642959595
11.7787132263184 3.69037508964539
-7.75934648513794 1.58934283256531
8.92796421051025 -2.40417170524597
5.42121362686157 4.35129737854004
3.58317613601685 4.98470258712769
3.04773569107056 5.66898202896118
0.561857163906097 -0.249238520860672
-6.35246849060059 -1.97278606891632
};
\end{axis}

\end{tikzpicture}\vspace{-.5em}
    \caption*{\small $t=\nicefrac{1}{2}$}
  \end{subfigure}  
  \hfill
  \begin{subfigure}[b]{\blockwidth}
    \centering
\begin{tikzpicture}

\begin{axis}[
height=\figureheight,
hide x axis,
hide y axis,
tick pos=left,
width=\figurewidth,
xmin=-7.96473726945852, xmax=16.5689135381996,
ymin=-10.1051134125469, ymax=14.4285373951112
]
\addplot graphics [includegraphics cmd=\pgfimage,xmin=-20.6783675721025, xmax=28.0639850258937, ymin=-14.1669761290465, ymax=18.3279256029509] {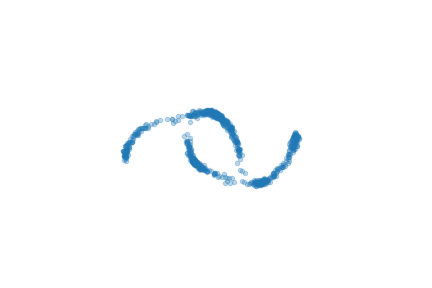};
\addplot [draw=red, fill=red!50, mark=*, only marks, mark options={solid}, mark size=2pt, opacity=0.8]
table{%
x  y
12.9552555084229 4.29890060424805
13.3716583251953 3.17004346847534
-7.17733383178711 2.52120113372803
9.50479507446289 -2.82684803009033
5.19372701644897 4.41709136962891
4.38437175750732 5.42100954055786
3.66518402099609 5.84365081787109
0.788959980010986 0.159530133008957
-6.95842123031616 -0.97951078414917
};
\end{axis}

\end{tikzpicture}\vspace{-.5em}
    \caption*{\small $t=\nicefrac{3}{4}$}
  \end{subfigure}
  \begin{subfigure}[b]{\blockwidth}
    \centering
\begin{tikzpicture}

\begin{axis}[
height=\figureheight,
hide x axis,
hide y axis,
tick pos=left,
width=\figurewidth,
xmin=-7.96473726945852, xmax=16.5689135381996,
ymin=-10.1051134125469, ymax=14.4285373951112
]
\addplot graphics [includegraphics cmd=\pgfimage,xmin=-20.6783675721025, xmax=28.0639850258937, ymin=-14.1669761290465, ymax=18.3279256029509] {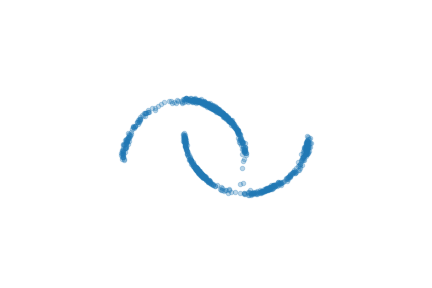};
\end{axis}

\end{tikzpicture}\vspace{-.5em}
    \caption*{\small $t=T$}
  \end{subfigure}\\[-6pt]
  \caption{2D toy experiments from scikit-learn with both cases starting from a Gaussian: The {\sc two circles} (top) and {\sc two moons} (bottom) data sets, with observations (red markers) constraining the problem. For the circles, the $10$ circular observations at $t=0.5$ first force the method to create a circle that then splits into two; in the lower plot the observations at $t\in [0.25, 0.5, 0.75] $ split the data into clusters before joining them into two moons. See \cref{fig:2dtoyipfp} in the Appendix for the IPFP result. \looseness-1}%
  \label{fig:2dtoy}
\end{figure}

\paragraph{Experiment setup}
In all experiments, the forward and backward drift functions $f_{\theta}$ and $b_{\phi}$ are parametrized as neural networks. For low-dimensional experiments, we apply the MLP block design as in \citet{bortoli2021diffusion}, and for the image experiment an U-Net as in \citet{nichol2021ImprovedDD}. The latent state SDE was simulated by Euler--Maruyama with a fixed time-step of $0.01$ over $100$ steps and $1000$ particles if not otherwise stated.
All low-dimensional (at most $d=5$) experiments were run on a MacBook Pro laptop CPU, whereas the image experiments used a single NVIDIA A100 GPU and ran for 5~h 10~min. Notice that since ISB only performs particle filtering outside the stochastic gradient training loop, the training runtime is in the same order as in the earlier Schrödinger bridge image generation experiments of \citet{bortoli2021diffusion}. Thus we omit any wall-clock timings. Full details for all the experiments are included in \cref{app:experiment}.

All experiment settings include a number of hyperparameter choices, some typical to all diffusion problems and some specific to particle filtering and smoothing. The diffusion $g(t)$ is a pre-determined function not optimized during training. We divide the experiments into two main subsets: problems of `sharpening to achieve a data distribution' and `optimal transport problems'. In the former, the initial distribution has a support overlapping with the terminal distribution and the process noise level $g(t)$ goes from high to low as time progresses. Conversely, in the latter setting, the particles sampled from the initial distribution must travel to reach the support of the terminal distribution, and we chose to use a constant process noise level. Perhaps the most significant choice of hyperparameter is the observational noise level, as it imposes a preference on how closely should the observational points be followed, see \cref{app:discussobs} for details.\looseness-1

\paragraph{2D toy examples}
We show illustrative results for the {\sc two moons} and {\sc circles} from scikit-learn. We add artificial observation data to bias the processes.  %
For the circles, the observational data consists of $10$ points, spaced evenly on the circle. %
The points are all observed simultaneously, at halfway through the process, forcing the marginal density of the generating SDE to collapse to the small circle, and then to expand.
For the {\sc two moons}, the observational data is collected from $10$ trajectories of a diffusion model, which generates the two moons from noise, and these $10$ trajectories are then observed at three points in time. Results are visualized in \cref{fig:2dtoy} (see videos in supplement). For reference, we have included plots of the IPFP dynamics in the supplement, see \cref{fig:2dtoyipfp}. \looseness-1

\begin{figure*}[t!]
  \centering\scriptsize
  \begin{tikzpicture}

      \setlength{\figurewidth}{.196\linewidth}
      \newcommand{\birds}[1]{\includegraphics[width=.95\figurewidth]{./fig/birds#1}}
      \newcommand{\dove}{\includegraphics[width=1em]{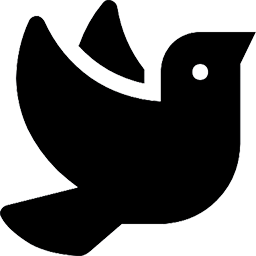}}

      \tikzset{cross/.style={cross out, draw=black, minimum size=2*(#1-\pgflinewidth), inner sep=0pt, outer sep=0pt, line width=1pt}, cross/.default={3pt}}

      \node (a) at (0\figurewidth,0) {\birds{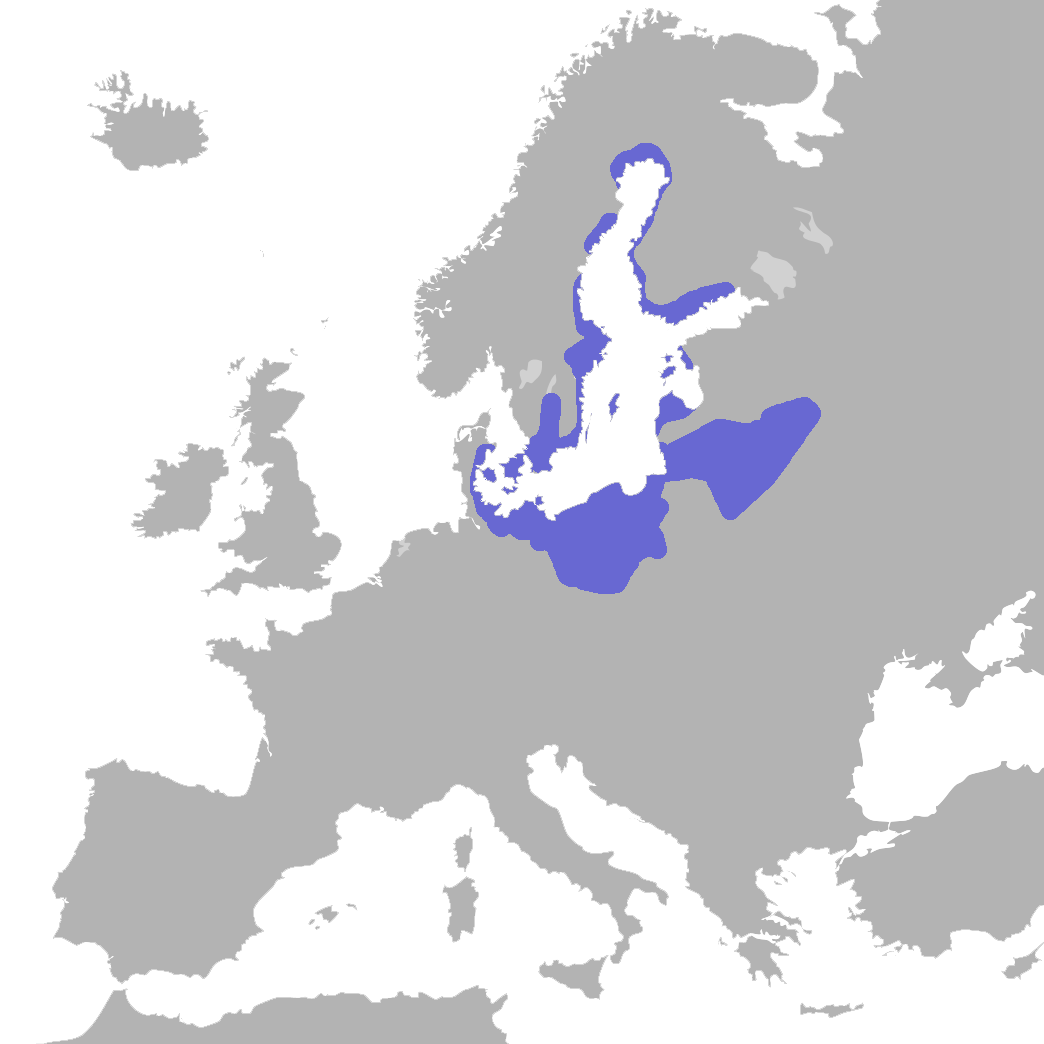}};
      \node (b) at (1\figurewidth,0) {\birds{}};
      \node (c) at (2\figurewidth,0) {\birds{}};
      \node (d) at (3\figurewidth,0) {\birds{}};
      \node (e) at (4\figurewidth,0) {\birds{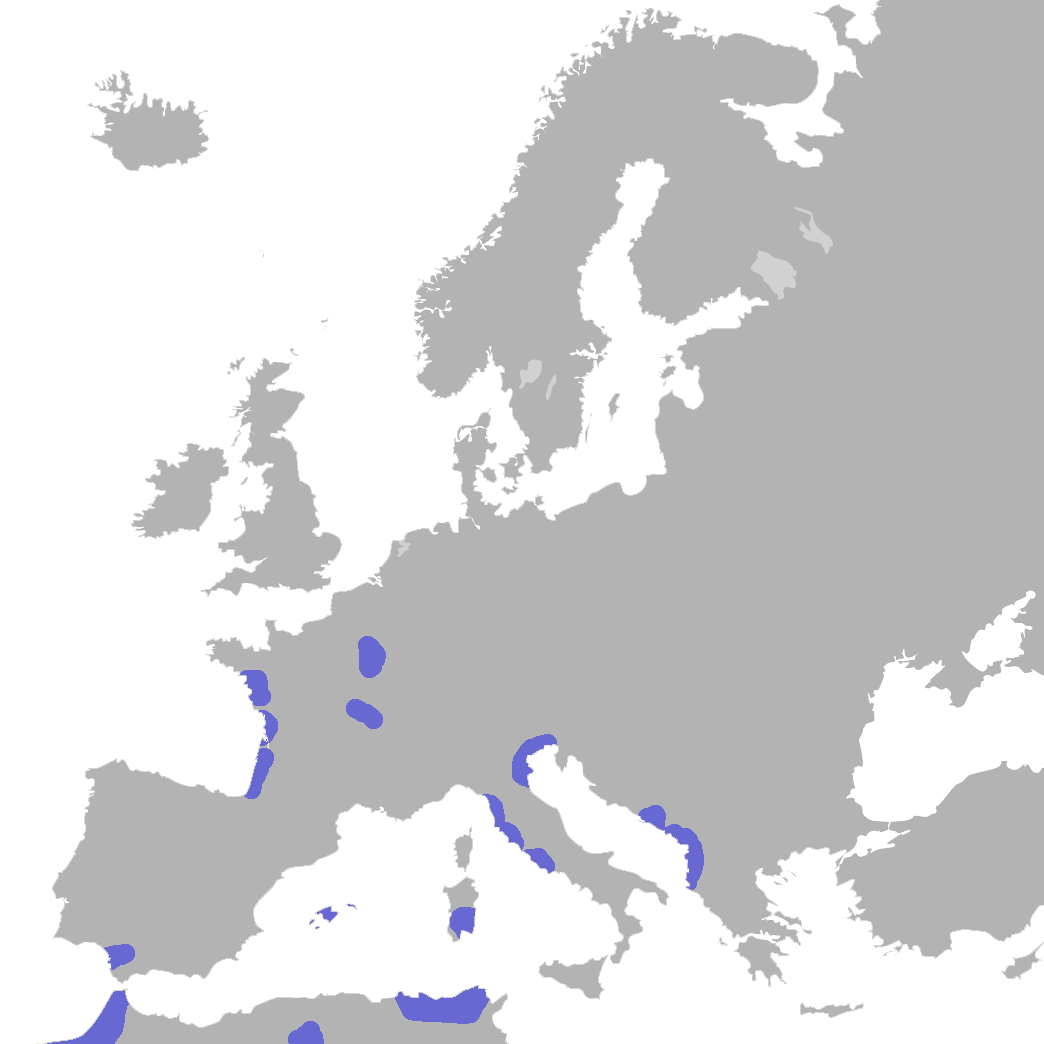}};

      \tikzstyle{label} = [anchor=south west,inner sep=2pt,outer sep=2pt,draw=black,rounded corners=2pt,fill=white,minimum width=1cm,align=center]
      \node[label] (la) at (a.north west) {Summer};
      \node[label,anchor=south east] (lc) at (e.north east) {Winter};
      \node[label,anchor=south] (lb) at (c.north) {Example sightings during migration};
      \draw[black,-latex',line width=1.2pt] (la) -- node[fill=white]{\tiny \dove} (lb) -- node[fill=white]{\tiny \dove} (lc);

      \node[rotate=90] at (-.52\figurewidth,0) {Bird observations};
      \node[rotate=90] at (-.52\figurewidth,-\figurewidth) {ISB result};
      \foreach \x/\y [count=\i] in {.65/.5,.55/.45} { %
        \node[cross] at ({.5\figurewidth+.95*\x*\figurewidth},{-.5\figurewidth+\y\figurewidth}) {\tiny};%
.       }

 \foreach \x/\y [count=\i] in {.5/.35} { %
        \node[cross] at ({1.5\figurewidth+.95*\x*\figurewidth},{-.5\figurewidth+\y\figurewidth}) {\tiny};%
.       }

 \foreach \x/\y [count=\i] in {.35/.3,.3/.2} { %
        \node[cross] at ({2.5\figurewidth+.95*\x*\figurewidth},{-.5\figurewidth+\y\figurewidth}) {\tiny};%
.       }

      \newcommand{\result}[1]{\includegraphics[width=.95\figurewidth]{./fig/birds/fit_model_europe_#1}}      
      \foreach \x [count=\i from 0] in {start,25,50,75,end}
        \node at (\i\figurewidth,-\figurewidth) {\result{\x}};
      
  \end{tikzpicture}
  \vspace*{-6pt}
  \caption{Bird migration example. The top row describes nesting and wintering areas and example sightings during migration. The bottom shows the marginal densities of the ISB model from the initial to terminal distribution, matching bird sightings along the migration.}
  \label{fig:birds}
\end{figure*}
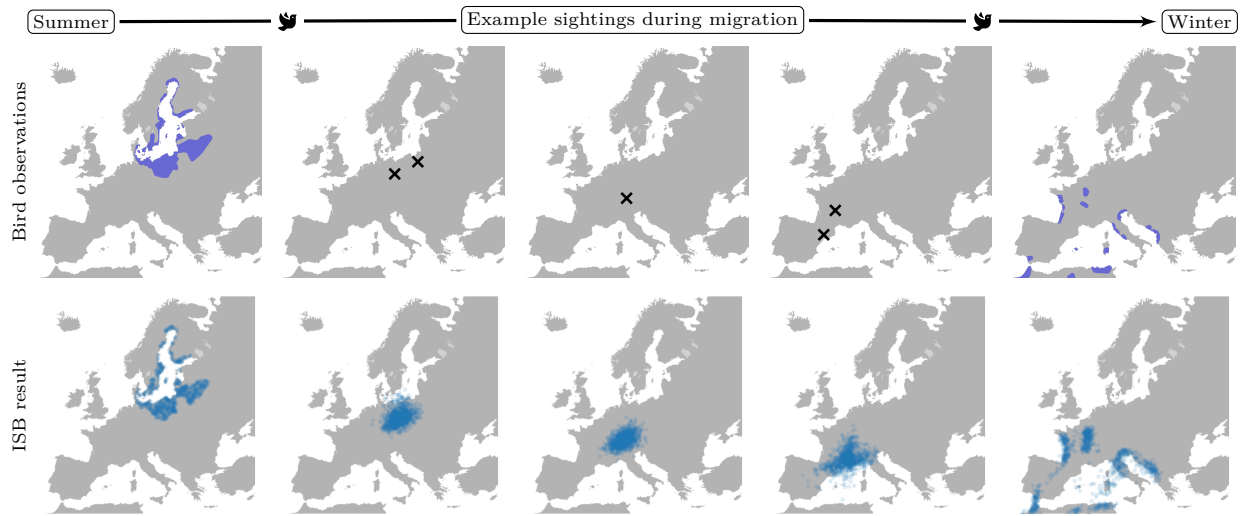

\paragraph{Quantitative comparison on the \Benes SDE}
In order to quantify how observing a process in between its initial and terminal states steers the ISB model to areas with higher likelihood, we test its performance on a \Benes SDE model \citep[see, \eg][]{Sarkka+Solin:2019}. The \Benes SDE is a non-linear one-dimensional SDE of form $\dd x_t = \tanh(x_t)\dd t + \dd\beta_t$ with $x_0 = 0$, but its marginal density is available in closed-form, allowing for negative log-likelihood evaluation.
We simulate trajectories from the \Benes SDE and from the reverse drift and stack the reversed trajectories. The terminal distribution is shifted and scaled so that the \Benes SDE itself does not solve the transport problem from $\pi_0$ to $\pi_T$, see \cref{app:benes} for details and visualizations of the processes.

We fit a Schrödinger bridge model with no observational data as a baseline, using the \Benes SDE drift as the reference model. The ISB model is initialized with a zero-drift model (not with the \Benes as reference), thus making learning more challenging. We compare the models in terms of negative log predictive density in \cref{tbl:benes}, where we see that the ISB model captures the process well on average (over the entire time-horizon) and at selected marginal times.

\paragraph{Bird migration}
Bird migration can be seen as a regular seasonal transport problem, where birds move (typically North--South) along a flyway, between breeding and wintering grounds. We take this as a motivating example of constrained optimal transport, where the geographical and constraints and preferred routes are accounted for by bird sighting data (see \cref{fig:birds} top). By adapting data from \citet{ambrosini2014} and \citet{pellegrino2015lack}, we propose a simplified data set for geese migration in Europe (OIBMD: ornithologically implausible bird migration data; available in the supplement). 
We applied the ISB for $12$ iterations, with a linear observation noise schedule from $1$ to $0.2$, and constant diffusion noise $0.05$. The drift function was initialized as a zero-function, and thus the method did not rely on a separately fit model optimized for generating the wintering distribution based on the breeding distribution. For comparison, we include the Schrödinger bridge results in \cref{app:birds}. \looseness-1

\paragraph{Constraining an image generation process}
We demonstrate that the ISB approach scales well to high-dimensional inputs by studying a proof-of-concept image generation task. We modify the diffusion generative process of the MNIST~\citep{MNIST} digit $8$ by artificial observations steering the dynamical system in the middle of the generation process. While the concept of observations in case of image generation is somewhat unnatural, it showcases the scalability of the method to high-dimensional data spaces. Here, the drift is initialized using a pre-trained neural network obtained by first running a Schrödinger bridge model for image generation. The process is then given an observation in the form of a bottom-half of a MNIST digit $8$ in the middle of the dynamical process. As the learned model uses information from the observation both before and after the observation time, the lower half of the image is sharper than the upper half.  We provide further details on this experiment and sampled trajectories in \cref{app:mnist}, and an ablation of multi-modal MNIST image generation in \cref{app:mnist-multi}

\begin{table}[t]
\begin{minipage}[t]{0.47\textwidth}
  \footnotesize
  \caption{Results for the \Benes experiment. We report the negative log predictive density (NLPD, lower better) of the \Benes marginal likelihood over generated particles at the initial and terminal distributions and at the middle of the transport process.\label{tbl:benes}}
  \renewcommand{\tabcolsep}{4pt}
  \begin{tabularx}{\linewidth}{l c c c}
    \toprule
    & \multicolumn{3}{c}{NLPD} \\
    \sc Method & \sc Average & \sc Middle & \sc End \\
    \midrule
    Schrödinger B  & $4.787$ & $3.565$ & $0.1919$ \\
    Iterative smoothing B & $\bf 3.557$ & $\bf 2.985$ & $\bf 0.1567$ \\
    \bottomrule
  \end{tabularx}  
\end{minipage}
\hspace*{\fill}
\begin{minipage}[t]{0.50\textwidth}
  \footnotesize
  \caption{Results for the single-cell embryo RNA experiment. We compare ISB to TrajectoryNet, IPML, and our implementation of IPFP. Unlike the other methods, our model is able to utilize the intermediate data distributions while training. \label{tbl:singlecell}}
  \begin{tabularx}{\linewidth}{l c c c c c}
    \toprule
    & \multicolumn{5}{c}{Earth mover's distance} \\
    \sc Method & \sc $t{=}0$ & \sc $t{=}1$ & \sc $t{=}2$ & \sc $t{=}3$  & \sc $t{=}T$ \\
    \midrule
    TrajectoryNet & $0.62$ & $1.15$ & $1.49$  & $1.26$  & $0.99$\\
    IPML & $\mathbf{0.34}$ & $1.13$ & $ 1.35$ & $1.01$ & $\mathbf{0.49}$\\
    IPFP (no obs) & $0.57$ & $1.53$ & $1.86$ & $1.32$ & $0.85$ \\
    ISB (single-cell obs) & $ 0.57$ &  $\mathbf{1.04}$ & $\mathbf{1.24}$ &  $\mathbf{0.94}$ & $0.83$ \\
    \bottomrule
  \end{tabularx}
\end{minipage}
\end{table}

\paragraph{Single-cell embryo RNA-seq}
Lastly, we evaluated our approach on an Embryoid body scRNA-seq time course \citep{tong2020trajectory}.
The data consists of RNA measurements collected over five time ranges from a developing human embryo system. 
No trajectory information is available, instead we only have access to snapshots of RNA data. This leads to a data set over $5$ time ranges, the first from days 0--3 and the last from days 15--18. 
In the experiment, we followed the protocol by \citet{tong2020trajectory}, reduced the data dimensionality to $d=5$ using PCA, and used the first and last time ranges as the initial and terminal constraints. All other time ranges are considered observational data.  Contrary to the other experiments, intermediate data are imprecise (only a time range of multiple days is known) but abundant.\looseness-1

We learned the ISB using a zero drift and compared it against an unconditional bridge obtained through the IPFP \citep{bortoli2021diffusion}---see \cref{fig:singlecell}. The ISB learns to generate trajectories with marginals closer to the observed data while performing comparably to the IPFP at the initial and terminal stages. This improvement is also verified numerically in \cref{tbl:singlecell}, showing that the ISB obtains a lower Earth mover's distance between the generated marginals and the observational data than IPFP. Additionally, \cref{tbl:singlecell} lists the performance of previous works that do not use the intermediate data during training \citep{tong2020trajectory} or only use it to construct an informative reference drift \citep{vargas2021solving}, see \cref{app:singlecell} for details.
In both cases, ISB outperforms the other approaches w.r.t.\ the intermediate marginal distributions ($t=1,2,3$), while IPML \citep{vargas2021solving} outperforms ISB at the initial and terminal stages due to its data-driven reference drift. Notice that while we reduced the dimensionality via PCA to $5$ for fair comparisons to \citet{vargas2021solving}, the ISB model would also allow modelling the full state-space model, with observations in the high-dimensional gene space and a latent SDE.

\section{Discussion and Conclusion}
\label{sec:discussion}
The dynamic Schrödinger bridge problem provides an appealing setting for posing optimal transport problems as learning non-linear diffusion processes and enables efficient iterative solvers.
However, while recent works have state-of-the-art performance in many complex application domains, they are typically limited to learning bridges with only initial and terminal constraints dependent on observed data.
In this work, we have extended this paradigm and introduced the Iterative Smoothing Bridge (ISB), an iterative algorithm for generating data-conditional smoothing bridges.
For this, we leveraged the strong connections between the constrained bridging problem and particle filtering in sequential Monte Carlo, extending them from pure inference to learning. 
We thoroughly assessed the applicability and flexibility of our approach in various experimental settings, including synthetic data sets and complex real-world scenarios (\eg, bird migration, conditional image generation, and modelling single-cell RNA-sequencing time-series). 
Our experiments showed that ISB generalizes well to high-dimensional data, is computationally efficient, and provides accurate estimates of the marginals at initial, terminal, and intermediate times.

Accurately modelling the dynamics of complex systems under both path constraints induced by sparse observations and initial and terminal constraints is a key challenge in many application domains. 
These include biomedical applications, demographic modelling, and environmental dynamics, but also machine learning specific applications such as reinforcement learning, planning, and time-series modelling.
All these applications have in common that the dynamic nature of the problem is driven by the progression of time, and not only the progression of a generative process as often is the case in, \eg, generative image models. Thus, constraints over intermediate stages have a natural role and interpretation in this wider set of dynamic diffusion modelling applications.
We believe the proposed ISB algorithm opens up new avenues for diffusion models in relevant real-world modelling tasks and will be stimulating for future work.
Recent work suggest close connections between sequential DA, commonly applied in many real-world scientific domains, and the Schr\"odinger bridge problem \citep{reich2019data} further emphasising the potential for future scientific applications of our work by exploiting these links.
Moreover, in the future more sophisticated observational models, alternative strategies to account for multiple observations, and different noise schedules could be explored. 
Lastly, the proposed approach could naturally be extended to other types of optimal transport problems, such as the Wasserstein barycenter, a frequently employed case of the multi-marginal optimal transport problem.

A reference implementation of the ISB model can be found at \url{https://github.com/AaltoML/iterative-smoothing-bridge}.

\begin{figure*}[t!]
    \centering\footnotesize
    \begin{tikzpicture}[scale=1]
        \node[anchor=south west] (I) at (0,-0.3cm) {\includegraphics[trim=250 80 220 80, clip, width=0.4\textwidth]{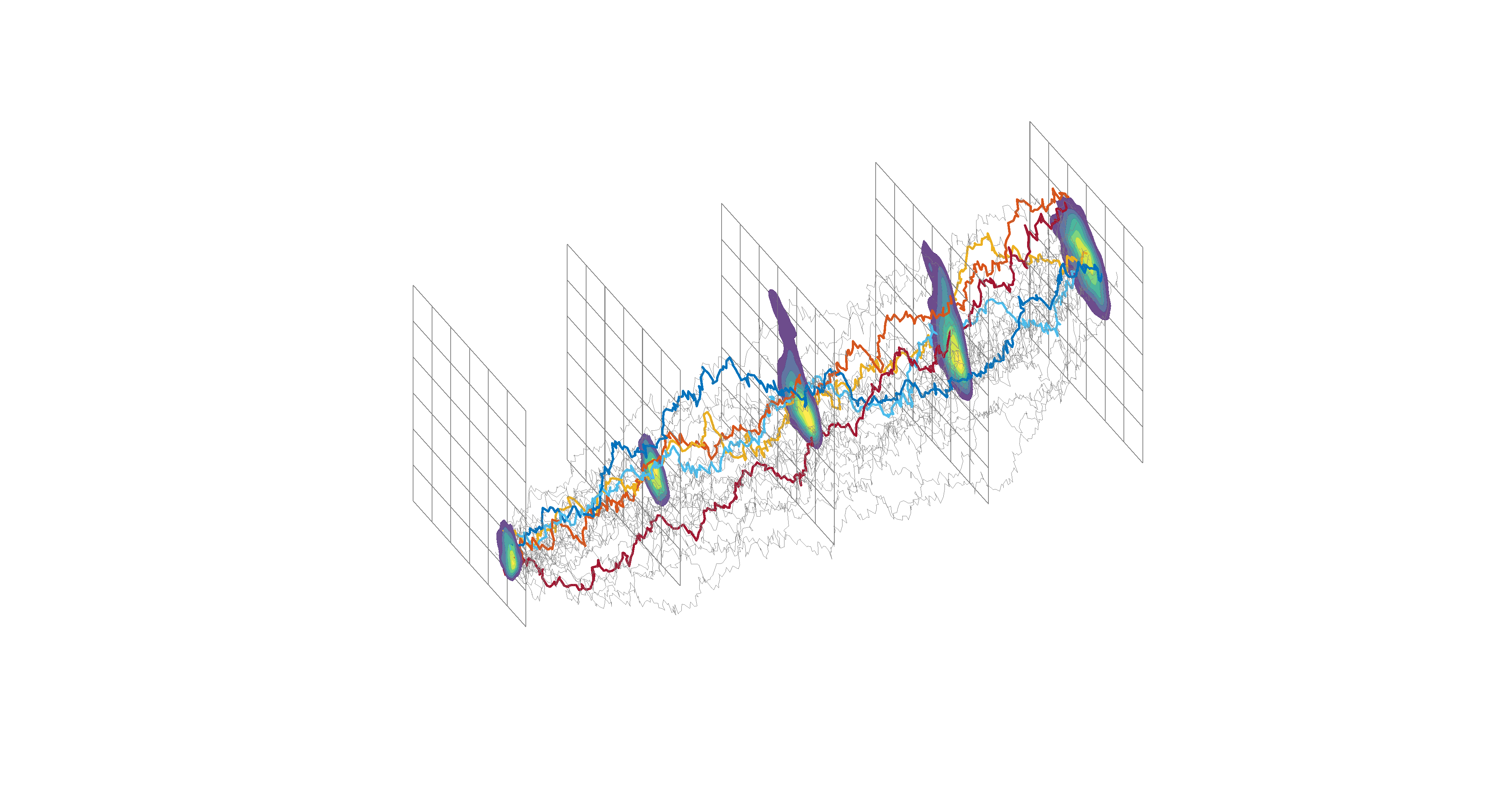}};

        \coordinate (a) at (.1,1.3);
        \draw[black,thick,-latex',->] (a) -- node[midway,rotate=90,yshift=6pt]{\tiny Principal axis \#1}++(0,2);
        \draw[black,thick,-latex',->] (a) -- node[midway,rotate=-45,yshift=-6pt]{\tiny PA \#2} ++(1.1,-1.2);

        \foreach \x [count=\i] in {0,1,2,3,T} {
          \node[rotate=-45] at (-.4+1.35*\i,2.5+\i*.36) {\tiny $t{=}\x$};          
        }

        \node at (3.5,-.25) {(a) Schrödinger bridge (via IPFP)};
        
    \end{tikzpicture}
    \hspace*{1cm}
    \begin{tikzpicture}[scale=1]
        \node[anchor=south west] (I) at (0,-0.3cm) {\includegraphics[trim=250 80 220 80, clip, width=0.4\textwidth]{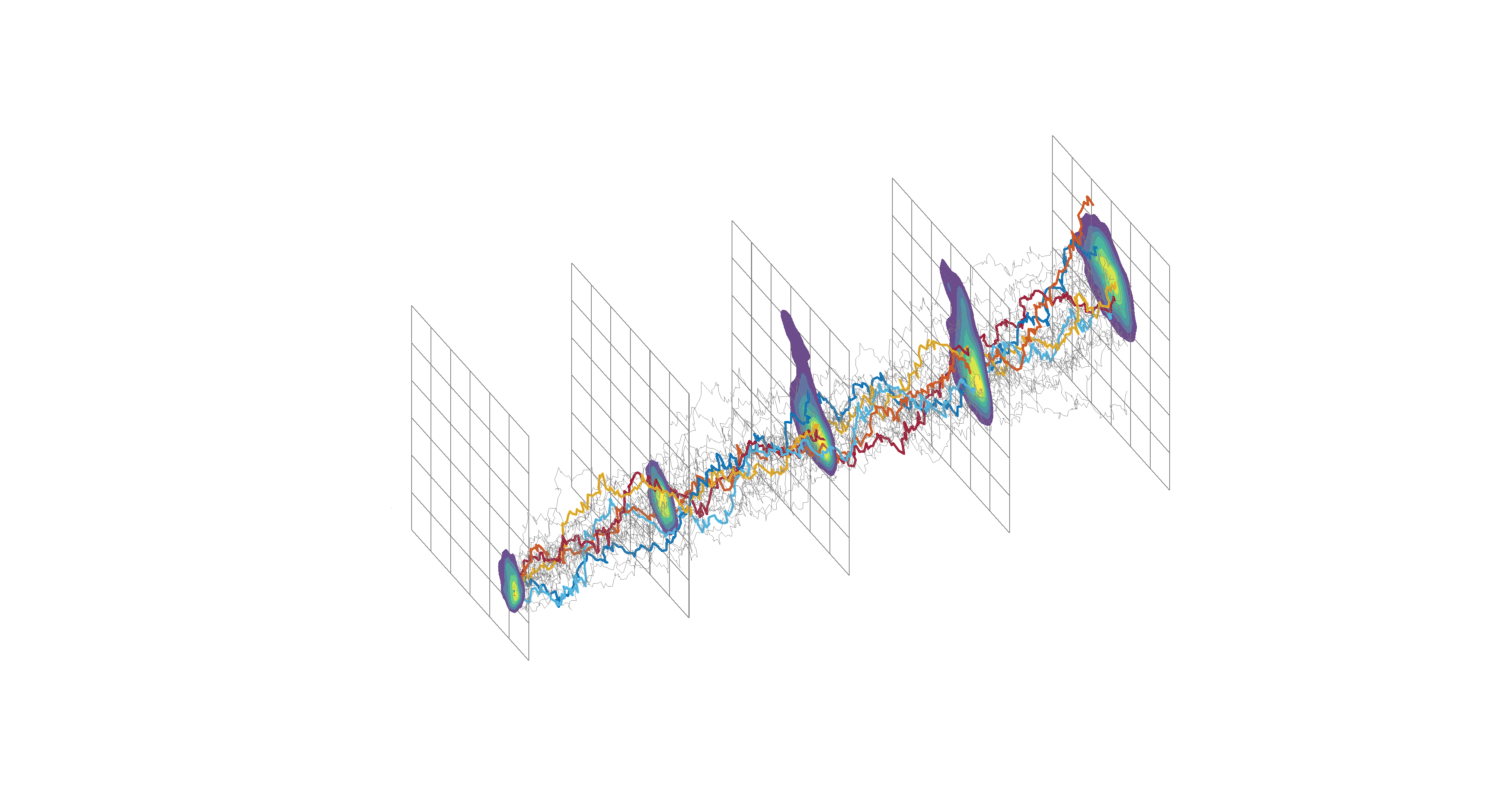}};

        \coordinate (a) at (.1,1.1);
        \draw[black,thick,-latex',->] (a) -- node[midway,rotate=90,yshift=6pt]{\tiny Principal axis \#1}++(0,2);
        \draw[black,thick,-latex',->] (a) -- node[midway,rotate=-45,yshift=-6pt]{\tiny PA \#2} ++(1.1,-1.2);

        \foreach \x [count=\i] in {0,1,2,3,T} {
          \node[rotate=-45] at (-.45+1.39*\i,2.35+\i*.365) {\tiny $t{=}\x$};          
        }

        \node at (3.5,-.25) {(b) Iterative Smoothing Bridge};
        
    \end{tikzpicture}
    \vspace*{-6pt}
    \caption{Illustration of the trajectories of the single-cell experiment for the Schrödinger bridge (a) and the ISB (b), projected onto the first two principal components. The first five trajectories are highlighted in colour, and intermediate observation densities visualized as slices. \label{fig:singlecell}}
\end{figure*}
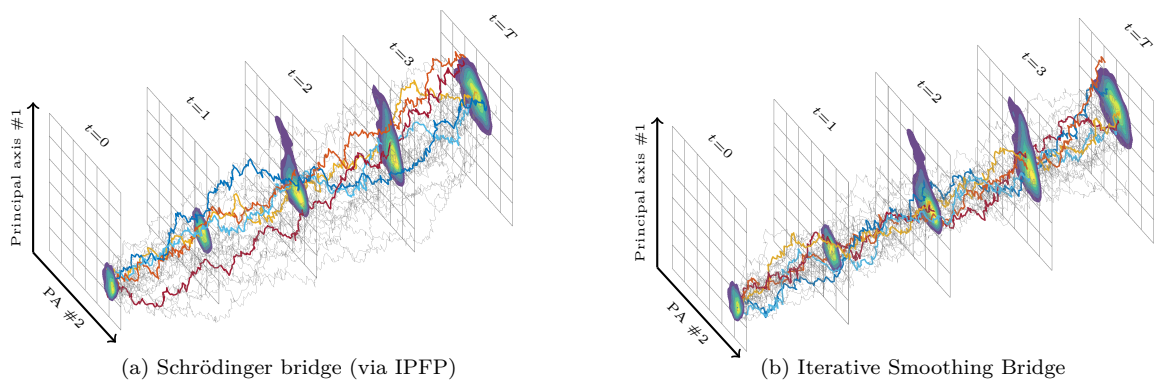

\section*{Acknowledgements and Disclosure of Funding}
Authors acknowledge funding from the Academy of Finland (grants 339730, 324345, and 347279). We also acknowledge the computational resources provided by the Aalto Science-IT project and CSC -- IT Center for Science, Finland. We wish to thank Adrien Corenflos for sharing an implementation of differentiable resampling in PyTorch, and Prakhar Verma for  comments on the manuscript.

\bibliographystyle{tmlr}

\begin{thebibliography}{69}
\providecommand{\natexlab}[1]{#1}
\providecommand{\url}[1]{\texttt{#1}}
\expandafter\ifx\csname urlstyle\endcsname\relax
  \providecommand{\doi}[1]{doi: #1}\else
  \providecommand{\doi}{doi: \begingroup \urlstyle{rm}\Url}\fi

\bibitem[Ambrosini et~al.(2014)Ambrosini, Borgoni, Rubolini, Sicurella,
  Fiedler, Bairlein, Baillie, Robinson, Clark, Spina, and Saino]{ambrosini2014}
Roberto Ambrosini, Riccardo Borgoni, Diego Rubolini, Beatrice Sicurella,
  Wolfgang Fiedler, Franz Bairlein, Stephen~R. Baillie, Robert~A. Robinson,
  Jacquie~A. Clark, Fernando Spina, and Nicola Saino.
\newblock Modelling the progression of bird migration with conditional
  autoregressive models applied to ringing data.
\newblock \emph{PLoS ONE}, 9\penalty0 (7):\penalty0 1--10, 07 2014.

\bibitem[Asch et~al.(2016)Asch, Bocquet, and Nodet]{asch2016data}
Mark Asch, Marc Bocquet, and Ma{\"e}lle Nodet.
\newblock \emph{Data Assimilation: Methods, Algorithms, and Applications}.
\newblock SIAM, 2016.

\bibitem[Beiser et~al.(2023)Beiser, Holm, and Eidsvik]{beiser2023comparison}
Florian Beiser, H{\aa}vard~Heitlo Holm, and Jo~Eidsvik.
\newblock Comparison of ensemble-based data assimilation methods for sparse
  oceanographic data.
\newblock \emph{arXiv preprint arXiv:2302.07197}, 2023.

\bibitem[Berner et~al.(2022)Berner, Richter, and Ullrich]{berner2022an}
Julius Berner, Lorenz Richter, and Karen Ullrich.
\newblock An optimal control perspective on diffusion-based generative
  modeling.
\newblock In \emph{NeurIPS 2022 Workshop on Score-Based Methods}, 2022.

\bibitem[Bernton et~al.(2019)Bernton, Heng, Doucet, and
  Jacob]{bernton2019schrodinger}
Espen Bernton, Jeremy Heng, Arnaud Doucet, and Pierre~E. Jacob.
\newblock Schr\"odinger bridge samplers.
\newblock \emph{arXiv preprint arXiv:1912.13170}, 2019.

\bibitem[Bresler(1986)]{bresler1986twofilter}
Yoram Bresler.
\newblock Two-filter formulae for discrete-time non-linear {B}ayesian
  smoothing.
\newblock \emph{International Journal of Control}, 43\penalty0 (2):\penalty0
  629--641, 1986.

\bibitem[Briers et~al.(2009)Briers, Doucet, and Maskell]{briers2009smoothing}
Mark Briers, Arnaud Doucet, and Simon Maskell.
\newblock Smoothing algorithms for state–space models.
\newblock \emph{Annals of the Institute of Statistical Mathematics},
  62:\penalty0 61--89, 02 2009.

\bibitem[Bunne et~al.(2022)Bunne, Hsieh, Cuturi, and
  Krause]{bunne2022recovering}
Charlotte Bunne, Ya{-}Ping Hsieh, Marco Cuturi, and Andreas Krause.
\newblock Recovering stochastic dynamics via {G}aussian {S}chr{\"{o}}dinger
  bridges.
\newblock \emph{arXiv preprint arXiv:2202.05722}, 2022.

\bibitem[Caluya \& Halder(2021)Caluya and Halder]{caluya2021constraint}
Kenneth~F. Caluya and Abhishek Halder.
\newblock Reflected {S}chr{\"{o}}dinger bridge: Density control with path
  constraints.
\newblock In \emph{American Control Conference, {ACC} 2021}, pp.\  1137--1142.
  {IEEE}, 2021.

\bibitem[Caluya \& Halder(2022)Caluya and Halder]{caluya2021wasserstein}
Kenneth~F. Caluya and Abhishek Halder.
\newblock Wasserstein proximal algorithms for the {S}chr{\"{o}}dinger bridge
  problem: Density control with nonlinear drift.
\newblock \emph{{IEEE} Transactions on Automatic Control}, 67\penalty0
  (3):\penalty0 1163--1178, 2022.

\bibitem[Chen et~al.(2021)Chen, Georgiou, and Pavon]{chen2021stochastic}
Yongxin Chen, Tryphon~T. Georgiou, and Michele Pavon.
\newblock Stochastic control liaisons: {R}ichard {S}inkhorn meets {G}aspard
  {M}onge on a {S}chr{\"{o}}dinger {B}ridge.
\newblock \emph{{SIAM} Review}, 63\penalty0 (2):\penalty0 249--313, 2021.

\bibitem[Chopin \& Papaspiliopoulos(2020)Chopin and
  Papaspiliopoulos]{chopin2020introduction}
Nicolas Chopin and Omiros Papaspiliopoulos.
\newblock \emph{An Introduction to Sequential Monte Carlo}.
\newblock Springer, 2020.

\bibitem[Corenflos et~al.(2021)Corenflos, Thornton, Deligiannidis, and
  Doucet]{corenflos2021diffpf}
Adrien Corenflos, James Thornton, George Deligiannidis, and Arnaud Doucet.
\newblock Differentiable particle filtering via entropy-regularized optimal
  transport.
\newblock In \emph{Proceedings of the 38th International Conference on Machine
  Learning, {ICML} 2021}, volume 139 of \emph{Proceedings of Machine Learning
  Research}, pp.\  2100--2111. {PMLR}, 2021.

\bibitem[Cuturi(2013)]{cuturi2013sinkhorn}
Marco Cuturi.
\newblock Sinkhorn distances: Lightspeed computation of optimal transport.
\newblock In \emph{Advances in Neural Information Processing Systems 26
  (NIPS)}, pp.\  2292--2300. Curran Associates, Inc., 2013.

\bibitem[De~Bortoli et~al.(2021)De~Bortoli, Thornton, Heng, and
  Doucet]{bortoli2021diffusion}
Valentin De~Bortoli, James Thornton, Jeremy Heng, and Arnaud Doucet.
\newblock Diffusion {S}chr{\"o}dinger bridge with applications to score-based
  generative modeling.
\newblock In \emph{Advances in Neural Information Processing Systems 34
  (NeurIPS)}, pp.\  17695--17709. Curran Associates, Inc., 2021.

\bibitem[Dhariwal \& Nichol(2021)Dhariwal and Nichol]{dhariwal2021diffusion}
Prafulla Dhariwal and Alexander~Quinn Nichol.
\newblock Diffusion models beat {GAN}s on image synthesis.
\newblock In \emph{Advances in Neural Information Processing Systems 35
  (NeurIPS)}, pp.\  8780--8794. Curran Associates, Inc., 2021.

\bibitem[Doucet et~al.(2000)Doucet, Godsill, and Andrieu]{doucet2000smc}
Arnaud Doucet, Simon Godsill, and Christophe Andrieu.
\newblock On sequential {M}onte {C}arlo sampling methods for {B}ayesian
  filtering.
\newblock \emph{Statistics and Computing}, 10\penalty0 (3):\penalty0 197--208,
  2000.

\bibitem[Doucet et~al.(2001)Doucet, De~Freitas, and
  Gordon]{doucet2001sequential}
Arnaud Doucet, Nando De~Freitas, and Neil~James Gordon.
\newblock \emph{Sequential {M}onte {C}arlo methods in practice}.
\newblock Statistics for Engineering and Information Science. Springer, 2001.

\bibitem[Doucet et~al.(2022)Doucet, Grathwohl, Matthews, and
  Strathmann]{doucet2022scorebased}
Arnaud Doucet, Will~Sussman Grathwohl, Alexander G. D.~G. Matthews, and Heiko
  Strathmann.
\newblock Score-based diffusion meets annealed importance sampling.
\newblock In \emph{Advances in Neural Information Processing Systems 35}, 2022.

\bibitem[Fernandes et~al.(2021)Fernandes, Vargas, Ek, and
  Campbell]{fernandes2021shooting}
David~Lopes Fernandes, Francisco Vargas, Carl~Henrik Ek, and Neill~DF Campbell.
\newblock Shooting {S}chr{\"o}dinger’s cat.
\newblock In \emph{Proceedings of the Fourth Symposium on Advances in
  Approximate Bayesian Inference (AABI)}, 2021.

\bibitem[F{\"o}llmer(1988)]{follmer1988random}
Hans F{\"o}llmer.
\newblock Random fields and diffusion processes.
\newblock In \emph{{\'E}cole d'{\'E}t{\'e} de Probabilit{\'e}s de Saint-Flour
  XV--XVII, 1985--87}, pp.\  101--203. Springer, 1988.

\bibitem[Haussmann \& Pardoux(1986)Haussmann and Pardoux]{haussmann1986time}
Ulrich~G Haussmann and Etienne Pardoux.
\newblock Time reversal of diffusions.
\newblock \emph{The Annals of Probability}, pp.\  1188--1205, 1986.

\bibitem[Ho et~al.(2020)Ho, Jain, and Abbeel]{ho2020denoise}
Jonathan Ho, Ajay Jain, and Pieter Abbeel.
\newblock Denoising diffusion probabilistic models.
\newblock In \emph{Advances in Neural Information Processing Systems 33
  (NeurIPS)}, pp.\  6840--6851. Curran Associates, Inc., 2020.

\bibitem[Hostettler(2015)]{hostettler2015twofilter}
Roland Hostettler.
\newblock A two filter particle smoother for {W}iener state-space systems.
\newblock In \emph{2015 IEEE Conference on Control Applications (CCA)}, pp.\
  412--417, 2015.

\bibitem[Hyv{\"a}rinen \& Dayan(2005)Hyv{\"a}rinen and
  Dayan]{hyvarinen2005estimation}
Aapo Hyv{\"a}rinen and Peter Dayan.
\newblock Estimation of non-normalized statistical models by score matching.
\newblock \emph{Journal of Machine Learning Research}, 6\penalty0 (4):\penalty0
  695--709, 2005.

\bibitem[Janner et~al.(2022)Janner, Du, Tenenbaum, and
  Levine]{Janner2022planning}
Michael Janner, Yilun Du, Joshua~B. Tenenbaum, and Sergey Levine.
\newblock Planning with diffusion for flexible behavior synthesis.
\newblock In Kamalika Chaudhuri, Stefanie Jegelka, Le~Song, Csaba
  Szepesv{\'{a}}ri, Gang Niu, and Sivan Sabato (eds.), \emph{International
  Conference on Machine Learning {(ICML)}}, volume 162, pp.\  9902--9915.
  {PMLR}, 2022.

\bibitem[Kokkala et~al.(2014)Kokkala, Solin, and Särkkä]{kokkala2014em}
Juho Kokkala, Arno Solin, and Simo Särkkä.
\newblock Expectation maximization based parameter estimation by sigma-point
  and particle smoothing.
\newblock In \emph{Proceedings of the 17th International Conference on
  Information Fusion (FUSION)}, pp.\  1--8, 2014.

\bibitem[Kullback(1968)]{kullback1968probability}
Solomon Kullback.
\newblock Probability densities with given marginals.
\newblock \emph{The Annals of Mathematical Statistics}, 39\penalty0
  (4):\penalty0 1236--1243, 1968.

\bibitem[LeCun et~al.(1998)LeCun, Cortes, and Burges]{MNIST}
Yann LeCun, Corinna Cortes, and Christopher~J.C. Burges.
\newblock The {MNIST} database of handwritten digits, 1998.
\newblock URL \url{http://yann.lecun.com/exdb/mnist/}.

\bibitem[Li et~al.(2020)Li, Wong, Chen, and Duvenaud]{li2020sdeadjoint}
Xuechen Li, Ting-Kam~Leonard Wong, Ricky T.~Q. Chen, and David Duvenaud.
\newblock Scalable gradients for stochastic differential equations.
\newblock In \emph{Proceedings of the Twenty Third International Conference on
  Artificial Intelligence and Statistics (AISTATS)}, volume 108 of
  \emph{Proceedings of Machine Learning Research}, pp.\  3870--3882. PMLR,
  2020.

\bibitem[Liu et~al.(2022)Liu, Chen, So, and Theodorou]{liu2022deep}
Guan-Horng Liu, Tianrong Chen, Oswin So, and Evangelos Theodorou.
\newblock Deep generalized schr\"odinger bridge.
\newblock In \emph{Advances in Neural Information Processing Systems 35}, 2022.

\bibitem[Léonard(2014)]{leonard2014schrodi}
Christian Léonard.
\newblock A survey of the {S}chrödinger problem and some of its connections
  with optimal transport.
\newblock \emph{Discrete \& Continuous Dynamical Systems}, 34\penalty0
  (4):\penalty0 1533--1574, 2014.

\bibitem[Maoutsa(2023)]{maoutsa2023geometric}
Dimitra Maoutsa.
\newblock Geometric constraints improve inference of sparsely observed
  stochastic dynamics.
\newblock \emph{arXiv preprint arXiv:2304.00423}, 2023.

\bibitem[Maoutsa \& Opper(2022)Maoutsa and Opper]{maoutsa2021deterministic}
Dimitra Maoutsa and Manfred Opper.
\newblock Deterministic particle flows for constraining stochastic nonlinear
  systems.
\newblock \emph{Physical Review Research}, 4\penalty0 (4), 2022.

\bibitem[Maoutsa et~al.(2020)Maoutsa, Reich, and Opper]{maoutsa2020interacting}
Dimitra Maoutsa, Sebastian Reich, and Manfred Opper.
\newblock Interacting particle solutions of {F}okker–{P}lanck equations
  through gradient–log–density estimation.
\newblock \emph{Entropy}, 22:\penalty0 802, 07 2020.

\bibitem[Mendez et~al.(2020)Mendez, Hoffman, Cherry, Lemmon, and
  Weinberg]{mendez2020cell}
Mario~J Mendez, Matthew~J Hoffman, Elizabeth~M Cherry, Christopher~A Lemmon,
  and Seth~H Weinberg.
\newblock Cell fate forecasting: {A} data-assimilation approach to predict
  epithelial-mesenchymal transition.
\newblock \emph{Biophysical Journal}, 118\penalty0 (7):\penalty0 1749--1768,
  2020.

\bibitem[Mitter(1996)]{mitter1996filtering}
S.K. Mitter.
\newblock Filtering and stochastic control: a historical perspective.
\newblock \emph{IEEE Control Systems Magazine}, 16\penalty0 (3):\penalty0
  67--76, 1996.

\bibitem[Nichol \& Dhariwal(2021)Nichol and Dhariwal]{nichol2021ImprovedDD}
Alexander~Quinn Nichol and Prafulla Dhariwal.
\newblock Improved denoising diffusion probabilistic models.
\newblock In \emph{Proceedings of the 38th International Conference on Machine
  Learning (ICML)}, volume 139 of \emph{Proceedings of Machine Learning
  Research}, pp.\  8162--8171. PMLR, 2021.

\bibitem[{\O}ksendal(2003)]{Oksendal:2003}
Bernt {\O}ksendal.
\newblock \emph{Stochastic Differential Equations: {A}n Introduction with
  Applications}.
\newblock Springer, New York, NY, sixth edition, 2003.

\bibitem[Park et~al.(2022)Park, Lee, and Kwon]{Park2022sde}
Sung~Woo Park, Kyungjae Lee, and Junseok Kwon.
\newblock Neural markov controlled {SDE:} {S}tochastic optimization for
  continuous-time data.
\newblock In \emph{International Conference on Learning Representations
  {(ICLR)}}, 2022.

\bibitem[Pellegrino et~al.(2015)Pellegrino, Cucco, Follestad, and
  Boos]{pellegrino2015lack}
Irene Pellegrino, Marco Cucco, Arne Follestad, and Mathieu Boos.
\newblock Lack of genetic structure in greylag goose ({A}nser anser)
  populations along the {E}uropean {A}tlantic flyway.
\newblock \emph{PeerJ}, 3:\penalty0 e1161, 2015.

\bibitem[Peyr{\'{e}} \& Cuturi(2019)Peyr{\'{e}} and
  Cuturi]{peyre2017computational}
Gabriel Peyr{\'{e}} and Marco Cuturi.
\newblock Computational optimal transport.
\newblock \emph{Foundations and Trends in Machine Learning}, 11\penalty0
  (5-6):\penalty0 355--607, 2019.

\bibitem[Pitt \& Shephard(1999)Pitt and Shephard]{pitt1999filtering}
Michael~K. Pitt and Neil Shephard.
\newblock Filtering via simulation: Auxiliary particle filters.
\newblock \emph{Journal of the American Statistical Association}, 94\penalty0
  (446):\penalty0 590--599, 1999.

\bibitem[Rasul et~al.(2021)Rasul, Sheikh, Schuster, Bergmann, and
  Vollgraf]{Rasul2021timeseries}
Kashif Rasul, Abdul{-}Saboor Sheikh, Ingmar Schuster, Urs~M. Bergmann, and
  Roland Vollgraf.
\newblock Multivariate probabilistic time series forecasting via conditioned
  normalizing flows.
\newblock In \emph{International Conference on Learning Representations
  {(ICLR)}}, 2021.

\bibitem[Reich(2013)]{reich2013nonparametric}
Sebastian Reich.
\newblock A nonparametric ensemble transform method for {B}ayesian inference.
\newblock \emph{SIAM Journal on Scentific Computing}, 35, 2013.

\bibitem[Reich(2019)]{reich2019data}
Sebastian Reich.
\newblock Data assimilation: {T}he {S}chrödinger perspective.
\newblock \emph{Acta Numerica}, 28:\penalty0 635--711, 2019.

\bibitem[Rubanova et~al.(2019)Rubanova, Chen, and
  Duvenaud]{rubanova2019latentode}
Yulia Rubanova, Tian~Qi Chen, and David~K Duvenaud.
\newblock Latent ordinary differential equations for irregularly-sampled time
  series.
\newblock In \emph{Advances in Neural Information Processing Systems 32
  (NeurIPS)}, pp.\  5321--5331. Curran Associates, Inc., 2019.

\bibitem[Ruschendorf(1995)]{ruschendorf1995convergence}
Ludger Ruschendorf.
\newblock Convergence of the iterative proportional fitting procedure.
\newblock \emph{The Annals of Statistics}, 23\penalty0 (4):\penalty0
  1160--1174, 1995.

\bibitem[Rüschendorf \& Thomsen(1993)Rüschendorf and
  Thomsen]{ruschendorf1993note}
L.~Rüschendorf and W.~Thomsen.
\newblock Note on the {S}chr\"odinger equation and {I}-projections.
\newblock \emph{Statistics \& Probability Letters}, 17\penalty0 (5):\penalty0
  369--375, 1993.

\bibitem[S{\"a}rkk{\"a}(2013)]{sarkka2013bayesian}
Simo S{\"a}rkk{\"a}.
\newblock \emph{Bayesian Filtering and Smoothing}.
\newblock Cambridge University Press, Cambridge, UK, 2013.

\bibitem[S{\"a}rkk{\"a} \& Solin(2019)S{\"a}rkk{\"a} and
  Solin]{Sarkka+Solin:2019}
Simo S{\"a}rkk{\"a} and Arno Solin.
\newblock \emph{Applied Stochastic Differential Equations}.
\newblock Cambridge University Press, Cambridge, UK, 2019.

\bibitem[Schr\"odinger(1932)]{schrod1932surla}
E.~Schr\"odinger.
\newblock Sur la th\'eorie relativiste de l'\'electron et l'interpr\'etation de
  la m\'ecanique quantique.
\newblock \emph{Annales de l'institut Henri Poincar\'e}, 2\penalty0
  (4):\penalty0 269--310, 1932.

\bibitem[Schön et~al.(2011)Schön, Wills, and Ninness]{schon2011system}
Thomas~B. Schön, Adrian Wills, and Brett Ninness.
\newblock System identification of nonlinear state-space models.
\newblock \emph{Automatica}, 47\penalty0 (1):\penalty0 39--49, 2011.

\bibitem[Shi et~al.(2022)Shi, Bortoli, Deligiannidis, and
  Doucet]{Shi2022condSchroedinger}
Yuyang Shi, Valentin~De Bortoli, George Deligiannidis, and Arnaud Doucet.
\newblock Conditional simulation using diffusion {S}chr\"odinger bridges.
\newblock In \emph{38th Conference on Uncertainty in Artificial Intelligence}.
  UAI, 2022.

\bibitem[Shi et~al.(2023)Shi, Bortoli, Campbell, and Doucet]{shi2023diffusion}
Yuyang Shi, Valentin~De Bortoli, Andrew Campbell, and Arnaud Doucet.
\newblock Diffusion schr\"odinger bridge matching.
\newblock \emph{arXiv preprint arXiv:2303.16852}, 2023.

\bibitem[Song et~al.(2021{\natexlab{a}})Song, Durkan, Murray, and
  Ermon]{song2021maximum}
Yang Song, Conor Durkan, Iain Murray, and Stefano Ermon.
\newblock Maximum likelihood training of score-based diffusion models.
\newblock In \emph{Advances in Neural Information Processing Systems 35
  (NeurIPS)}, pp.\  1415--1428. Curran Associates, Inc., 2021{\natexlab{a}}.

\bibitem[Song et~al.(2021{\natexlab{b}})Song, Sohl-Dickstein, Kingma, Kumar,
  Ermon, and Poole]{song2020sdegen}
Yang Song, Jascha Sohl-Dickstein, Diederik~P. Kingma, Abhishek Kumar, Stefano
  Ermon, and Ben Poole.
\newblock Score-based generative modeling through stochastic differential
  equations.
\newblock In \emph{International Conference on Learning Representations
  (ICLR)}, 2021{\natexlab{b}}.

\bibitem[Song et~al.(2022)Song, Shen, Xing, and Ermon]{song2022solving}
Yang Song, Liyue Shen, Lei Xing, and Stefano Ermon.
\newblock Solving inverse problems in medical imaging with score-based
  generative models.
\newblock In \emph{International Conference on Learning Representations
  (ICLR)}, 2022.

\bibitem[Svensson \& Schön(2017)Svensson and Schön]{svensson2017flexible}
Andreas Svensson and Thomas~B. Schön.
\newblock A flexible state–space model for learning nonlinear dynamical
  systems.
\newblock \emph{Automatica}, 80:\penalty0 189--199, 2017.

\bibitem[Tashiro et~al.(2021)Tashiro, Song, Song, and Ermon]{tashiro2021csdi}
Yusuke Tashiro, Jiaming Song, Yang Song, and Stefano Ermon.
\newblock {CSDI}: {C}onditional score-based diffusion models for probabilistic
  time series imputation.
\newblock In \emph{Advances in Neural Information Processing Systems 35
  (NeurIPS)}, pp.\  24804--24816. Curran Associates, Inc., 2021.

\bibitem[Tianrong~Chen(2022)]{chen2022likelihood}
Evangelos A.~Theodorou Tianrong~Chen, Guan-Horng~Liu.
\newblock Likelihood training of schr\"odinger bridge using forward-backward
  {SDE}s theory.
\newblock In \emph{International Conference on Learning Representations
  (ICLR)}, 2022.

\bibitem[Todorov(2008)]{todorov2008general}
Emanuel Todorov.
\newblock General duality between optimal control and estimation.
\newblock In \emph{Proceedings of the 47th IEEE Conference on Decision and
  Control}, pp.\  4286--4292, 2008.

\bibitem[Tong et~al.(2020)Tong, Huang, Wolf, van Dijk, and
  Krishnaswamy]{tong2020trajectory}
Alexander Tong, Jessie Huang, Guy Wolf, David van Dijk, and Smita Krishnaswamy.
\newblock Trajectorynet: {A} dynamic optimal transport network for modeling
  cellular dynamics.
\newblock In \emph{{ICML}}, volume 119 of \emph{Proceedings of Machine Learning
  Research}, pp.\  9526--9536. {PMLR}, 2020.

\bibitem[Vargas et~al.(2021)Vargas, Thodoroff, Lamacraft, and
  Lawrence]{vargas2021solving}
Francisco Vargas, Pierre Thodoroff, Austen Lamacraft, and Neil Lawrence.
\newblock Solving {S}chr{\"o}dinger bridges via maximum likelihood.
\newblock \emph{Entropy}, 23\penalty0 (9):\penalty0 1134, 2021.

\bibitem[Vincent(2011)]{vincent2011connection}
Pascal Vincent.
\newblock A connection between score matching and denoising autoencoders.
\newblock \emph{Neural Computation}, 23\penalty0 (7):\penalty0 1661--1674,
  2011.

\bibitem[Wang et~al.(2000)Wang, Zou, and Zhu]{wang2000data}
Bin Wang, Xiaolei Zou, and Jiang Zhu.
\newblock Data assimilation and its applications.
\newblock \emph{Proceedings of the National Academy of Sciences}, 97\penalty0
  (21):\penalty0 11143--11144, 2000.

\bibitem[Wang et~al.(2021)Wang, Jiao, Xu, Wang, and Yang]{wang2021DeepGL}
Gefei Wang, Yuling Jiao, Qian Xu, Yang Wang, and Can Yang.
\newblock Deep generative learning via {S}chr{\"o}dinger bridge.
\newblock In \emph{Proceedings of the 38th International Conference on Machine
  Learning (ICML)}, volume 139 of \emph{Proceedings of Machine Learning
  Research}, pp.\  10794--10804. PMLR, 2021.

\bibitem[Whitaker et~al.(2009)Whitaker, Compo, and
  Th{\'e}paut]{whitaker2009comparison}
Jeffrey~S Whitaker, Gilbert~P Compo, and Jean-No{\"e}l Th{\'e}paut.
\newblock A comparison of variational and ensemble-based data assimilation
  systems for reanalysis of sparse observations.
\newblock \emph{Monthly Weather Review}, 137\penalty0 (6):\penalty0 1991--1999,
  2009.

\bibitem[Zhang et~al.(2022)Zhang, Zhu, and Lin]{zhang2022neural}
Jingdong Zhang, Qunxi Zhu, and Wei Lin.
\newblock Neural stochastic control.
\newblock In \emph{Advances in Neural Information Processing Systems 35}, 2022.

\end{thebibliography}

\clearpage

\onecolumn

\appendix
\section{Method Details}\label{app:maths}
We present the details of the objective function derivation in \cref{app:meanmatch} and explain the connection of the backward drift function to Hamilton--Jacobi equations in \cref{app:hamiltonjacobi}. In \cref{app:infinity}, we discuss the behaviour of our model at the limit $M \to \infty$, that is, when the observations fully represent the marginal densities of the stochastic process.
\subsection{Deriving the Mean-matching Loss at Observation Times}\label{app:meanmatch}
Recall that the forward loss is written as 
\begin{equation}\label{eq:appfullloss}
\ell(\phi) = \sum_{i=1}^N \left[ \sum_{k=1}^{K }\ell_{k, \text{obs}}^i (\phi) \mathbb{I}_{y_{t_k} \not = \emptyset} +\ell_{k, \text{nobs}}^i (\phi) \mathbb{I}_{y_{t_k} = \emptyset} \right],
\end{equation}
where the loss at observations  $\ell_{k, \text{obs}}^i (\phi)$ and loss elsewhere $\ell_{k, \text{nobs}}^i (\phi)$ are
\begin{equation}\label{eq:appbasicloss}
   \ell_{k+1,\text{nobs}}^i = \| b_{l,  \phi}(\vx_{t_{k+1}}^i, t_{k+1})\Delta_k - \vx_{t_{k+1}}^i  - f_{l-1, \theta}(\vx_{t_{k+1}}^i, t_k)\Delta_k + \vx_{t_k}^i +  f_{l-1, \theta}(\vx_{t_k}^i, t_k)\Delta_k \|^2, 
\end{equation}
\begin{multline}\label{eq:appobsloss}
  \ell_{k+1,\text{{obs}}}^i =  \|  b_{l,  \phi}(\vx_{t_{k+1}}^i, t_{k+1})\Delta_k - \vx_{t_{k+1}}^i - f_{l-1, \theta}(\vx_{t_{k+1}}^i, t_k)\Delta_k \\
  +\frac{1}{C_{\epsilon, i}}  \textstyle\sum_{n=1}^N T_{(\epsilon),i, n} \left( \vx_{t_k}^n + f_{l-1, \theta}(\vx_{t_k}^n, t_k)\Delta_k \right)\|^2, 
\end{multline}
For convenience, we state the backward loss functions which follow similarly to their forward versions. The backward loss is defined as
\begin{equation}\label{eq:appfulllossbwd}
\overleftarrow{\ell(\theta)} = \sum_{i=1}^N \left[ \sum_{k=1}^{K }\overleftarrow{\ell}_{k, \text{obs}}^i (\theta) \mathbb{I}_{y_{t_k} \not = \emptyset} +\overleftarrow{\ell}_{k, \text{nobs}}^i (\theta) \mathbb{I}_{y_{t_k} = \emptyset} \right],
\end{equation}
where the loss at observations  $\overleftarrow{\ell}_{k, \text{obs}}^i (\theta)$ and loss elsewhere $\overleftarrow{\ell}_{k, \text{nobs}}^i (\theta)$ are
\begin{equation}\label{eq:appbasiclossbwd}
   \overleftarrow{\ell}_{k+1,\text{nobs}}^i = \| f_{l,  \theta}(\vx_{t_{k+1}}^i, t_{k+1})\Delta_k - \vx_{t_{k+1}}^i     - b_{l, \theta}(\vx_{t_{k+1}}^i, t_k)\Delta_k + \vx_{t_k}^i +  b_{l, \theta}(\vx_{t_k}^i, t_k)\Delta_k \|^2,
\end{equation}
\begin{multline}\label{eq:appobslossbwd}
 \overleftarrow{\ell}_{k+1,\text{{obs}}}^i = \|  f_{l,  \theta}(\vx_{t_{k+1}}^i, t_{k+1})\Delta_k - \vx_{t_{k+1}}^i - b_{l, \phi}(\vx_{t_{k+1}}^i, t_k)\Delta_k \\
  +\frac{1}{C_{\epsilon, i}}  \textstyle\sum_{n=1}^N T_{(\epsilon),i, n} \left( \vx_{t_k}^n + b_{l, \phi}(\vx_{t_k}^n, t_k)\Delta_k \right)\|^2.
\end{multline}

\begin{prop}\label{prop:obsloss}
Define the forward SDE as
\begin{align}
  \dd \vx_t &= f_{l, \theta}(\vx_{t}, t)\dd t + g(t) \dd \vbeta_t, \quad &&\vx_0 \sim \pi_0, \label{eq:forwardSDEprop}
\end{align}
and a backward SDE drift as
\begin{equation}\label{eq:smoothreverse}
b_{l, \phi} (\vx_{t_{k+1}}, t_{k+1}) = f_{l-1, \theta}(\vx_{t_{k+1}}, t_k )- g(t_{k+1})^2 \nabla \ln p_{t_{k+1}},
\end{equation}
where $p_{t_{k+1}}$ is the particle filtering density after differential resampling at time $t_{k+1}$. Then $b_{l, \phi}(\vx_{t_{k+1}}, t_{k+1})$
minimizes the loss function
\begin{multline}\label{eq:propobsloss}
  \ell_{k+1,\text{{obs}}}^i =  \|  b_{l,  \phi}(\vx_{t_{k+1}}^i, t_{k+1})\Delta_k - \vx_{t_{k+1}}^i - f_{l-1, \theta}(\vx_{t_{k+1}}^i, t_k)\Delta_k \\
  +\frac{1}{C_{\epsilon, i}}  \textstyle\sum_{n=1}^N T_{(\epsilon),i, n} \left( \vx_{t_k}^n + f_{l-1, \theta}(\vx_{t_k}^n, t_k)\Delta_k \right)\|^2, 
\end{multline}
where we denote $\MC_{\epsilon, i} = \frac{1}{g(t_{k+1})^2 \Delta_k} \Var\left( \sum_{n=1}^N T_{(\epsilon), i, n} \tilde{\vx}_{t_{k+1}}^n\right)$, and $\{ \tilde{\vx}_{t_{k+1}}^i\}_{i=1}^N$ are the particles before resampling.
\end{prop}
\begin{proof}
Our objective is to find a backward drift function $b_{l, \phi}(\vx_{t_{k+1}}, t_{k+1})$ as in \cref{eq:smoothreverse}. 
Notice that at observation times $t_k$, this is not equivalent to finding the reverse drift of the SDE forward transition and differential resampling combined, since the drift function $f_{l-1, \theta}$ alone does not map the particles $\{ \vx_{t_{k}}^i\}_{i=1}^N$ to the particles  $\{ \vx_{t_{k+1}}^i\}_{i=1}^N$. We will derive a loss function for learning the backward drift as in \cref{eq:smoothreverse} below, leaving the discussion on why it is a meaningful choice of a backward drift to \cref{app:hamiltonjacobi}. Our derivation closely follows the proof of Proposition $3$ in \citet{bortoli2021diffusion}, but we provide the details here for the sake of completeness. 

First, we give the transition density $p_{\vx_{t_{k}} \mid \vx_{t_{k-1}}^i}(\vx_{k})$ and apply it to derive the observation time loss $\ell_{k, \text{obs}}^i$. The derivation for the loss $\ell_{k, \text{no obs}}^i$ is skipped since it is as in the proof of Proposition $3$ in \citet{bortoli2021diffusion}.
Suppose that at $t_k$, there are observations. By definition, the particles before resampling $\{ \tilde{\vx}_{t_{k+1}}^i\}_{i=1}^N$ are generated by the Gaussian transition density
\begin{equation}\label{eq:transfwd}
    p(\tilde{\vx}_{t_{k+1}} \mid \vx_{t_k}^i) = \N(\tilde{\vx}_{t_{k+1}} \mid \vx_{t_k}^i + \delta_k f_l(\vx_{t_k}^i,t_k ), g(t_{k+1})^2 \Delta_k \MI).
\end{equation}
Recall that the resampled particles are defined as a weighted average of all the particles, $\vx_{t_k}^i = \sum_{n=1}^N \tilde{\vx}_{t_k}^n \, T_{(\epsilon),i, n}$. Thus, the transition density from $\{ \vx_{t_{k}}^i\}_{i=1}^N$ to the particles  $\{ \vx_{t_{k+1}}^i\}_{i=1}^N$ is also a Gaussian,
\begin{equation}\label{eq:transobs}
 p(\vx_{t_{k+1}^i} \mid \vx_{t_{k}}^i) =  \N(\tilde{\vx}_{t_{k+1}} \mid \sum_{n=1}^N T_{(\epsilon), i, n}( \vx_{t_{k-1}}^n+ \Delta_k  f_{l-1, \theta}(\vx_{t_{k}}^n,t_{k} )), g(t_{k+1})^2 \Delta_k C_{\epsilon, i} \MI_d) .
\end{equation}
We will derive the loss function \cref{eq:obsloss} by modifying the mean matching proof in \citet{bortoli2021diffusion} by the transition mean \cref{eq:transobs} and the backward drift definition \cref{eq:smoothreverse}. Using the particle filtering approximation, the marginal density can be decomposed as $p_{t_{k+1}}(\vx_{k+1}) = \sum_{i=1}^N p_{t_k}(\vx_k^i) p_{\vx_{k+1} \mid \vx_{k}^i} (\vx_{k+1}) $. By substituting the transition density \cref{eq:transobs} it follows that
\begin{equation}\label{eq:loss_step1}
p_{t_{k+1}}(\vx_{t_{k+1}})  = \frac{1}{Z}\sum_{i=1}^N p_{t_k}(\vx_{t_k}^i) \exp \left(- \frac{\|\left(\sum_{n=1}^N T_{(\epsilon),i ,n} (\vx_{t_k}^i + f_{l-1, \theta}(\vx_{t_k}, t_k)) \right)  - \vx_{t_{k+1}} \|^2}{2g(t_{k+1})^2 C_{\epsilon, i}\Delta_k} \right),
\end{equation}
where $Z$ is the normalization constant of \cref{eq:transobs}. As in the proof of Proposition $3$ of \citet{bortoli2021diffusion}, we derive an expression for the score function. Since $\nabla \ln p_{t_{k+1}} (\vx_{t_{k+1}}) = \frac{\nabla_{\vx_{t_{k+1}}} p_{t_{k+1}}(\vx_{t_{k+1}})}{p_{t_{k+1}(\vx_{t_{k+1}})}}$, we first manipulate $\nabla_{\vx_{t_{k+1}}} p_{t_{k+1}}(\vx_{t_{k+1}})$, 
\begin{align}
\nabla_{\vx_{t_{k+1}}}p_{t_{k+1}} (\vx_{t_{k+1}})  
&= \frac{1}{Z}\sum_{i=1}^N \nabla_{\vx_{t_{k+1}}} p(\vx_{t_k}^i)
 \exp \left(- \frac{\|\left(\sum_{n=1}^N T_{(\epsilon),i ,n} (\vx_{t_k}^i + f_{l-1, \theta}(\vx_{t_k}, t_k)) \right)  - \vx_{t_{k+1}} \|^2} { 2g(t_{k+1})^2 C_{\epsilon, i}\Delta_k}\right) \nonumber \\
&= \frac{1}{Z}\bigg(\sum_{i=1}^N  p(\vx_{t_k}^i)\left(\sum_{n=1}^N \frac{1}{g(t_{k+1})^2 \Delta_k C_{\epsilon, i}}\left(T_{(\epsilon), i, n}(\vx_{t_k}^i + f_{l-1, \theta}(\vx_{t_k}, t_k))   - \vx_{t_{k+1}} \right)\right) \nonumber \\
& \qquad \exp \left(- \frac{\|\left(\sum_{n=1}^N T_{(\epsilon),i ,n} (\vx_{t_k}^i + f_{l-1, \theta}(\vx_{t_k}, t_k)) \right)  - \vx_{t_{k+1}} \|^2} { 2g(t_{k+1})^2 C_{\epsilon, i}\Delta_k}\right)\bigg).
\end{align}
Substituting $p_{t_{k}}(x_k^i) = \frac{p_{t_{k+1}}(\vx_{t_{k+1}})  p_{\vx_{k+1} \mid \vx_{k}^i} (\vx_{k+1}) }{ p_{\vx_{k}^i \mid \vx_{k+1}} (\vx_{k}^i) }$ to the equation above gives
\begin{equation}
\nabla_{\vx_{t_{k+1}}}  p_{t_{k+1}} (\vx_{t_{k+1}})  = p_{t_{k+1}}(\vx_{t_{k+1}})\sum_{i=1}^N p_{\vx_{k+1} \mid \vx_{k}^i}  (\vx_{k}^i)
\left(\sum_{n=1}^N \frac{\left(T_{(\epsilon), i, n}(\vx_{t_k}^i + f_{l-1, \theta}(\vx_{t_k}, t_k))   - \vx_{t_{k+1}} \right)}{g(t_{k+1})^2 \Delta_k C_{\epsilon, i}}\right),
\end{equation}
and dividing by $p_{t_{k+1}(\vx_{t_{k+1}})}$ yields
\begin{equation}\label{eq:loss_step2}
\begin{aligned}
\nabla \ln p_{t_{k+1}}(\vx_{t_{k+1}}) 
= \sum_{i=1}^N p_{\vx_{t_k^i} \mid \vx_{t_{k+1}}}(\vx_{t_{k}^i}) \left(\sum_{n=1}^N \frac{\left(T_{(\epsilon), i, n}(\vx_{t_k}^i + f_{l-1, \theta}(\vx_{t_k}, t_k))   - \vx_{t_{k+1}} \right)}{g(t_{k+1})^2 \Delta_k C_{\epsilon, i}} \right).
\end{aligned}
\end{equation}
Substituting \cref{eq:loss_step2} to the definition of the optimal backward drift \cref{eq:smoothreverse} gives
\begin{align}
b_{l, \phi} (\vx_{t_{k+1}}, t_{k+1})  
& = f_{l-1, \theta}(\vx_{t_{k+1}}, t_k) -  g(t_{k+1})^2 \nabla \ln p_{t_{k+1}}(\vx_{k+1}) \nonumber\\
&=  f_{l-1, \theta}(\vx_{t_{k+1}}, t_k) \nonumber \\
&\quad - g(t_{k+1})^2 \sum_{i=1}^N p_{\vx_{t_k^i} \mid \vx_{t_{k+1}}}(\vx_{t_{k+1}}) \left(\sum_{n=1}^N \frac{\left(T_{(\epsilon), i, n}(\vx_{t_k}^i + f_{l-1, \theta}(\vx_{t_k}, t_k))   - \vx_{t_{k+1}} \right)}{g(t_{k+1})^2 \Delta_k C_{\epsilon, i}} \right),
\end{align}
where taking $f_{l-1, \theta}(\vx_{t_{k+1}}, t_k)$ inside the sum yields
\begin{multline}
b_{l, \phi} (\vx_{t_{k+1}}, t_{k+1})  =  \sum_{i=1}^N p_{\vx_{t_k^i} \mid \vx_{t_{k+1}}}(\vx_{t_{k+1}}) \\
\left(\frac{1}{C_{\epsilon, i}} \left(\sum_{n=1}^N T_{(\epsilon),i ,n} (\vx_{t_k}^i + f_{l-1, \theta}(\vx_{t_k}, t_k)) \right) - \frac{\vx_{t_{k+1}}}{C_{\epsilon, i}} - \Delta_k f_{l-1, \theta}(\vx_{t_{k+1}},t_k) \right) / \Delta_k).
\end{multline}
Multiplying the equation above by $\Delta_k$ gives
\begin{equation}
\begin{aligned}
\Delta_k b_{l, \phi} (\vx_{t_{k+1}}^i, t_{k+1})  
 =  \left(\sum_{n=1}^N T_{(\epsilon),i ,n} (\vx_{t_k}^n + f_{l-1, \theta}(\vx_{t_k}^n, t_k)) \right)
- \frac{\vx_{t_{k+1}}^i}{C_{\epsilon, i}} - \Delta_k  f_{l-1, \theta}(\vx_{t_{k+1}^i},t_k) .
\end{aligned}
\end{equation}
Thus we may set the objective for finding the optimal backward drift $b_{l, \phi}$ as
\begin{multline}\label{eq:trueloss}
 \ell_{k+1,\text{no obs}}^i = \|b_{l,  \phi}(\vx_{t_{k+1}}^i, t_{k+1})\Delta_k - \frac{\vx_{t_{k+1}}^i}{C_{\epsilon, i}} - f_{l-1, \theta}(\vx_{t_{k+1}}^i, t_k)\Delta_k  \\
  +\frac{1}{C_{\epsilon, i}} \textstyle\sum_{n=1}^N T_{(\epsilon),i, n} \left( \vx_{t_k}^n + f_{l-1, \theta}(\vx_{t_k}^n, t_k)\Delta_k \right)\|^2 .
\end{multline}
\end{proof}

Notice that if the weights before resampling are uniform, then $T_{(\epsilon)} = \MI_N$, and for all $i \in 1, 2, \dots, N$ it holds that $C_{\epsilon, i} = 1$, since all but one of the terms in the sum $\frac{1}{g(t_{k+1})^2} \mathrm{Var}\left( \sum_{n=1}^N T_{(\epsilon), i, n} \tilde{\vx}_{t_{k+1}}^n\right)$ vanish. Similarly, for one-hot weights $C_{\epsilon, i} = 1$. In practice, we set the constant $C_{\epsilon, i} = 1$ as in \cref{eq:obsloss} and observe good empirical performance with the simplified loss function.
\subsection{Connection to Hamilton--Jacobi Equations}\label{app:hamiltonjacobi}
We connect the backward drift function $b_{l, \phi}(\vx_{t_{k+1}}, t_{k+1}) = f_{l-1, \theta}(\vx_{t_k+1}, t_k)- g(t_{k+1})^2 \nabla \ln p_{t_{k+1}}(\vx_{t_{k+1}})$ to the Hamilton--Jacobi equations for stochastic control through following the setting of \citet{maoutsa2021deterministic}, which applies the drift $f_{l-1, \theta}(\vx_{t}, t)- g(t)^2 \nabla \ln p_{t}(\vx_t)$ for a backwards SDE initialized at $\pi_T$.

Consider a stochastic control problem with a path constraint $U(\vx_t, t)$, optimizing the following loss function,
\begin{equation}\label{eq:funcloss}
\mathcal{J} = \frac{1}{N}\sum_{i=1}^N \int_{t=0}^T \frac{1}{2g(t)^2}\| f_{\theta}(\vx_t^i, t) - f(\vx_t^i, t)\|^2 +U(\vx_t^i, t) \dd t  - \ln \chi(\vx_T^i),
\end{equation}
with the paths, $\vx_t^i$ sampled as trajectories from the SDE
\begin{equation}\label{eq:fwdsde}
\vx_0 \sim \pi_0, \quad 
\dd\vx_t = f_{l-1, \theta}(\vx_t, t)\dd t + g(t) \dd \vbeta_t,
\end{equation}
and the loss $\ln \chi(\vx_T^i)$ measures distance from the distribution $\pi_T$. Since we set the path constraint via observational data, our method resembles setting $U(\vx_t^i, t) = 0$ when $t$ is not an observation time, and $U(\vx_t^i) = -\log \vp (\vy \mid \vx_{t}^i)$, where $\vp (\vy \mid \vx_{t}^i)$ is the observation model. 

Let $q_t(\vx)$ denote the marginal density of the controlled (drift $f_{\theta}$) SDE at time $t$.
In \citet{maoutsa2021deterministic}, the marginal density is decomposed as
\begin{equation}\label{eq:decomp}
    q_t(\vx) = \varphi_t(\vx) p_t(\vx),
\end{equation}
where $\varphi_t(\vx)$ is a solution to a backwards Fokker-Planck-Kolmogorov (FPK) partial differential equation starting from $\varphi_T(\vx) = \pi_T$, and the density evolves as in
\begin{equation}
    \frac{\dd \varphi_t(\vx)}{\dd t} = -\mathcal{L}_{f}^{\dagger}\varphi_t(\vx) + U(\vx, t)\varphi_t(\vx),
\end{equation}
where $\mathcal{L}_f^{\dagger}$ is the adjoint FPK operator to the uncontrolled system. 
The density $p_t(\vx)$ corresponds to the forward filtering problem, initialized with $\pi_0$, 
\begin{equation}\label{eq:pathconstraintpde}
    \frac{\dd p_t(\vx)}{\dd t} = \mathcal{L}_f(p_t(\vx))-U(\vx, t)p_t(\vx),
\end{equation}
where $\mathcal{L}_f$ is the FPK operator of the uncontrolled SDE (with drift $f$). The particle filtering trajectories $\{\vx_{t_k} \}^i$ generated in our method are samples from the density defined by \cref{eq:pathconstraintpde}.  In the context of our method, the path constraint matches the log-weights of particle filtering at observation times and is zero elsewhere.

In \citet{maoutsa2021deterministic}, a backward evolution for $q_t$ is applied, using the backwards time $\tilde{q}_{T- \tau}(\vx) = q_{\tau}(\vx)$, yielding a backwards SDE starting from $\tilde{q}_0(\vx) = \{ \vx_{T}^i\}_{i=1}^N$, reweighted according to $\pi_T$.
The backward samples from $\tilde{q}$ are generated following the SDE dynamics
\begin{equation}
    \dd\vx_\tau^{i} = (f(\vx_{\tau}^{i}, T-\tau) + g(t)^2 \nabla \ln p_{T- \tau}(\vx_{\tau}^{i})\dd t + g(t) \dd \beta_{\tau} .
\end{equation}
We have thus selected the backward drift $b_{l, \phi}$ to match the drift of $\tilde{q}_t(x)$, the backward controlled density. Intuitively, our choice of $b_{l, \phi}$ is a drift which generates the smoothed particles when initialized at $\{\vx_T^i \}_{i=1}^N$, the terminal state of the forward SDE. The discrepancy between $\pi_T$ and the distribution induced by $\{\vx_T^i \}_{i=1}^N$ then motivates the use of an iterative scheme after learning to simulate from $q_t(x)$.

\subsection{Observing the Full Marginal Density}\label{app:infinity}

Suppose that at time $t_k$, we let the number of observations grow unbounded. We analyse the behaviour of our model at the resampling step, at the limit $M \to \infty$ for the number of observations and $\sigma \to 0$ for the observation noise. When applying the bootstrap proposal, recall that we combined the multiple observations to compute the log-weights as
\begin{equation}
  \log w_{t_k}^i = -\frac{1}{2\sigma^2} \sum_{\vy_j \in \data^H_{i, t_k}} \| \vx_{t_k}^i - \vy_j\|^2,
\end{equation}
which works well in practice for the sparse-data settings we have considered. Below we analyse the behaviour of an alternative way to combine the weights and show that given an infinite number of observations, it creates samples from the true underlying distribution.
\begin{prop}\label{prop:density}
Let $\{\vx_{t_k}^i \}_{i=1}^{N}$ be a set of particles and $\{\vy_{j} \}_{j=1}^{M}$ the observations at time $t_k$. Assume that the observations have been sampled from a density $\rho_{t_k}$ and that for all $i$ it holds that $\vx_{t_k}^i \in \mathrm{supp}(\rho_{t_k})$. Define the particle weights as
\begin{equation}\label{eq:logweighting}
  \log w_{t_k, \sigma, M}^i =   \log  \bigg(\frac{1}{Z |\data^{H(M)}_{i, t_k}|}\sum_{\vy_j \in \data^{H(M)}_{i, t_k}}\exp(- \| \vx_{t_k}^i - \vy_j\|^2/ 2\sigma^2) \bigg),
\end{equation}
where $Z$ is the normalization constant of the observation model Gaussian $p(\vy \mid \vx_{t_k}^i)$.
Then for each particle $\vx_{t_k}^i$, its weight satisfies
\begin{equation}
\lim_{\sigma \to 0}\lim_{M \to \infty}  w_{t_k, \sigma, M}^i = \rho_{t_k}(x_{t_k}).
\end{equation}
\end{prop}
\begin{proof}
We drop the $\sigma$ and $H(M)$ from the weight notation for simplicity of notation, but remark that the particle filtering weights are dependent on both quantities.
Consider the number of particles $N$ fixed, and denote the $d$-dimensional sphere centered at $\vx_{t_k}^i$ as $B(\vx_{t_k}^i, r)$.
Since each particle $\vx_{t_k}^i  $ lies in the support of the true underlying marginal density $\rho_{t_k}$, then for any radius $r > 0$ such that $B(\vx_{t_k}^i, r) \in \mathrm{supp}(\rho_{t_k})$, and $H > 0$, we may choose $M$ high enough so that the points $\vy_j \in \data^H_{i, t_k}$ satisfy $\vy_j \in B(\vx_{t_k}^i, r)$.
It follows from \cref{eq:logweighting} that
\begin{equation}
w_{t_k}^i = \frac{1}{Z | \data^{H(M)}_{i, t_k}|} \sum_{\vy_j \in \data^{H(M)}_{i, t_k}}\exp(- \| \vx_{t_k}^i - \vy_j\|^2/ 2\sigma^2).
\end{equation}
For any $r>0$ and with observation noise  $\sigma = cr$, we may set $c, H(M)$ so that the sum above approximates the integral
\begin{equation}
w_{r, t_k}^i \approx \frac{1}{|B(\vx_{t_k}^i, r)|} \int_{B(\vx_{t_k}^i, r)}p(\vy  \mid \vx_{t_k}^i)  \rho_t(\vy) \dd \vy.
\end{equation}
By applying the Lebesque differentiation theorem, we obtain that for almost every $\vx_{t_k}^i$, we have  $\lim_{r \to 0}w_{t_k, r}^i = \rho_{t_k}(\vx_{t_k}^i)$, since as $\sigma \to 0$, the density $p(\vy \mid \vx_{t_k}^i)$ collapses to the Dirac delta of $\vx_{t_k}^i$.
\end{proof}

\cref{prop:density} can be interpreted as the infinite limit of a kernel density estimate of the true underlying distribution. Resampling accurately reweights the particles so that the probability of resampling particle $\vx_{t_k}^i$ is proportional to the density $\rho_{t_k}$ compared to the other particles.
Notice that the result does not guarantee that the particles will cover the support of $\rho_{t_k}$, since we did not assume that the drift initialization generates a marginal density at time $t_k$  covering its support.

\section{Experimental Details}
\label{app:experiment}

\subsection{2D Toy Data Sets}
\label{app:2dtoy} 
For the constrained transport problem for two-dimensional scikit-learn, the observational data we chose to use was different for each of the three data sets presented; {\sc two moons}, {\sc two circles} and the S-shape. All three experiments had the same discretization ($t\in [0, 0.99])$, $\Delta_k = 0.01$), learning rate $0.001$, and differentiable resampling regularization parameter $\epsilon= 0.01$. The process noise $g(t)^2$ follows a linear schedule from $0.001$ to $1$, with low noise at time $t=0$ and high noise at $t=0.99$, and each iteration of the ISB method trains the forward and backward drift networks each for $5000$ iterations, with batch size $256$. When running on a Macbook Pro CPU, it took approximately $6$ minutes to complete for the {\sc two circles} experiment for instance, while the exact runtime varies based on factors such as the number of observations and the number of ISB iterations required. Other hyperparameters are explained below.

\paragraph{Two moons}
The observational data consists of $10$ points selected from the Schrödinger bridge trajectories, all observed at $t \in [0.25, 0.5, 0.75]$ with an exponential observation noise schedule $\kappa(l) = 1.25^{l-1}$. The ISB was run for $6$ epochs and initialized with a drift from the pre-trained Schrödinger bridge model from the unconstrained problem. 

\paragraph{Two circles}
The observational data consists of $10$ points which lie evenly distributed on a circle,  observed at $t=0.5$ with an exponential observational noise schedule $\kappa(l) = 0.5 \cdot 1.25^{l-1}$. The ISB was run for $6$ epochs and initialized with a drift from the pre-trained Schrödinger bridge model from the unconstrained problem.

\paragraph{S-shape}
The observational data consists of $6$ points, with pairs being observed at times $t \in [0.4, 0.5, 0.6]$. We used a bilinear observational noise schedule with a linear decay for the first half of the iterations from $\kappa(0)^2  = 4$ to $\kappa(L/2)^2 = 1$ and a linear ascend for the second half of the iterations from $\kappa(L/2)^2 = 1$ to $\kappa(L)^2 = 4$. The ISB ran for $6$ epochs, with a zero drift initialization.

\begin{figure*}[t!]
  \centering
  \setlength{\figurewidth}{0.27\textwidth}
  \setlength{\figureheight}{\figurewidth}

  \begin{subfigure}[b]{.18\textwidth}
    \centering
\begin{tikzpicture}

\begin{axis}[
height=\figureheight,
hide x axis,
hide y axis,
tick pos=left,
width=\figurewidth,
xmin=-12.0590478317368, xmax=11.9153978679085,
ymin=-12.1915820827153, ymax=11.78286361693
]
\addplot graphics [includegraphics cmd=\pgfimage,xmin=-24.4828913813874, xmax=23.1484576907555, ymin=-16.1608611720605, ymax=15.5933715427014] {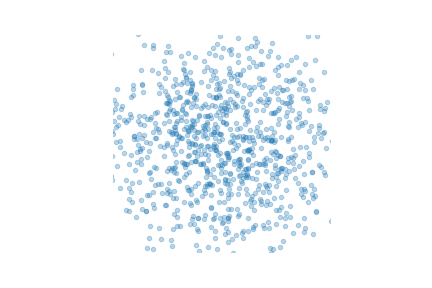};
\end{axis}

\end{tikzpicture}
  \end{subfigure}
  \hfill
  \begin{subfigure}[b]{.18\textwidth}
    \centering
\begin{tikzpicture}

\begin{axis}[
height=\figureheight,
hide x axis,
hide y axis,
tick pos=left,
width=\figurewidth,
xmin=-12.0590478317368, xmax=11.9153978679085,
ymin=-12.1915820827153, ymax=11.78286361693
]
\addplot graphics [includegraphics cmd=\pgfimage,xmin=-24.4828913813874, xmax=23.1484576907555, ymin=-16.1608611720605, ymax=15.5933715427014] {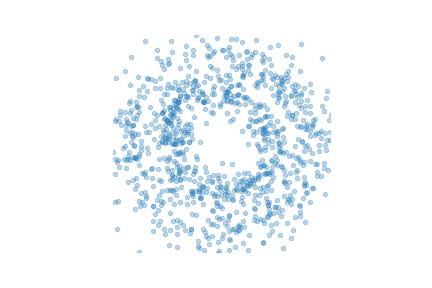};
\end{axis}

\end{tikzpicture}
  \end{subfigure}
  \begin{subfigure}[b]{.18\textwidth}
    \centering
\begin{tikzpicture}

\begin{axis}[
height=\figureheight,
hide x axis,
hide y axis,
tick pos=left,
width=\figurewidth,
xmin=-12.0590478317368, xmax=11.9153978679085,
ymin=-12.1915820827153, ymax=11.78286361693
]
\addplot graphics [includegraphics cmd=\pgfimage,xmin=-24.4828913813874, xmax=23.1484576907555, ymin=-16.1608611720605, ymax=15.5933715427014] {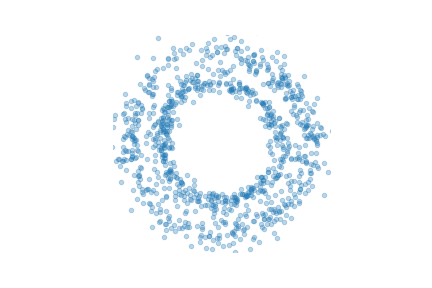};
\addplot [draw=red, fill=red!50, mark=*, only marks, mark options={solid}, mark size=2pt, opacity=0.8]
table{%
x                      y
1.5 3
3.26335573196411 2.42705106735229
4.35316944122314 0.92705100774765
4.35316944122314 -0.92705100774765
3.26335573196411 -2.42705106735229
1.5 -3
-0.263355761766434 -2.42705106735229
-1.35316956043243 -0.92705100774765
-1.35316956043243 0.92705100774765
-0.263355761766434 2.42705106735229
};
\end{axis}

\end{tikzpicture}
  \end{subfigure}  
  \hfill
  \begin{subfigure}[b]{.18\textwidth}
    \centering
\begin{tikzpicture}

\begin{axis}[
height=\figureheight,
hide x axis,
hide y axis,
tick pos=left,
width=\figurewidth,
xmin=-12.0590478317368, xmax=11.9153978679085,
ymin=-12.1915820827153, ymax=11.78286361693
]
\addplot graphics [includegraphics cmd=\pgfimage,xmin=-24.4828913813874, xmax=23.1484576907555, ymin=-16.1608611720605, ymax=15.5933715427014] {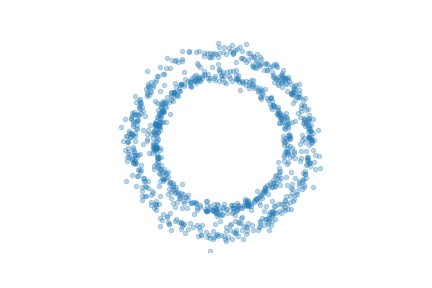};
\end{axis}

\end{tikzpicture}
  \end{subfigure}
  \hfill  
  \begin{subfigure}[b]{.18\textwidth}
    \centering
\begin{tikzpicture}

\begin{axis}[
height=\figureheight,
hide x axis,
hide y axis,
tick pos=left,
width=\figurewidth,
xmin=-12.0590478317368, xmax=11.9153978679085,
ymin=-12.1915820827153, ymax=11.78286361693
]
\addplot graphics [includegraphics cmd=\pgfimage,xmin=-24.4828913813874, xmax=23.1484576907555, ymin=-16.1608611720605, ymax=15.5933715427014] {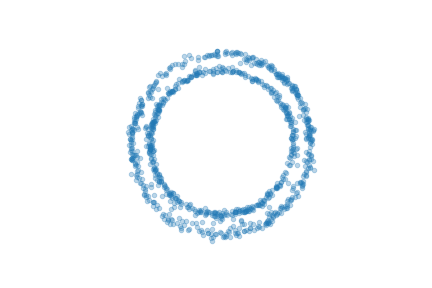};
\end{axis}

\end{tikzpicture}
  \end{subfigure}\\[-1em]
  \begin{subfigure}[b]{.18\textwidth}
    \centering
\begin{tikzpicture}

\begin{axis}[
height=\figureheight,
hide x axis,
hide y axis,
tick pos=left,
width=\figurewidth,
xmin=-8.94578966682228, xmax=15.8074380629519,
ymin=-10.577301428008, ymax=14.1759263017662
]
\addplot graphics [includegraphics cmd=\pgfimage,xmin=-21.7732073479801, xmax=27.4053907906441, ymin=-14.67551793956, ymax=18.1102141528562] {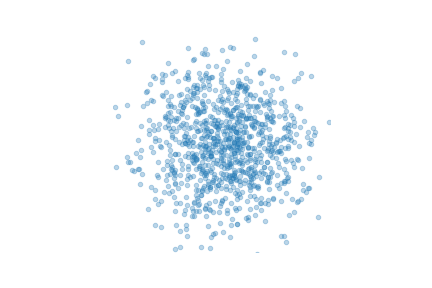};
\end{axis}

\end{tikzpicture}\\[-1em]
    \caption{$t=0.00$}
  \end{subfigure}
  \hfill
  \begin{subfigure}[b]{.18\textwidth}
    \centering
\begin{tikzpicture}

\begin{axis}[
height=\figureheight,
hide x axis,
hide y axis,
tick pos=left,
width=\figurewidth,
xmin=-8.94578966682228, xmax=15.8074380629519,
ymin=-10.577301428008, ymax=14.1759263017662
]
\addplot graphics [includegraphics cmd=\pgfimage,xmin=-21.7732073479801, xmax=27.4053907906441, ymin=-14.67551793956, ymax=18.1102141528562] {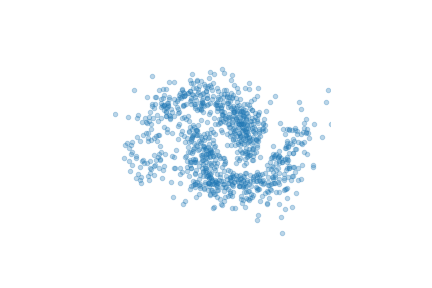};
\addplot [draw=red, fill=red!50, mark=*, only marks, mark options={solid}, mark size=2pt, opacity=0.8]
table{%
x                      y
-5.74272012710571 -1.2900493144989
10.282054901123 6.23231601715088
11.1886224746704 5.12994289398193
-7.97766780853271 1.20648860931396
8.09059429168701 -2.23533320426941
4.45591497421265 4.15110778808594
3.3600537776947 3.95255351066589
3.7745156288147 6.05395746231079
-0.357202261686325 -0.101248815655708
-5.85373735427856 -3.47627544403076
};
\end{axis}

\end{tikzpicture}\\[-1em]
    \caption{$t=0.25$}
  \end{subfigure}
  \hfill  
  \begin{subfigure}[b]{.18\textwidth}
    \centering
\begin{tikzpicture}

\begin{axis}[
height=\figureheight,
hide x axis,
hide y axis,
tick pos=left,
width=\figurewidth,
xmin=-8.94578966682228, xmax=15.8074380629519,
ymin=-10.577301428008, ymax=14.1759263017662
]
\addplot graphics [includegraphics cmd=\pgfimage,xmin=-21.7732073479801, xmax=27.4053907906441, ymin=-14.67551793956, ymax=18.1102141528562] {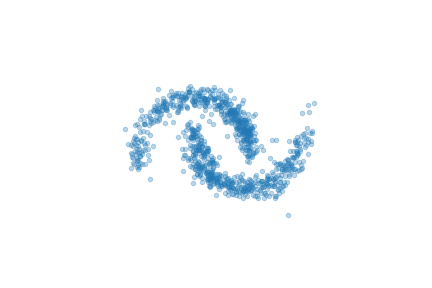};
\addplot [draw=red, fill=red!50, mark=*, only marks, mark options={solid}, mark size=2pt, opacity=0.8]
table{%
x                      y
-5.85671424865723 -0.975227236747742
11.3966102600098 4.97208642959595
11.7787132263184 3.69037508964539
-7.75934648513794 1.58934283256531
8.92796421051025 -2.40417170524597
5.42121362686157 4.35129737854004
3.58317613601685 4.98470258712769
3.04773569107056 5.66898202896118
0.561857163906097 -0.249238520860672
-6.35246849060059 -1.97278606891632
};
\end{axis}

\end{tikzpicture}\\[-1em]
    \caption{$t=0.50$}
  \end{subfigure}  
  \hfill
  \begin{subfigure}[b]{.18\textwidth}
    \centering
\begin{tikzpicture}

\begin{axis}[
height=\figureheight,
hide x axis,
hide y axis,
tick pos=left,
width=\figurewidth,
xmin=-8.94578966682228, xmax=15.8074380629519,
ymin=-10.577301428008, ymax=14.1759263017662
]
\addplot graphics [includegraphics cmd=\pgfimage,xmin=-21.7732073479801, xmax=27.4053907906441, ymin=-14.67551793956, ymax=18.1102141528562] {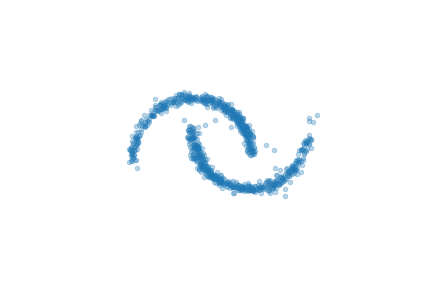};
\addplot [draw=red, fill=red!50, mark=*, only marks, mark options={solid}, mark size=2pt, opacity=0.8]
table{%
x                      y
12.9552555084229 4.29890060424805
13.3716583251953 3.17004346847534
-7.17733383178711 2.52120113372803
9.50479507446289 -2.82684803009033
5.19372701644897 4.41709136962891
4.38437175750732 5.42100954055786
3.66518402099609 5.84365081787109
0.788959980010986 0.159530133008957
-6.95842123031616 -0.97951078414917
};
\end{axis}

\end{tikzpicture}\\[-1em]
    \caption{$t=0.75$}
  \end{subfigure}
  \begin{subfigure}[b]{.18\textwidth}
    \centering
\begin{tikzpicture}

\begin{axis}[
height=\figureheight,
hide x axis,
hide y axis,
tick pos=left,
width=\figurewidth,
xmin=-8.94578966682228, xmax=15.8074380629519,
ymin=-10.577301428008, ymax=14.1759263017662
]
\addplot graphics [includegraphics cmd=\pgfimage,xmin=-21.7732073479801, xmax=27.4053907906441, ymin=-14.67551793956, ymax=18.1102141528562] {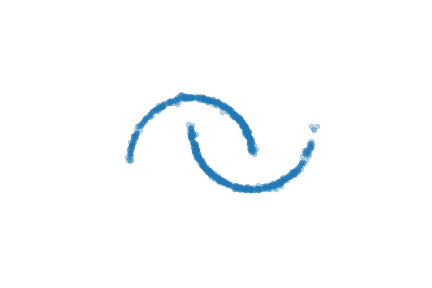};
\end{axis}

\end{tikzpicture}\\[-1em]
    \caption{$t=0.99$}
  \end{subfigure}
  \caption{The unconditioned Schrödinger bridges for the scikit-learn 2D experiment, corresponding to the  \cref{fig:2dtoy}. The observations (red markers) were not used in any way during training but are included in the figure for reference, to show that the unconstrained dynamics greatly differ from the learned ISB model dynamics. \looseness-1}
  \label{fig:2dtoyipfp}
\end{figure*}
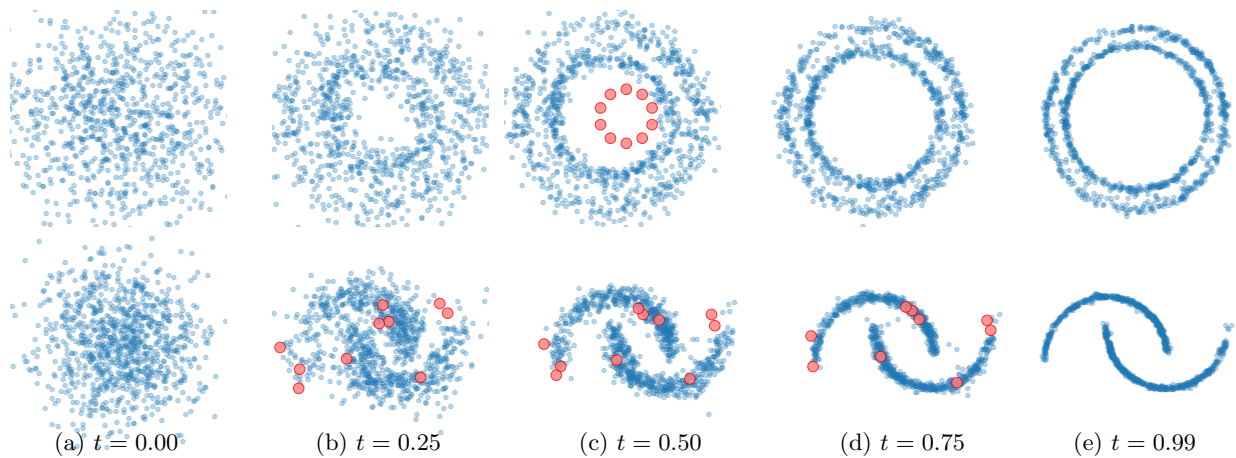

\begin{figure*}[b!]
  \centering
  \scriptsize
  \pgfplotsset{scale only axis,y tick label style={rotate=90},xtick={0,4,8,12}}
  \def\datapath{./fig/benes}
  \setlength{\figurewidth}{0.5\textwidth}
  \setlength{\figureheight}{0.65\figurewidth}  
  \begin{subfigure}[b]{.5\textwidth}
    \centering
    \input{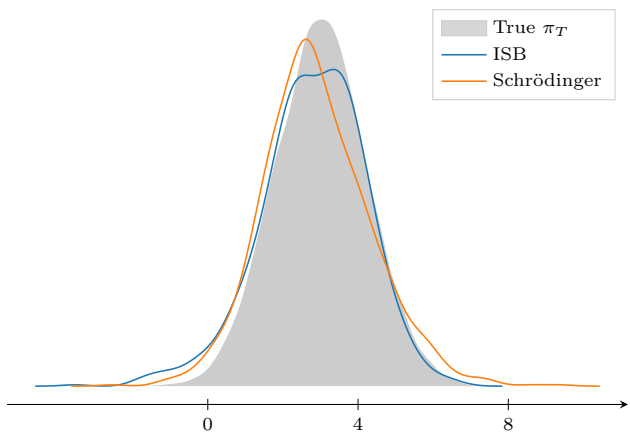}
  \end{subfigure}
  \caption{A kernel density estimate of the Bene\v{s} SDE terminal state.  We compare $\pi_T$ to the Schrödinger bridge and ISB terminal states. Both unconstrained Schrödinger bridge and ISB terminal states succeed in representing $\pi_T$ well, with the Schrödinger bridge terminal state more closely matching $\pi_T$ near its mean.}
\label{fig:benesterminal}
\end{figure*}
\subsection{The Bene\v{s} SDE}
\label{app:benes}

In the \Benes SDE experiment, we obtain the sparse observational data from sampled \Benes SDE trajectories while the terminal state is a shifted and scaled ($3+5x_T$) version of a \Benes marginal density. As the \Benes trajectories were first generated by simulating the SDE until $t=6$ and then in reverse from $t=6$ to $t=0$, we set $T=11.97$. We apply the analytical expression for the \Benes marginal density for computing $\log p_{t}(\vx)$,
\begin{equation}
   p_t(\vx) = \frac{1}{\sqrt{2\pi t}} \frac{\cosh(\vx)}{\cosh(\vx_0)} \exp\bigg(-\frac{1}{2} t\bigg) \, \exp\bigg(-\frac{1}{2 t}(\vx-\vx_0)^2\bigg).
\end{equation}
See the \Benes SDE trajectories in \cref{fig:benes-a}. As expected, the transport model with no observations performs well in the generative task, but its trajectories cover also some low-likelihood space around $t=6$ (in the middle part in \cref{fig:benes-b}). The observations for the ISB model were sampled from the generated trajectories, $10$ observations at $10$ random time-instances (see \cref{fig:benes-c})

Both the unconstrained Schrödinger bridge model and the ISB model were run for $3$ iterations, using a learning rate of $0.001$ for the neural networks. Likely due to the fact that the problem was only one-dimensional, the convergence of the Schrödinger bridge to a process which matches the desired terminal state was fast, and we chose not to run the model for a higher number of ISB iterations, see \cref{fig:benesterminal} for a comparison of the trained model marginal densities and the true terminal distribution $\pi_T$. We set the observation noise schedule to the constant $0.7$,  and at each iteration of the ISB or the unconstrained Schrödinger bridge the drift neural networks were trained for $5000$ iterations each with the batch size $256$, and the trajectories were refreshed every $500$ iterations with a cache size of $1000$ particles. The number of nearest neighbours to compare to was $H=10$.

\begin{figure*}[t]
\label{fig:benes}
  \centering
  \scriptsize
  \pgfplotsset{scale only axis,y tick label style={rotate=90},xtick={0,4,8,12}}
  \def\datapath{./fig/benes}
  \setlength{\figurewidth}{0.25\textwidth}
  \setlength{\figureheight}{\figurewidth}  
  \begin{subfigure}[b]{.32\textwidth}
    \centering
\begin{tikzpicture}

\begin{axis}[
height=\figureheight,
tick align=outside,
tick pos=left,
width=\figurewidth,
x grid style={white!69.0196078431373!black},
xlabel={Time, \(\displaystyle t\)},
xmin=0, xmax=12,
xtick style={color=black},
y grid style={white!69.0196078431373!black},
ymin=-15, ymax=15,
ytick style={color=black},
axis x line=bottom,
axis y line=left,
]
\addplot graphics [includegraphics cmd=\pgfimage,xmin=-1.93548387096774, xmax=13.5483870967742, ymin=-19.9668874172185, ymax=19.7682119205298] {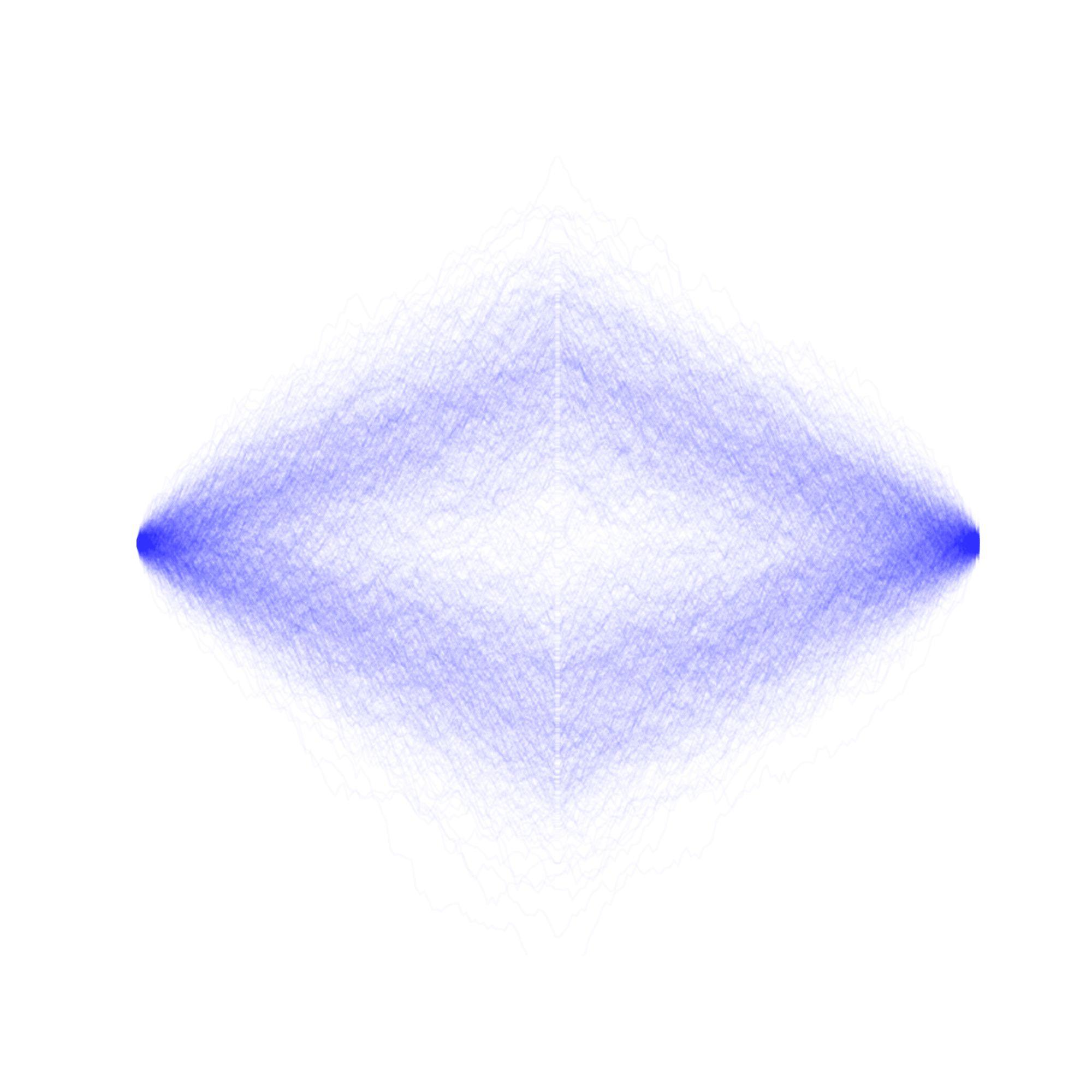};
\end{axis}

\end{tikzpicture}
    \caption{Trajectories of the \Benes SDE}
    \label{fig:benes-a}
  \end{subfigure}
  \hfill
   \begin{subfigure}[b]{.32\textwidth}
    \centering
\begin{tikzpicture}

\begin{axis}[
height=\figureheight,
tick align=outside,
tick pos=left,
width=\figurewidth,
x grid style={white!69.0196078431373!black},
xlabel={Time, \(\displaystyle t\)},
xmin=0, xmax=12,
xtick style={color=black},
y grid style={white!69.0196078431373!black},
ymin=-15, ymax=15,
ytick style={color=black},
axis x line=bottom,
axis y line=left,
]
\addplot graphics [includegraphics cmd=\pgfimage,xmin=-1.93548387096774, xmax=13.5483870967742, ymin=-19.9668874172185, ymax=19.7682119205298] {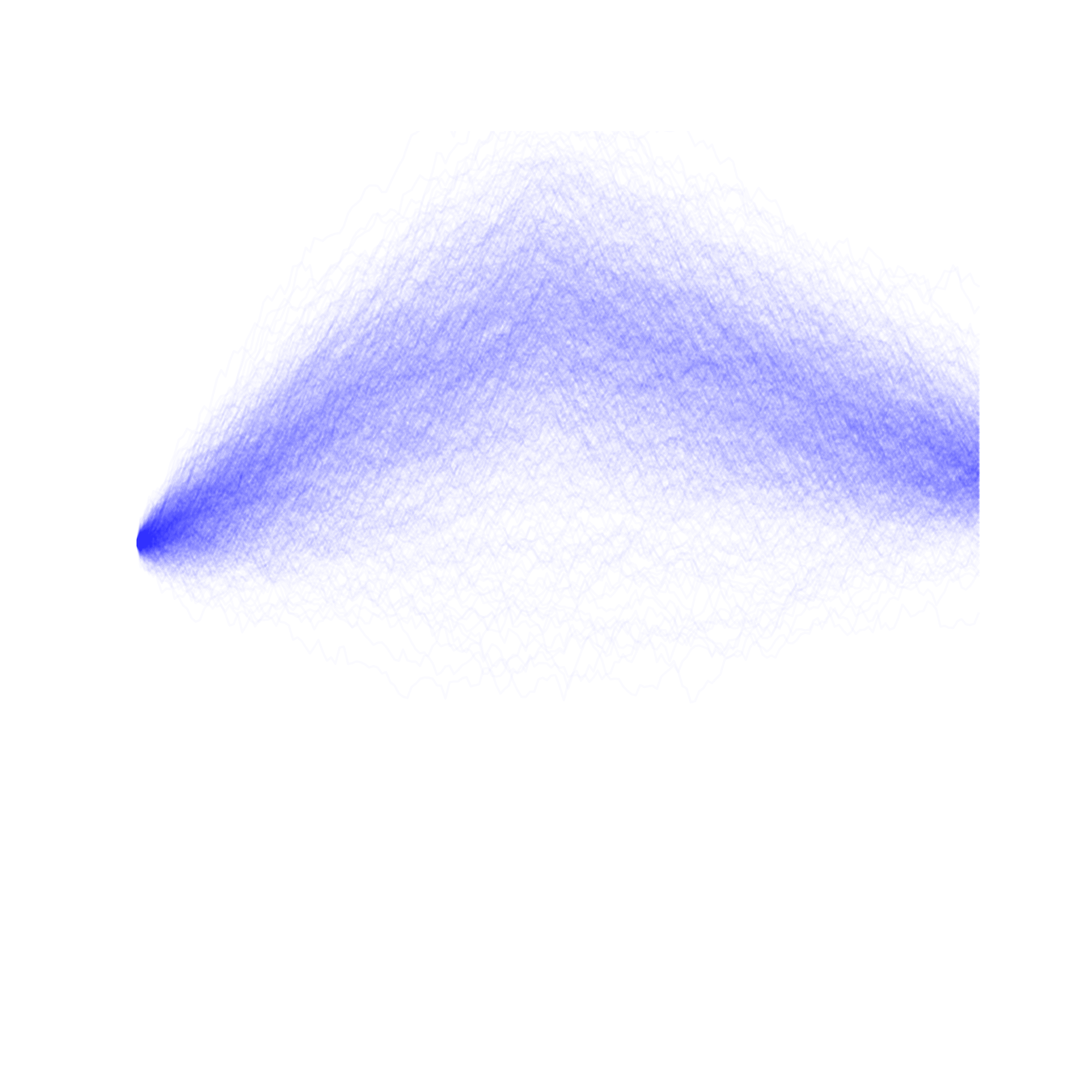};
\end{axis}

\end{tikzpicture}
    \caption{SBP trajectories}
    \label{fig:benes-b}    
  \end{subfigure}
  \hfill
  \begin{subfigure}[b]{.32\textwidth}
    \centering
\begin{tikzpicture}

\begin{axis}[
height=\figureheight,
tick align=outside,
tick pos=left,
width=\figurewidth,
x grid style={white!69.0196078431373!black},
xlabel={Time, \(\displaystyle t\)},
xmin=0, xmax=12,
xtick style={color=black},
y grid style={white!69.0196078431373!black},
ymin=-15, ymax=15,
ytick style={color=black},
axis x line=bottom,
axis y line=left,
]
\addplot graphics [includegraphics cmd=\pgfimage,xmin=-1.93548387096774, xmax=13.5483870967742, ymin=-19.9668874172185, ymax=19.7682119205298] {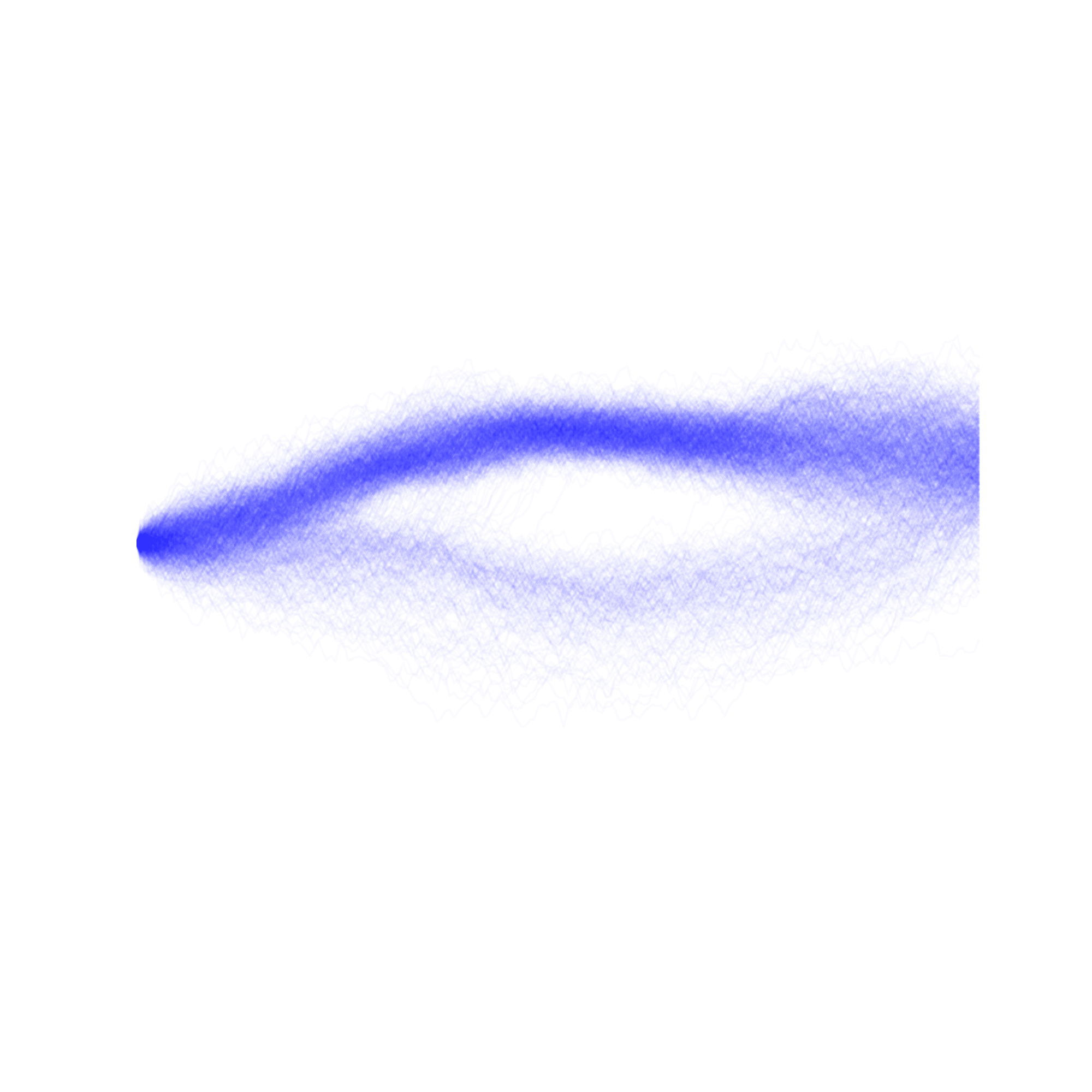};
\addplot [draw=red, fill=red!50, mark=*, only marks, mark options={solid}, mark size=1.75pt, opacity=0.8]
table{%
x                      y
8.39999961853027 2.77584075927734
7.91999959945679 2.3698103427887
6.59999942779541 3.33613801002502
5.76000022888184 -6.42137813568115
4.01999998092651 -2.9710066318512
3.83999991416931 -3.16476106643677
3.53999996185303 -2.18802905082703
2.16000008583069 -1.80536270141602
0.540000021457672 -1.00247526168823
0.480000019073486 -1.16861605644226
8.39999961853027 -0.802420735359192
7.91999959945679 -1.11832368373871
6.59999942779541 -3.12408900260925
5.76000022888184 -0.653072893619537
4.01999998092651 0.273781359195709
3.83999991416931 0.161338806152344
3.53999996185303 0.412588655948639
2.16000008583069 1.45810091495514
0.540000021457672 -0.0687554702162743
0.480000019073486 0.185456022620201
8.39999961853027 -4.91025257110596
7.91999959945679 -5.75618028640747
6.59999942779541 -7.48824882507324
5.76000022888184 7.28471231460571
4.01999998092651 6.64821624755859
3.83999991416931 6.46508312225342
3.53999996185303 4.95643615722656
2.16000008583069 1.43728137016296
0.540000021457672 1.10985696315765
0.480000019073486 1.01150941848755
8.39999961853027 -3.33730220794678
7.91999959945679 -4.67712688446045
6.59999942779541 -6.75832843780518
5.76000022888184 3.14960074424744
4.01999998092651 4.22397756576538
3.83999991416931 3.88610696792603
3.53999996185303 4.18695068359375
2.16000008583069 0.178776055574417
0.540000021457672 0.520784735679626
0.480000019073486 0.276561379432678
8.39999961853027 4.13615131378174
7.91999959945679 4.53694677352905
6.59999942779541 5.33790159225464
5.76000022888184 -4.48580932617188
4.01999998092651 -4.4669771194458
3.83999991416931 -3.93499422073364
3.53999996185303 -2.65869331359863
2.16000008583069 -1.08206582069397
0.540000021457672 -0.901625871658325
0.480000019073486 -0.794871509075165
8.39999961853027 3.20768308639526
7.91999959945679 3.07015800476074
6.59999942779541 4.07322692871094
5.76000022888184 6.95915699005127
4.01999998092651 6.34803867340088
3.83999991416931 5.80090093612671
3.53999996185303 5.41589164733887
2.16000008583069 0.779441773891449
0.540000021457672 0.404618471860886
0.480000019073486 0.484403014183044
8.39999961853027 0.846628487110138
7.91999959945679 0.489238113164902
6.59999942779541 0.126059964299202
5.76000022888184 -2.62779068946838
4.01999998092651 -0.498392909765244
3.83999991416931 -0.519110560417175
3.53999996185303 0.00727069936692715
2.16000008583069 -0.886314988136292
0.540000021457672 -0.104776151478291
0.480000019073486 0.0212062019854784
8.39999961853027 4.75660800933838
7.91999959945679 4.14523506164551
6.59999942779541 3.83526015281677
5.76000022888184 4.30058336257935
4.01999998092651 2.78824234008789
3.83999991416931 3.01853013038635
3.53999996185303 2.80643248558044
2.16000008583069 1.90970432758331
0.540000021457672 0.0677784904837608
0.480000019073486 -0.112006448209286
8.39999961853027 -6.33619832992554
7.91999959945679 -7.04404640197754
6.59999942779541 -9.8674898147583
5.76000022888184 4.69409942626953
4.01999998092651 2.80999493598938
3.83999991416931 2.62865543365479
3.53999996185303 2.17022442817688
2.16000008583069 1.34645068645477
0.540000021457672 0.456274181604385
0.480000019073486 -0.0132818724960089
8.39999961853027 -2.58253145217896
7.91999959945679 -3.01222944259644
6.59999942779541 -5.07422161102295
5.76000022888184 -7.57088613510132
4.01999998092651 -4.51119375228882
3.83999991416931 -4.87745380401611
3.53999996185303 -4.72102212905884
2.16000008583069 -2.4377613067627
0.540000021457672 -0.599174797534943
0.480000019073486 -0.60036039352417
};
\end{axis}

\end{tikzpicture}
    \caption{ISB trajectories (ours)}
    \label{fig:benes-c}    
  \end{subfigure}
   \caption{Comparison of the solution for the SBP (with \Benes SDE reference drift) and the ISB (with zero initial drift) on the \Benes SDE under sparse observations (\protect\tikz[baseline=-.5ex]\protect\node[circle,draw=red,fill=red!50,inner sep=1.5pt,]{};). The target distribution $\pi_T$ is slightly shifted and scaled from the \Benes SDE. Even if the SBP has the true model as reference drift, its trajectories degenerate into a unimodal distribution, while the ISB manages to cover both modes even if only sparse observations are available.}
  \label{fig:benes}
\end{figure*}

\subsection{The Bird Migration Data Set}
\label{app:birds}
The ISB model learned bird migration trajectories which transport the particles from the Northern Europe summer habitats to the southern winter habitats, see \cref{fig:birds-app} for a comparison of a Schrödinger bridge and ISB. Since the problem lies on a sphere,  Schrödinger bridge methods adjusted for learning on Riemannian manifolds could have been applied here. For simplicity, we mapped the problem to a two-dimensional plane using a Mercator projection and solved the problem on a $[0,5]\times [0, 5]$ square. The SDE had the discretization $t \in [0, 0.99]$, $\Delta_k = 0.01$ and a constant process noise $g(t)^2 = 0.05$. The model was trained for $12$ iterations, and initialized with a zero drift, while the observational data was chosen by the authors to promote learning trajectories clearly different from the unconstrained transport trajectories. The observation noise schedule was piecewise linear (starting at $2$, going to $0.1$ at iteration $6$, then rising linearly to reach $2$ at iteration $12$). At each ISB iteration, the neural networks were trained for $5000$ iterations each, and the trajectories were refreshed every $1000$ iterations. We used a batch size of $256$ and a learning rate of $0.001$.

\subsection{The MNIST Generation Task}
\label{app:mnist}
Applying state-space model approaches such as particle filtering and smoothing to generative diffusion models directly in the observation space (that is, not in a lower-dimensional latent space) has to our knowledge not been explored before. Some experimental design choices had a great impact on the training objectives sensibility, as the observational data is completely artificial and its timing during the process modifies the filtering distribution significantly. As the MNIST conditional generative model was trained to display the scalability of our method beyond low-dimensional toy examples, we did not further explore optimizing the hyperparameters or the observation model.  To avoid the background noise in MNIST images in the middle of the generative process impacting the particle filtering weights excessively, the observation model is a Gaussian with masked inputs equal to zero in pixels where the observation image is black, see \cref{fig:mnist} for sampled trajectories. The figure shows the progression of seven samples, where the lower half of the eight resemble the observation target. 
\par 
The SDE was run for time $t \in [0, 0.5]$, with the digit eight observed at $t=0.38$. The ISB method was applied for $10$ iterations, with a discretization $t \in [0, 0.495]$, $\Delta_k = 0.005$, and the process noise $g(t)^2$ followed a linear schedule from $0.0001$ to $1$. At each iteration of the method, the forward and backward drift neural networks were trained for $5000$ iterations with a batch size of $256$, and the trajectory cache regenerated every $1000$ iterations. The observational data consisted of a single sample of a lower half of the digit eight, observed at time $t=0.38$. The observation noise schedule was a constant $\kappa (l) = 0.3$.

\begin{figure*}[!t]
  \centering \includegraphics[width=\textwidth]{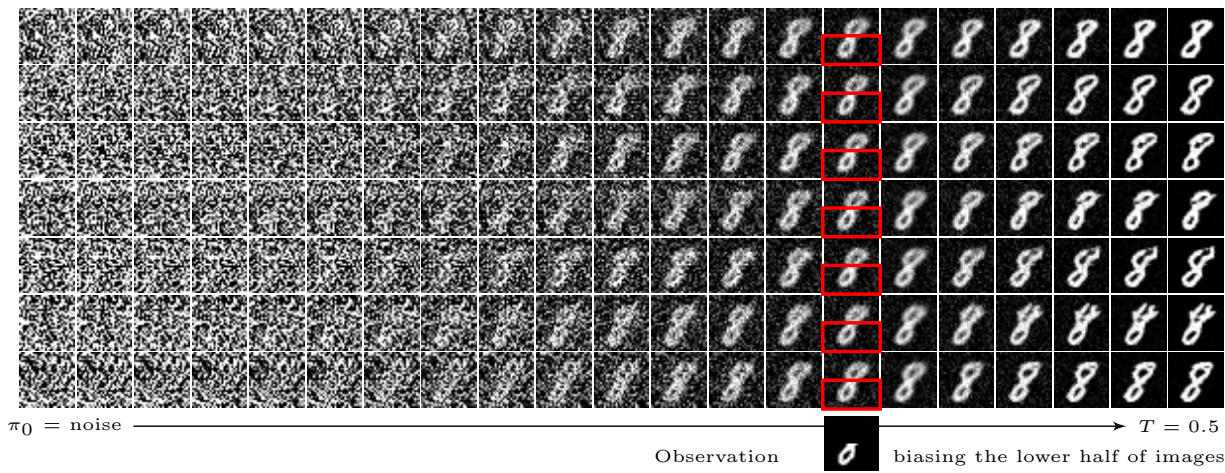}
  \caption{Model trajectories for MNIST digit `$8$' conditioned on a lower-loop of a single `8' at $t=0.38$ to bias the lower half of the digits to look alike, with the effect still visible at terminal time $T$.\looseness-1}
  \label{fig:mnist}
  \vspace*{-1em}
\end{figure*}

\subsection{MNIST for Multi-modal Data Generation}\label{app:mnist-multi}
In addition to the MNIST experiment as explained in \cref{app:mnist}, we evaluated the performance of our model on a multi-modal generation task. The reference model was a Schrödinger bridge trained via IPFP from normal distribution noise to MNIST digits eight and nine, and the observation was a single upper loop of a figure eight. Our goal was to generate trajectories that match the observation but still generate both digits eight and nine by the end of the trajectory. Based on the results in \cref{fig:mnist-multi}, ISB succeeds in the task. Most hyperparameters were kept the same as explained in \cref{app:mnist}, but the cache size if IPFP training was increased to $5000$ and the number of trajectories in the particle filtering step of ISB was increased to $1000$, to encourage a sufficiently versatile sample of both modes.

\begin{figure*}[!t]
  \centering \includegraphics[width=\textwidth]{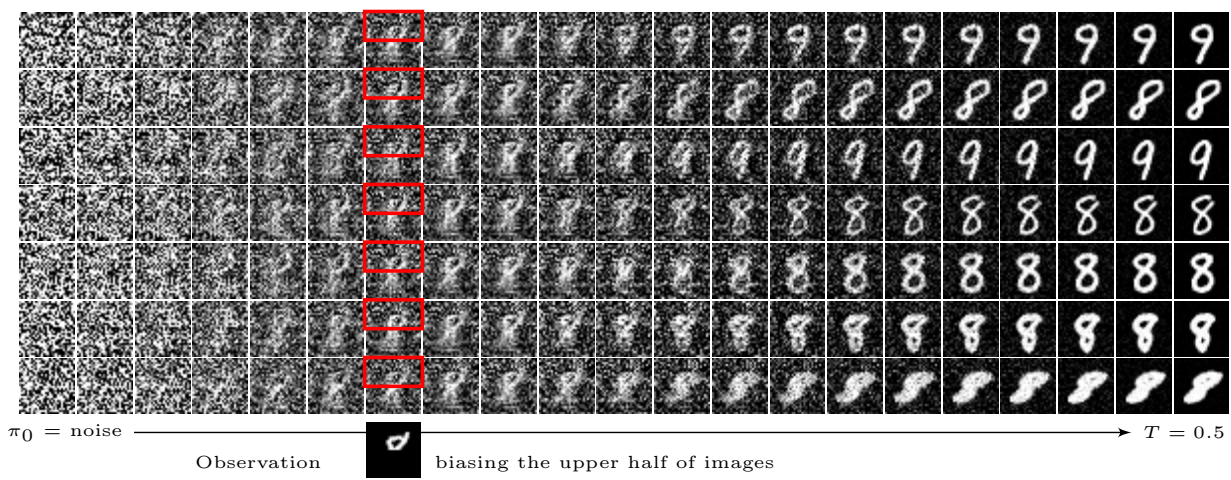}
  \caption{Model trajectories for MNIST digits `$8$' and '$9$' conditioned on a upper-loop of a single `8' at $t=0.38$ to bias the upper half of the digits to look alike, with the effect still visible at terminal time $T$, while observing both modes at the end of the trajectory.}
  \label{fig:mnist-multi}
  \vspace*{-1em}
\end{figure*}
\subsection{Single-Cell Data Set}
\label{app:singlecell}
We directly use the preprocessed data from the TrajectoryNet \citep{tong2020trajectory} repository.
A major difference between our implementation and \citet{vargas2021solving} is the reference drift. We set the reference drift to zero, which means that we utilize the intermediate data only as observations in the state-space model. On the contrary, \citet{vargas2021solving} fits a mixture model of $15$ Gaussians on the combined data set (across all measurement times) and sets the reference drift to the gradient of the log-likelihood of the mixture model. Effectively, such a reference drift aids in keeping the SDE trajectories within the support of the combined data set. We remark that if the intermediate observed marginals had clearly disjoint support, combining all the data would cause the mixture model to have `gaps' and could cause an unstable reference model drift. Thus we consider our approach of setting the reference drift to zero as more generally applicable. 

As in \citet{vargas2021solving}, we set the process noise to $g(t) = 1$ and model the SDE between time $t \in [0, 4]$. The learning rate is set to $0.001$ with a batch size of $256$ and the number of neural network training iterations equal to $5000$. We apply the ISB for $6$ iterations. We perform filtering using $1000$ points from the intermediate data sets, but compute the Earth mover's distance by comparing it to all available data.  As the observational data at $T=1, 2, 3$ consists of a high number of data points, the parameters $H$ (number of nearest neighbours) and $\sigma$ (observation noise) need to be carefully set.
We set $H=10$ to only include the close neighbourhood of each particle and set the observation noise schedule as constant $0.7$.

\section{Computational Considerations }\label{app:ablation}
In \cref{sec:compmethods}, we raised a number of important computational considerations for the constrained transport problem. Below we discuss them in detail, analyzing the limit $L \to \infty$  from the perspective of setting the observation noise schedule in \cref{app:discussobs}, and presenting ablation results on modifying the initial drift in the bird migration experiment in \cref{app:initdrift}. Finally, we study the impact of the observation noise schedule on effective sample size in the filtering step in \cref{app:ess}.

\subsection{Discussion on Observation Noise}\label{app:discussobs}
We briefly mentioned in \cref{sec:compmethods} that when letting $L \to \infty$, the choice of observation noise should be carefully planned in order for the ISB procedure to have a stationary point. Here we explain why an unbounded observation noise schedule $\kappa(l)$ implies convergence to the IPF method for uncontrolled Schrödinger bridges \citep{bortoli2021diffusion}, when using a nearest neighbour bootstrap filter as the proposal density.
\par 
\begin{prop}\label{prop:obsnoise}
Let $\Omega \in \mathbb{R}^d$ be a bounded domain where both the observations and SDE trajectories lie, and let the particle filtering weights $\{w_{l, t_k}^i\}_{i=1}^N$ at ISB iteration $l$ be 
\begin{equation}\label{eq:obsnoiseproof}
  \log w_{l, t_k}^i = -\frac{1}{2\kappa (l)^2} \sum_{\vy_j \in \data^H_{t_k}} \| \vx_{t_k}^i - \vy_j\|^2.
\end{equation}
If the schedule $\kappa(l)$ is unbounded with respect to $l$, then for any $\delta$ there exists $l'$ such that for the normalized weights it holds
\begin{equation}\label{eq:obsnoiseprop}
| \hat{w}_{l', t_k}^i - \frac{1}{N}| \leq \delta.
\end{equation}
\end{prop}
\begin{proof}
Since $\kappa (l)$ is unbounded, for any $S > 0$ $\exists$ $l'$ such that $\kappa (l') \geq S$. We choose the value of $S$ so that the following derivation yields \cref{eq:obsnoiseprop}.

Let $S = \sqrt{0.5R^{-1}  |\data^H_{t_k}| \diam(\Omega)^2}$, and apply the property that  $\|\vx_{t_k}^i - \vy_j\|^2 \leq \diam(\Omega)^2$ to \cref{eq:obsnoiseproof},
\begin{equation}\label{eq:obsnoiseupper}
\begin{aligned}
 &\log w_{l', t_k}^i \geq -\frac{1}{2S^2} \sum_{\vy_j \in \data^H_{t_k}} \| \vx_{t_k}^i - \vy_j\|^2  \\
& \geq -\frac{\sum_{\vy_j \in \data^H_{t_k}} \| \vx_{t_k}^i - \vy_j\|^2 }{R^{-1}  |\data^H_{t_k}| \diam(\Omega)^2}  \geq  -\frac{\sum_{\vy_j \in \data^H_{t_k}} \diam(\Omega)^2}{R^{-1}  |\data^H_{t_k}| \diam(\Omega)^2}  \geq- R.
 \end{aligned}
\end{equation}
The bound above is for the unnormalized weights, and the normalized log-weights are defined as
\begin{equation}\label{eq:weightnormalize}
 \log \hat{w}_{l', t_k}^i  =  \log w_{l', t_k}^i  -\log \bigg( \sum_{j=1}^N  \exp(\log w_{l', t_k}^j)\bigg),
\end{equation}
where for the normalizing constant it holds that
\begin{equation}\label{eq:obsnoiselower}
\log \bigg( \sum_{j=1}^N  \exp(\log w_{l', t_k}^j)\bigg) \leq \log\bigg( \sum_{j=1}^N  1\bigg) = \log(N),
\end{equation}
since $w_{l', t_k}^j$ is the value of a probability density and thus always $w_{l', t_k}^j \leq 1$. Combining \cref{eq:weightnormalize}, \cref{eq:obsnoiseupper} and \cref{eq:obsnoiselower}, it follows that
\begin{equation}
 \log \hat{w}_{l', t_k}^i - (-\log(N)\geq - R,
\end{equation}
where taking exponentials on both sides gives
\begin{equation}
  \hat{w}_{l', t_k}^i - \frac{1}{N} \geq - (1 - \exp(-R))\frac{1}{N} .
\end{equation}
Since the weights are normalized, even the largest particle weight $\hat{w}_{l', t_k}^j $ can differ from $\frac{1}{N}$ as much as every smaller weight in total lies under $\frac{1}{N}$,
\begin{equation}
 \hat{w}_{l', t_k}^j \leq \frac{1}{N}+ (N-1)\bigg((1 - \exp(-R))\frac{1}{N}\bigg),
\end{equation}
implying that for any weight $ \hat{w}_{l', t_k}^j$, it holds that
\begin{equation}\label{eq:obsabs}
| \hat{w}_{l', t_k}^j  -  \frac{1}{N} | \leq  (N-1)\bigg((1 - \exp(-R))\frac{1}{N}\bigg) \leq 1 - \exp(-R),
\end{equation}
and selecting $R = -\log (1- \delta)$ is sufficient for $\delta < 1$. 
\end{proof}

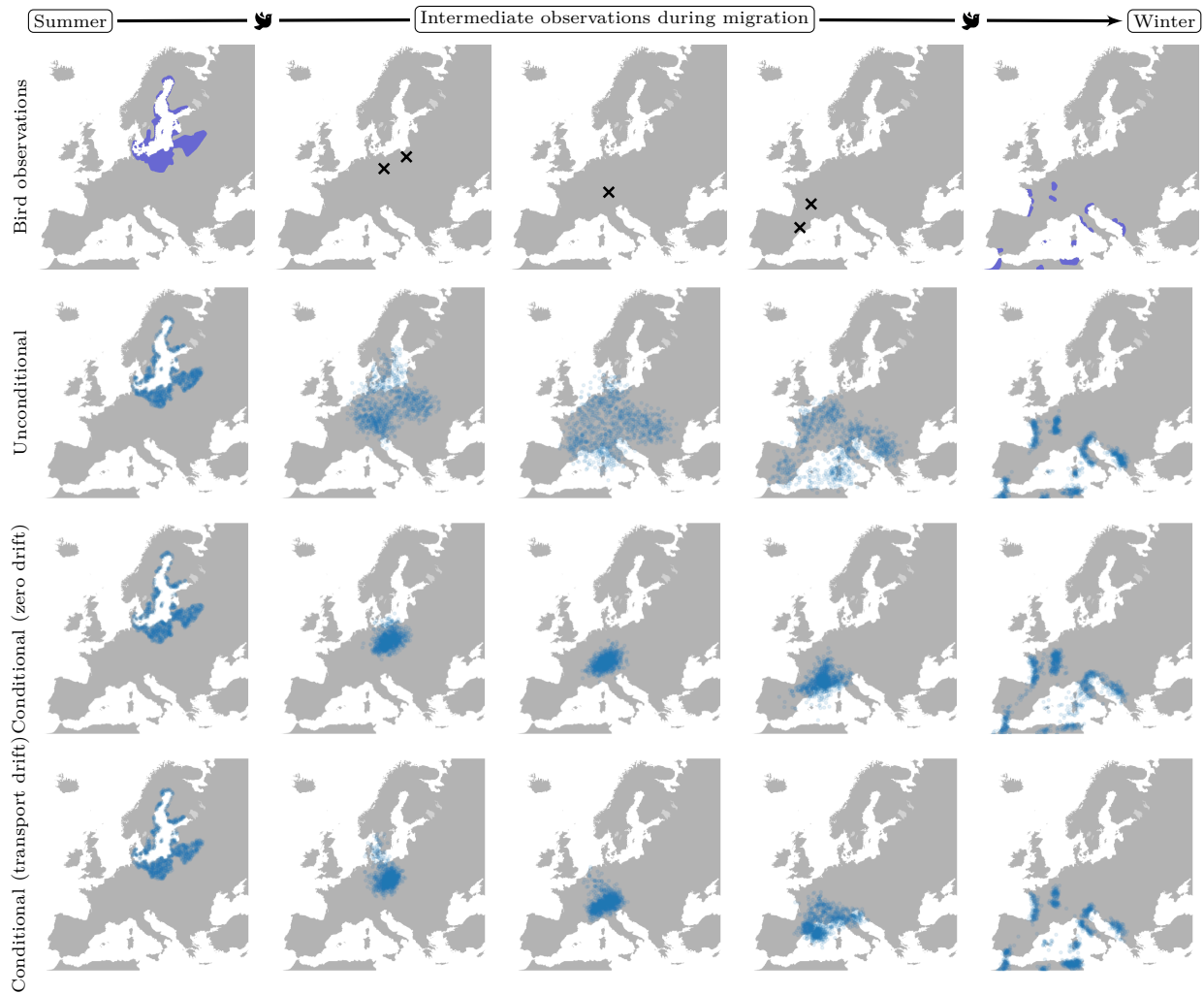
\begin{figure}[t]
  \centering\scriptsize
  \begin{tikzpicture}

      \setlength{\figurewidth}{.196\textwidth}
      \newcommand{\birds}[1]{\includegraphics[width=.95\figurewidth]{./fig/birds#1}}
      \newcommand{\dove}{\includegraphics[width=1em]{./fig/dove}}

      \tikzset{cross/.style={cross out, draw=black, minimum size=2*(#1-\pgflinewidth), inner sep=0pt, outer sep=0pt, line width=1pt}, cross/.default={3pt}}

      \node (a) at (0\figurewidth,0) {\birds{-july}};
      \node (b) at (1\figurewidth,0) {\birds{}};
      \node (c) at (2\figurewidth,0) {\birds{}};
      \node (d) at (3\figurewidth,0) {\birds{}};
      \node (e) at (4\figurewidth,0) {\birds{-january}};

      \tikzstyle{label} = [anchor=south west,inner sep=2pt,outer sep=2pt,draw=black,rounded corners=2pt,fill=white,minimum width=1cm,align=center]
      \node[label] (la) at (a.north west) {Summer};
      \node[label,anchor=south east] (lc) at (e.north east) {Winter};
      \node[label,anchor=south] (lb) at (c.north) {Intermediate observations during migration};
      \draw[black,-latex',line width=1.2pt] (la) -- node[fill=white]{\tiny \dove} (lb) -- node[fill=white]{\tiny \dove} (lc);

      \node[rotate=90] at (-.52\figurewidth,0) {Bird observations};
      \node[rotate=90] at (-.52\figurewidth,-\figurewidth) {Unconditional};
      \node[rotate=90] at (-.52\figurewidth,-2\figurewidth) {Conditional (zero drift)};
      \node[rotate=90] at (-.52\figurewidth,-3\figurewidth) {Conditional (transport drift)};

      \foreach \x/\y [count=\i] in {.65/.5,.55/.45} { %
        \node[cross] at ({.5\figurewidth+.95*\x*\figurewidth},{-.5\figurewidth+\y\figurewidth}) {\tiny};%
.       }

 \foreach \x/\y [count=\i] in {.5/.35} { %
        \node[cross] at ({1.5\figurewidth+.95*\x*\figurewidth},{-.5\figurewidth+\y\figurewidth}) {\tiny};%
.       }

 \foreach \x/\y [count=\i] in {.35/.3,.3/.2} { %
        \node[cross] at ({2.5\figurewidth+.95*\x*\figurewidth},{-.5\figurewidth+\y\figurewidth}) {\tiny};%
.       }

      \newcommand{\result}[1]{\includegraphics[width=.95\figurewidth]{./fig/birds/ot_model_europe_#1}}      
      \foreach \x [count=\i from 0] in {start,25,50,75,end}
        \node at (\i\figurewidth,-\figurewidth) {\result{\x}};

      \renewcommand{\result}[1]{\includegraphics[width=.95\figurewidth]{./fig/birds/fit_model_europe_#1}}      
      \foreach \x [count=\i from 0] in {start,25,50,75,end}
        \node at (\i\figurewidth,-2\figurewidth) {\result{\x}};
        
        \renewcommand{\result}[1]{\includegraphics[width=.95\figurewidth]{./fig/birds/alt-init_europe_#1}}      
      \foreach \x [count=\i from 0] in {start,25,50,75,end}
        \node at (\i\figurewidth,-3\figurewidth) {\result{\x}};
      
  \end{tikzpicture} 
  \caption{Top row: The first map image on the left describes the initial position of the birds, and the final one on the right depicts their position after migration. The observational data in the middle are bird observations during migration, at given timestamps. Second row: Marginal densities of a Schrödinger bridge model from the initial to terminal distribution, without using the observations. Third row: Marginal densities of our model, using both initial and terminal distributions and observational data and a zero drift initialization. Bottom row: Same as the third row, but with the second-row dynamics as initialization. }
  \label{fig:birds-app}
\end{figure}

Effectively, the above derivation implies that for an unbounded observation noise schedule $\kappa (l)$, the particle weights will converge to uniform weights. Since performing differentiable resampling on uniform weights implies that $\MT_{(\epsilon)} = \MI_{N}$, the ISB method trajectory generation step and the objective in training the backward drift converge to those of the IPF method for solving unconstrained Schrödinger bridges. Intuitively, this means that at the limit $L \to \infty$, our method will focus on reversing the trajectories and matching the terminal distribution while not further utilizing information from the observations.

\subsection{Ablation on Initial Drift}\label{app:initdrift}
We conducted an ablation study on drift initialization for the bird migration problem. As the distributions $\pi_0$ and $\pi_T$ (as pictured in \cref{fig:birds-app}) are complex, we consider the problem setting to be interesting for setting $f_0$ as the unconstrained transport problem drift. To this end, we trained a Schrödinger bridge model for $10$ epochs, and trained an ISB model with the same hyperparameter selections as explained in \cref{app:birds}, using the Schrödinger bridge as the initialization. Compare the two bottom rows of \cref{fig:birds-app} to see a selection of marginal densities of the two processes. Based on a visual analysis of the densities, it seems that the zero drift and pre-trained diffusion model initializations produce similar results around the observations, although the Schrödinger bridge initialization gave slightly sharper results at the terminal time.

\subsection{Effective Sample Size}\label{app:ess}
We studied the Effective Sample Size (ESS) of the particles generated during the filtering steps \circled{1} and \circled{3}. Ideally, our method would steer the particle trajectories near the observations quickly over the iterations, to allow for efficient training. Furthermore, as we additionally constrain the system to match the initial and terminal distributions, it is desired that the ESS will eventually rise to a high number, indicating that the particle filtering step no longer greatly adjusts the trajectories. In order to assess these properties, we computed the ESS in the scikit-learn {\sc two circles} experiment, see \cref{app:2dtoy} for other experimental details.  In \cref{tbl:ess}, we have reported three scenarios: {\em (i)}~an increasing noise schedule without learning of the drift function, {\em (ii)}~an ISB model with constant observation noise, {\em (iii)}~the proposed ISB model with the increasing noise schedule. Note that the first scenario only serves as a control to better isolate the impact of increasing noise in ISB. \cref{tbl:ess} indicates that our approach succeeds in directing the particle trajectories and consistently obtains higher ESS even when comparing the constant noise level ISB to an increasing noise level but no learning. 

\begin{table}[ht!]
  \centering \footnotesize
  \caption{Average and standard deviation of effective sample size (ESS) results over five seeds for the scikit-learn {\sc two circles} experiment, developing over six ISB iterations with $1000$ particles. We observe a rapid increase in ESS when the trajectories are steered towards the observations in the optimization step of the ISB compared to only increasing the observing noise, and see a diminishing impact of the filtering step as ESS nears $1000$. \label{tbl:ess}}
  \renewcommand{\tabcolsep}{4pt}
  \begin{tabularx}{.85\linewidth}{l c c c c c c}
    \toprule
    & \multicolumn{5}{c}{\sc ESS / Iteration}\\
    \sc Noise schedule & $l=1$ & $l=2$ & $l=3$ & $l=4$ & $l=5$  & $l=6$\\ 
    \midrule
 
    No learning and noise schedule $0.5 \times 1.25^{l-1}$ & $6 \pm 2$ & $9 \pm 2$ & $14 \pm 3$ & $27 \pm 5$ & $57 \pm 8$ & $121 \pm 11$\\
    ISB with a constant noise schedule $0.5$  & $6 \pm 2 $ & $66 \pm 21$ & $165 \pm 9$  & $300 \pm 24$ & $456 \pm 37$ & $564 \pm 46$\\
    ISB with a noise schedule $0.5 \times 1.25^{l-1}$ & $6 \pm 2$ & $ 92 \pm 21$ & $250 \pm 14$ & $442 \pm 39$ & $593 \pm 64 $ & $700 \pm 39$\\
    
    \bottomrule
  \end{tabularx}  
\end{table}

\section{Differentiable Resampling}\label{app:diff}
In the ISB model steps \circled{1} and \circled{3} presented in \cref{sec:steps}, we applied differentiable resampling \citep[see][]{corenflos2021diffpf}. Resampling itself is a basic block of particle filtering. A differentiable resampling step transports the particles and weights $(\tilde{\vx}_{t_k}^i, w_{t_k}^i)$  to a uniform distribution over a set of particles through applying the \emph{differentiable} ensemble transport map $\MT_{(\epsilon)}$, that is
\begin{equation}
   (\tilde{\vx}_{t_k}^i, w_{t_k}^i) \to (\tilde{\MX}^{\T}_{t_k} \, \MT_{(\epsilon),i}, \nicefrac{1}{N}) = (\vx_{t_k}^i, \nicefrac{1}{N}),
\end{equation}
where $\tilde{\MX}_{t_k} \in \R^{N \times d}$ denotes the stacked particles $\{\tilde{\vx}_{t_k}^i \}_{i=1}^N$ at time $t_k$ before resampling and $\vx_{t_k}^i$ denotes the particles post resampling. 
Here we give the definition of the map $\MT_{(\epsilon)}$ and review the regularized optimal transport problem which has to be solved to compute it. We partly follow the presentation in Sections 2 and 3 of  \citet{corenflos2021diffpf}, but directly apply the notation we use for particles and weights and focus on explaining the transport problem rather than the algorithm used to solve it.

The standard particle filtering resampling step consists of sampling $N$ particles from the categorical distribution defined by the weights $\{ w_{t_k}^i\}_{i=1}^N$, resulting in the particles with large weights being most likely to be repeated multiple times. A result from \citet{reich2013nonparametric} gives the property that the random resampling step can be approximated by a deterministic ensemble transform $\MT$.  In heuristic terms, the ensemble transform map will be selected so that the particles $\{ \vx_{t_k}^i\}_{i=1}^N$ will be transported with minimal cost, while allowing all the weights to be uniform.

Let $\mu$ and $\nu$ be atomic measures,  $\mu = \sum_{i=1}^N w_{t_k}^i \delta_{\tilde{\vx}_{t_{k}}^i}$ and $\nu = \sum_{i=1}^N N^{-1}\delta_{\tilde{\vx}_{t_k}^i}$, where $\delta_x$ is the Dirac delta at $x$. Then $\mu$ is the particle filtering distribution before resampling.
Define the elements of a cost matrix $\MC \in \mathbb{R}^{N \times N}$ as $C_{i, j} = \|\tilde{\vx}_{t_k}^i - \tilde{\vx}_{t_k}^j \|^2$, and the 2-Wasserstein distance between two atomic measures as
\begin{equation}
  \mathcal{W}_{2}^2(\mu, \nu) = \min_{P \in S(\mu, \nu)}  \sum_{i=1}^N \sum_{j=1}^N C_{i, j} P_{i, j}.
\end{equation}
Above the optimal matrix $\MP$ is to be found within $S(\mu, \nu)$, which is a space consisting of mixtures of $N$ particles to $N$ particles
such that the marginals coincide with the weights of $\mu$ and $\nu$, formally
\begin{equation}
  S( \mu, \nu ) = \left\{\MP \in [0, 1]^{N \times N} \mid \sum_{i=1}^N P_{i, j} = w_{t_k}^i,  \sum_{j=1}^N P_{i, j} =  \frac{1}{N} \right\}.
\end{equation} 

The entropy-regularized Wasserstein distance with regularization parameter $\epsilon$ is then
\begin{equation}
\mathcal{W}^2_{2, \epsilon} = \min_{\MP \in S(\mu, \nu)}\sum_{i=1}^N \sum_{j=1}^N P_{i, j}\left( C_{i, j}+ \epsilon \log \frac{P_{i, j}}{w_{t_k}^i \cdot \frac{1}{N}} \right).
\end{equation}
The unique minimizing transport map of the above Wasserstein distance is denoted by $\MP_{\epsilon}^\text{OPT}$, and the ensemble transport map is then set as $ \MT_{(\epsilon)} = N \MP_{\epsilon}^\text{OPT}$. This means that we can find the matrix $\MT_{(\epsilon)}$ via minimizing the regularized Wasserstein distance, which is done by applying the iterative Sinkhorn algorithm for entropy-regularized optimal transport \citep{cuturi2013sinkhorn}.

\end{document}